\documentclass[11pt]{article}

\usepackage{fullpage}
\usepackage[style=alphabetic]{biblatex}

\usepackage{amsmath,amsthm,amsfonts,amssymb}
\usepackage{hyperref}
\usepackage{color}
\usepackage{mathrsfs}
\usepackage{bm}
\usepackage{multirow}
\usepackage{booktabs}
\usepackage{makecell}
\usepackage[pdftex]{graphicx}
\usepackage{subfigure}
\usepackage{caption}
\usepackage{thmtools}
\usepackage{thm-restate}
\usepackage{setspace}
\usepackage{makecell}
\definecolor{light-gray}{gray}{0.85}
\usepackage{colortbl}
\usepackage{xcolor}
\usepackage{hhline}

\usepackage{hyperref}
\hypersetup{
     colorlinks=true,
     linkcolor=black,
     filecolor=blue,
     citecolor = blue,      
     urlcolor=blue,
     }

\usepackage{xcolor}

\usepackage[style=alphabetic]{biblatex}
\usepackage{hyperref}       
\usepackage{url}    
\usepackage{booktabs}       
\usepackage{nicefrac}       
\usepackage{microtype}  

\usepackage{amsmath,amsthm,amsxtra,amsfonts,graphicx,verbatim,epsfig,color,enumerate,array,mathtools,dsfont,multicol,mathrsfs,subfigure}

\usepackage{algorithm,etoolbox}
\usepackage{algorithmic}

\newtheorem{theorem}{Theorem}[section]
\newtheorem{lemma}[theorem]{Lemma}
\newtheorem{corollary}[theorem]{Corollary}
\newtheorem{remark}[theorem]{Remark}

\theoremstyle{definition}
\newtheorem{definition}[theorem]{Definition}

\newtheorem{assumption}{Assumption}

\def\R{\mathbb{R}}
\def\E{\mathbb{E}}

\newcommand{\cA}{\mathcal{A}}

\newcommand{\cD}{\mathcal{D}}
\newcommand{\cE}{\mathcal{E}}

\newcommand{\cG}{\mathcal{G}}
\newcommand{\cH}{\mathcal{H}}

\newcommand{\cM}{\mathcal{M}}
\newcommand{\cN}{\mathcal{N}}

\newcommand{\cP}{\mathcal{P}}

\newcommand{\cR}{\mathcal{R}}
\newcommand{\cS}{\mathcal{S}}
\newcommand{\cT}{\mathcal{T}}
\newcommand{\cU}{\mathcal{U}}

\newcommand{\bP}{\mathbb{P}}
\newcommand{\bI}{\mathbf{I}}
\newcommand{\x}{\mathbf{x}}
\newcommand{\z}{\mathbf{z}}
\newcommand{\vk}{\mathbf{k}}
\newcommand{\ca}{\mathbf{a}}
\newcommand{\y}{\mathbf{y}}
\newcommand{\X}{\mathbf{X}}
\newcommand{\Z}{\mathbf{Z}}
\newcommand{\A}{\mathbf{A}}
\newcommand{\K}{\mathbf{K}}
\newcommand{\W}{\mathbf{W}}
\newcommand{\I}{\mathbf{I}}

\newcommand{\High}[1]{{\color{black}{#1}}}
\newcommand{\Hi}[1]{{\color{black}{#1}}}
\newcommand{\Hig}[1]{{\color{black}{#1}}}

\newcommand{\deemph}[1]{{\color{black!40}#1}}

\newcommand{\argmax}{\mathop{\mathrm{argmax}}}

\newcommand{\beq}{\begin{equation}}
\newcommand{\eeq}{\end{equation}}
\newcommand{\beqn}{\begin{equation*}}
\newcommand{\eeqn}{\end{equation*}}
\newcommand{\beqa}{\begin{eqnarray}}
\newcommand{\eeqa}{\end{eqnarray}}
\newcommand{\beqan}{\begin{eqnarray*}}
\newcommand{\eeqan}{\end{eqnarray*}}

\renewcommand{\epsilon}{\varepsilon}
\renewcommand{\leq}{\leqslant}
\renewcommand{\geq}{\geqslant}
\renewcommand{\hat}{\widehat}

\usepackage{times}

\begin{document}

\title{(Private) Kernelized Bandits with Distributed Biased Feedback}

\author{%
Fengjiao Li\footnotemark[1]\footnote{{Fengjiao Li (\texttt{fengjiaoli@vt.edu}), Bo Ji (\texttt{boji@vt.edu}), Department of Computer Science, Virginia Tech.}} \quad Xingyu Zhou\footnote{Xingyu Zhou (\texttt{xingyu.zhou@wayne.edu}), Department of Electrical and Computer Engineering, Wayne State University. This work has been accepted by ACM SIGMETRICS'23.} 
\quad Bo Ji\footnotemark[1]}

\date{}

\maketitle

\begin{abstract}
In this paper, we study kernelized bandits with distributed biased feedback. This problem is motivated by several real-world applications (such as dynamic pricing, cellular network configuration, and policy making), where users from a large population contribute to the reward of the action chosen by a central entity,
but it is difficult to collect feedback from all users. Instead, only biased feedback (due to user heterogeneity) from a subset of users may be available. In addition to such partial biased feedback, we are also faced with two practical challenges due to communication cost and computation complexity.  
To tackle these challenges, we carefully design a new \emph{distributed phase-then-batch-based elimination (\texttt{DPBE})} algorithm, which samples users in phases for collecting feedback to reduce the bias and employs \emph{maximum variance reduction} to select actions in batches within each phase. 
By properly choosing the phase length, the batch size, and the confidence width used for eliminating suboptimal actions, we show that \texttt{DPBE} achieves a sublinear regret of $\tilde{O}(T^{1-\alpha/2}+\sqrt{\gamma_T T})$, where $\alpha\in (0,1)$ is the user-sampling parameter one can tune.  
Moreover, \texttt{DPBE} can significantly reduce both communication cost and computation complexity in distributed kernelized bandits, compared to some variants of the state-of-the-art algorithms (originally developed for standard kernelized bandits).
Furthermore, by incorporating various \emph{differential privacy} models (including the central, local, and shuffle models), we generalize \texttt{DPBE} to provide privacy guarantees for users participating in the distributed learning process.
Finally, we conduct extensive simulations to validate our theoretical results and evaluate the empirical performance.
\end{abstract}

\section{Introduction}
Bandit optimization is a popular  online learning paradigm for sequential decision making and has been widely used in a wide variety of real-world applications, including hyperparameter tuning \cite{li2017hyperband}, recommendation systems \cite{li2010contextual}, and dynamic pricing \cite{misra2019dynamic}. In such problems, each decision point (called an arm or action), if chosen, yields an  unknown reward. The goal of the agent is to maximize the cumulative reward by making proper decisions sequentially.
An important way to capture general (e.g., \emph{non-linear} and even \emph{non-convex}) unknown objective functions is to consider a smoothness condition specified by a small norm of a Reproducing Kernel Hilbert Space (RKHS) associated with a kernel function. This setup is often referred to as \emph{kernelized bandits}. 

Thanks to the strong link between RKHS functions and Gaussian processes (GP)~\cite{kanagawa2018gaussian,chowdhury2017kernelized,srinivas2009gaussian}, an extensive line of work has exploited GP models to estimate an unknown function $f$ given a set of (noisy) evaluations of its values $f(\x)$ at chosen actions $\x$.
However, in many applications, the value $f(\x)$ could represent an overall effect of action $\x$ on a large population of users where 
it is difficult for the learning agent to make direct observations; yet, the agent could collect some partial feedback from the distributed users in the population. In addition, feedback from these users could be biased due to user heterogeneity (e.g., different preferences). Therefore, we assume that each user $u$ in the population is associated with a local function $f_u$, which is a function sampled from a GP with mean $f$. 
\High{Consider the dynamic pricing problem~\cite{misra2019dynamic} as an example (see Figure~\ref{fig:dynamic_pricing}).
When a company sets a pricing mechanism $\x$, this decision influences all the customers, and every customer, based on her individual demand and preference, makes a choice (purchase or not), which contributes to the total profits $f(\x)$. Without knowing products' demand curves in advance, the company makes a sequence of pricing decisions with the goal of \emph{maximizing profits while learning}. That is, the company aims to infer the expected demand and thus the expected profits $f$ by collecting feedback from customers in each decision epoch.  
Note that it might be difficult for the company to
collect feedback from \emph{all} the customers - since purchases may take place at many local stores at different locations. For example, it is impractical for McDonald's headquarters to collect sales information from all of the nationwide customers within each decision epoch. Instead, the headquarter might be able to get feedback (i.e., sales information) from a subset of the customers. However, each customer's choice depends not only on her own preference towards the products and their prices but also on several other factors (location, competitors, promotion events, etc.), which is often \emph{biased} feedback for the overall profits.}
\begin{figure}[!t]
\centering
\includegraphics[width=0.5\columnwidth]{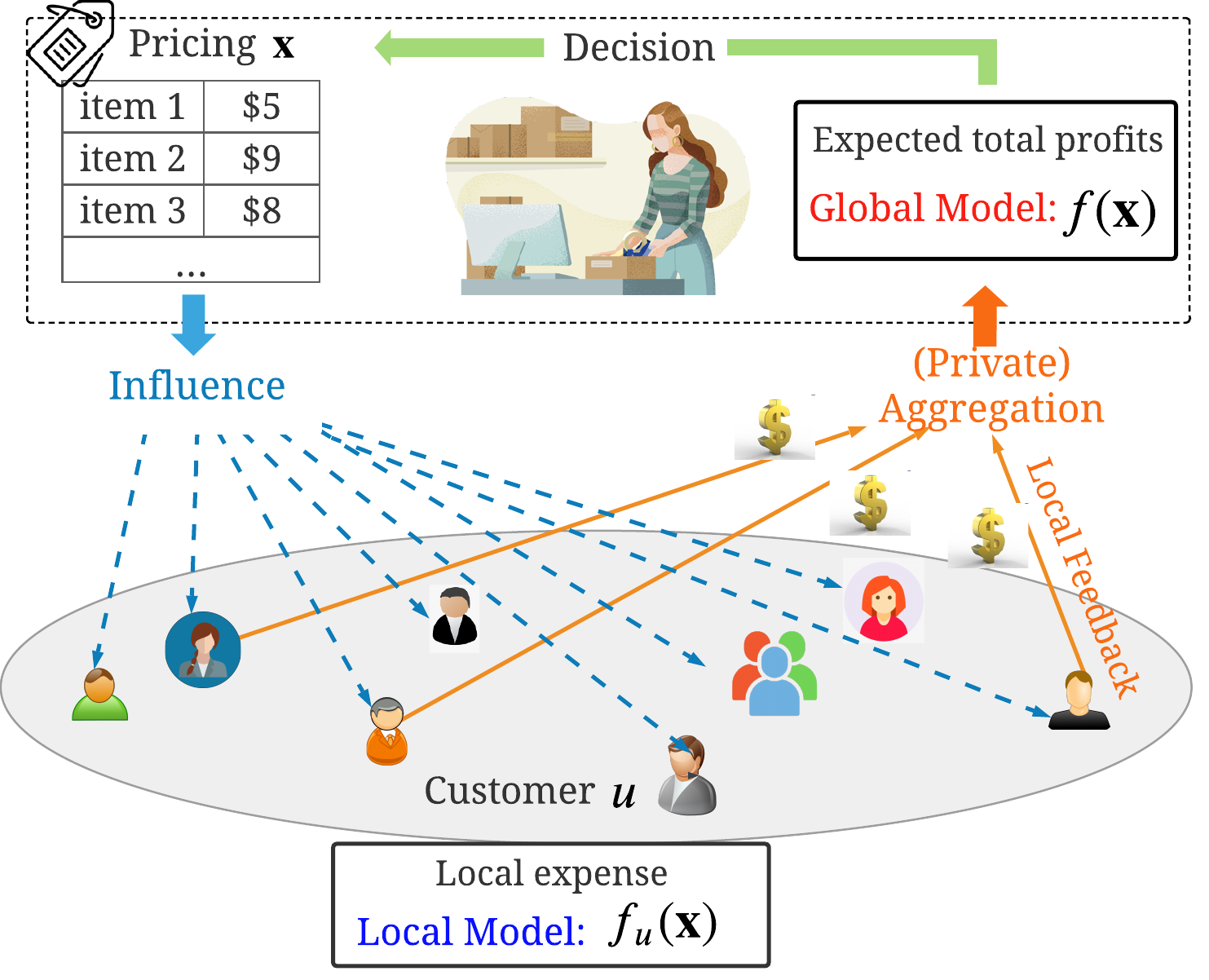}
\caption{Dynamic pricing: a motivating application of our problem.
} 
\label{fig:dynamic_pricing}
\end{figure}

To that end, we study a new kernelized bandit setting where the agent could not get direct evaluations of the unknown reward function but only distributed biased feedback. We refer to this setting as \emph{kernelized bandits with distributed biased feedback}. 
This bandit problem is shared by several other practical applications, including cellular network configuration \cite{mahimkar2021auric}  and public policy making~\cite{bouneffouf2020survey}. 
However, existing learning algorithms developed for standard kernelized/GP bandits (e.g., \texttt{GP-UCB}~\cite{srinivas2009gaussian,chowdhury2017kernelized}) rarely consider such partial biased feedback in a distributed setting. To solve this new problem, a learning algorithm needs to be able to learn the unknown function from such biased feedback in a \emph{sample-efficient} manner. Moreover, two practical challenges naturally arise in our problem: \emph{communication cost} due to distributed learning~\cite{chen2021communication} and \emph{computation complexity} due to GP update~\cite{calandriello2022scaling}. 
Therefore, not only need the learning algorithms be sample-efficient, but they must also be scalable in terms of both communication efficiency and computation complexity.


To that end, we propose the \emph{learning with communication} framework where the biased feedback is communicated in phases, and design a new \emph{distributed phase-then-batch-based elimination} algorithm that aggregates the distributed biased feedback in a communication-efficient manner and eliminates suboptimal actions in a computation-efficient manner while achieving a sublinear regret. 
Our main contributions are summarized as follows. 

\begin{itemize}
    \item 
    To the best of our knowledge, this is the first work that studies a new kernelized bandit setting with distributed biased feedback, where three key challenges (user heterogeneity, communication efficiency, and computation complexity) inherently arise in the design of sample-efficient, scalable learning algorithms.
    While it is natural to consider phased elimination type of algorithms in such settings, the standard phased elimination algorithm relies on the so-called (near-)optimal experimental design~\cite{lattimore2020learning}, which cannot be directly applied to kernelized bandits due to the possible infinite feature dimension of RKHS functions. 

    \item 
    To that end, we design a new phased elimination algorithm, called \emph{distributed phase-then-batch-based elimination (\texttt{DPBE})}, which is carefully crafted to address all the aforementioned challenges. In particular, \texttt{DPBE} adds a \emph{user-sampling} process to reduce the impact of bias from each individual user and selects actions according to \emph{maximum variance reduction} within each phase. Moreover, a \emph{batching} strategy is employed to improve both communication efficiency and computation complexity. 
    That is, instead of selecting a new action at each round, \texttt{DPBE} plays the same action for a batch of rounds before switching to the next one. 
    Not only does it help reduce the number of times one needs to 
    compute the next action via GP update, but it also allows for reducing the dimensions of the vectors and matrices involved in both communication and computation. 
    
    \item  We show that \texttt{DPBE} achieves a sublinear regret of $\Tilde{O}(T^{1-\alpha/2}+\sqrt{\gamma_T T})$\footnote{The notation $\Tilde{O}(\cdot)$ ignores polylog terms. Bounds on $\gamma_T$ of different kernel functions can be found in Appendix~\ref{app:maximum_info_gain}.} while incurring a communication cost of $O(\gamma_T T^{\alpha})$ and a computation complexity of $O((|\cD|\gamma_T^3+\gamma_T^4)\log T+\gamma_T T^{\alpha})$, where $\gamma_T$ is the \emph{maximum information gain} associated with the kernel of the unknown function $f$, 
    $\cD$ is the decision set, and $\alpha >0$ is a user-sampling parameter that we can tune. It is worth noting that \texttt{DPBE} with $\alpha \in (0,1)$ has a better computation complexity than some variants of the state-of-the-art algorithms (originally developed for standard kernelized bandits without biased feedback).
    Specifically, \texttt{DPBE} achieves three significant improvements compared to \High{the state-of-the-art algorithms}: 
    (i) user-sampling efficiency ($O(T^{\alpha})$ vs. $T$), 
    (ii) communication cost ($O(\gamma_T  T^{\alpha})$ vs. $T$), 
    and (iii) computation complexity ($O(\gamma_T T^{\alpha})$ vs. $O(T^3)$). Furthermore, we conduct extensive simulations to validate our theoretical results and evaluate the empirical performance in terms of regret, communication cost, and running time.

    \item    Finally, we generalize our phase-then-batch framework to incorporate various \emph{differential privacy} (DP) models (including the central, local, and shuffle models) into \texttt{DPBE}, which ensures privacy guarantees for users participating in the distributed learning process.

\end{itemize} 
\section{Related Work} \label{sec:related_work}
\noindent\textbf{Kernelized bandits.} 
Since \cite{srinivas2009gaussian} studied GP in the bandit setting, kernelized bandits (also called GP bandits) have been widely adopted to address black-box function optimization over a large or infinite domain \cite{chowdhury2017kernelized}. Considering different application scenarios, kernelized bandits under different settings have recently been studied, including heavy-tailed payoffs~\cite{ray2019bayesian}, model misspecification~\cite{bogunovic2021misspecified}, and corrupted rewards~\cite{bogunovic2022robust}. 
As typically considered in the literature, these works also assume that direct (noisy) feedback of the unknown function at a chosen action is available to the agent. In sharp contrast, we study a new, practical setting where only distributed biased feedback can be obtained. Under this setting, not only does one need to use biased feedback in a sample-efficient manner, but one also has to consider communication efficiency, which is a common issue in distributed bandit-learning settings. 






\noindent\textbf{Distributed/collaborative kernelized bandits.}
While distributed or collaborative kernelized bandits have been studied recently \cite{du2021collaborative, dai2020federated, sim2021collaborative}, we highlight the key difference between our model and theirs as follows: motivated by real-world applications, we aim to learn one (global) bandit while most of them also aim to learn every local model, which results in quite different regret definitions (their group regret vs. our standard regret defined in Section~\ref{sec:problem}). Moreover, they assume that every party (corresponding to a user in our problem) shares the same objective function. While \cite{dai2020federated} also studies similar bandit optimization with biased feedback, they assume a fixed number of local agents and bound the regret in terms of the distance between the target function and local functions, which could be very large. In addition, \cite{dai2020federated} does not consider communication efficiency, which is a key challenge in distributed learning.

Recently, the work of \cite{li2022differentially} studies a similar global reward maximization problem without direct feedback and also employs a phase-based elimination algorithm. However, the main difference is that they only consider linear bandits by assuming a linear reward function  while we study kernelized bandits that can capture general \emph{non-linear} and even \emph{non-convex} functions and recover linear bandits as a special case when choosing a linear kernel. This strict generalization introduces three unique challenges: (i) different from the linearly parameterized bandits where the bias in the feedback can be quantified with a same-dimension random vector (i.e., $\xi_u = \theta_u - \theta^* \in R^d$ at each user $u$), it is unclear how to make an assumption of the bias in the non-parametric kernelized bandits setting in order to learn the unknown global reward function;
(ii) due to the possible infinite feature dimension of functions in an RKHS, the (near-)optimal experimental design approach used in the phased-elimination algorithm for linear bandits cannot be directly adapted to kernelized bandits. Despite some recent efforts towards extending this experimental design based approach to kernelized bandits~\cite{zhu2021pure, camilleri2021high}, there still remain some key limitations (see our discussion below); (iii) since computation complexity is a critical bottleneck in kernelized bandits, a proper computation-efficient learning algorithm is desired when addressing our problem.  

\noindent \textbf{Experimental design for kernelized bandits.} In \cite{zhu2021pure}, the authors propose to adaptively embed the feature representation of each action into a lower-dimensional space in order to apply the (near-)optimal experimental design for finite-dimensional actions. However, the intermediate regret due to the approximation error over $T$ rounds is not considered at all because their goal is to 
find an $\epsilon$-optimal arm at the end of $T$ (i.e., a pure exploration problem) rather than minimizing the cumulative regret. While \cite{camilleri2021high} aims at minimizing the cumulative regret, their algorithm and analysis are more complex than ours: it requires a non-standard robust estimator, obtaining an optimal distribution on the simplex, drawing samples from this distribution, and solving a second optimization problem. In contrast, we simply use the standard GP posterior mean and variance estimators, which can be computed in closed-form. Moreover, our algorithm can also be easily extended to handle infinite action sets (see Remark~\ref{rk:infinite_set}) rather than a finite set considered in~\cite{camilleri2021high}.



\section{Preliminaries}\label{sec:problem}
\textbf{Notation.} Throughout this paper, we use lower-case letters (e.g., $x$) for scalars, lower-case bold letters (e.g., $\x$) for vectors, and upper-case bold letters (e.g., $\X$)  for matrices. Let $[n]\triangleq \{1,\dots, n\}$ denote any positive integer up to $n$, let $|\cU|$ denote the cardinality of set $\cU$, and let $\Vert \x\Vert_2$ denote the $\ell_2$-norm of vector $\x$. 

\subsection{Problem Setting}

We introduce a new kernelized bandit problem where the unknown function represents the overall reward over a large population containing an infinite number of users.
The unknown reward function $f: \cD\to \R$ is assumed to be fixed over a finite set of decisions $\cD\subseteq \R^d$.
At round $t$, the agent chooses an action $\x_t\in \cD $, leading to a reward with mean $f(\x_t)$. This reward is unknown to the agent but captures the overall effectiveness of action $\x_t$ over the entire population $\cU$, thus called \emph{global reward}. Meanwhile, each user $u$ in the population observes a (noisy) \emph{local reward}: $y_{u,t}= f_u(\x_t) + \eta_{u,t}$ with mean $f_u(\x_t)$, where $\eta_{u,t}$ is the noise, and $f_u: \cD\to \R$ is the local reward function, assumed to be an (unknown) realization of a random function (specified soon) with mean $f$. In this setting, the exact global reward corresponding to the entire population cannot be observed; only biased local reward feedback is available to the agent.
We make the following assumptions about the unknown function $f$, the local function $f_u$, and the noise in the reward observations.

\begin{assumption}
\label{ass:global_func}
We assume that function $f$ is in the Reproducing Kernel Hilbert Spaces (RKHS), denoted by $\cH_k$. Note that RKHS $\cH_k$ is completely specified by its kernel function $k(\cdot, \cdot)$ (and vice-versa), with an inner product $\langle \cdot, \cdot \rangle_k$ obeying the reproducing property: $f(\x) = \langle f(\cdot), k(\x, \cdot)\rangle_k$ for all $f \in \cH_k$~\cite{chowdhury2017kernelized}. We list the most commonly used kernel functions (such as Squared Exponential (SE) and Mat\'{e}rn kernels) in Appendix~\ref{app:kb_auxiliary_results}. 
Moreover, we assume that function $f$ has a bounded norm: $\Vert f \Vert_k \triangleq \sqrt{\langle f, f \rangle_k} \leq B$, and that the kernel function is also bounded: $k(\x,\x)\leq \kappa^2$ for every $\x \in \cD$, where both $B$ and $\kappa$ are positive constants.   
\end{assumption}

\begin{assumption}
\label{ass:local_func}
When the agent samples a user $u$ to collect feedback, the local reward function $f_u$ at $u$ is assumed to be a function sampled from the GP with mean $f$ and covariance\footnote{\High{Our theoretical framework is applicable to a more general setting where the covariance of the local reward function is $v^2k(\cdot, \cdot)$, i.e., $f_u \sim \cG\cP(f(\cdot), v^2k(\cdot, \cdot))$. This scaling parameter $v^2$ captures the variance of the bias in the local reward function $f_u$ with its mean being the global reward function $f$. For this more general setting, our theoretical results still hold with only a slight adjustment to the posterior variance in the confidence width function \eqref{eq:confidence_width}.}} $k(\cdot, \cdot)$, i.e., $f_u \sim \cG\cP(f(\cdot), k(\cdot, \cdot))$.  In addition, we assume that each user is sampled independently for collecting feedback.
\end{assumption}

\begin{assumption}
\label{ass:noise}
We assume that the observation noise $\eta_{u,t}{\sim} \cN(0,\sigma^2)$ is Gaussian with variance $\sigma>0$ and that it is independent and identically distributed (\emph{i.i.d.}) over time and across users. 
\end{assumption}

The goal of the agent is to maximize the cumulative global reward, 
or equivalently, to minimize the regret defined as follows:
\begin{equation}
\label{eq:regret}
R(T) \triangleq \sum_{t=1}^T \left(\max_{\x\in \cD} f(\x) - f(\x_{t})\right).
\end{equation}

\subsection{Learning with Communication}
For black-box function optimization based on noisy bandit feedback, kernelized bandit algorithms have shown strong empirical and theoretical performance. However, the agent in our problem setting does not have access to unbiased feedback of the object function $f$ but has to collect biased feedback from distributed users from a large population. This scenario 
leads to the following framework of \emph{learning with communication}. 

Communication happens when some users are selected 
to report their feedback to the agent based on their biased local reward observations. By aggregating such biased feedback from the users, the agent improves her confidence in estimating function $f$ and adjusts her decisions in the following rounds accordingly. To account for scalability, the agent collects distributed feedback from users periodically instead of immediately after making each decision.  We call the time duration between two communications as a \emph{phase}. Consider a particular phase $l$. Let $\cT_l$ be the set of round indices in the $l$-th phase and $U_l$ be the set of selected users, called \emph{participants}, that will report their feedback. With the actions $\{\x_t: t\in \cT_l\}$ chosen by the agent in this phase, each user $u$ in $U_l$ sends the feedback $g(\{y_{u,t}\}_{t\in\cT_l})$ to the agent at the end of the phase, where $g(\cdot)$ is a function (e.g., the average) 
of the local reward observations and is assumed to be the same 
for all users. Then, by aggregating all feedback $\{g(\{y_{u,t}\}_{t\in\cT_l})\}_{u\in U_l}$, the agent estimates $f$ and decides $\x_t$ for round $t$ in the next phase $\cT_{l+1}$. This learning with communication process is repeated until the end of $T$, with the goal of maximizing the cumulative (global) reward. 

In this framework, we assume that the agent can employ some existing incentive mechanisms~\cite{lim2020federated} in order to collect enough feedback for learning, but the cost has to be considered, e.g., the communication resources consumed for collecting feedback data.  
In addition, communication cost is also a critical factor in a general distributed learning system. 
In this work, we use the total quantity of communicated numbers (between the agent and all users) as another metric, in addition to the regret metric, 
to evaluate the communication efficiency of learning algorithms for our problem. Let $L$ be the total number of phases in $T$ rounds and $N_{u,l}\triangleq  \dim{(g(\{y_{u,t}\}_{t\in\cT_l}))}$ be the dimension of user $u$'s feedback (which is the number of scalars in user $u$'s feedback). 
Then, the total communication cost, denoted by $C(T)$, is as follows:
\begin{equation}
	C(T) \triangleq \sum_{l=1}^L \sum_{u\in U_l} N_{u,l}. \label{eq:comm_cost}
\end{equation}

In the following, we explain the learning with GP framework for standard kernelized bandits.

\subsection{Learning with Gaussian Process} \label{sec:preliminary}

A Gaussian process (GP) over input domain $\cD$, denoted by $\cG\cP(\mu(\cdot), k(\cdot, \cdot))$, is a collection of random
variables $\{f(\x)\}_{\x\in\cD}$ where every finite number of them $\{f(\x_i)\}_{i=1}^n, n\in \mathbb{N}$, is jointly Gaussian with
mean $\E[f(\x_i)] = \mu(\x_i)$ and covariance $\E[(f(\x_i)-\mu(\x_i))(f(\x_j)-\mu(\x_j))]=k(\x_i,\x_j)$ for every $1 \leq i, j\leq n$.
Hence, $\cG\cP(\mu(\cdot),k(\cdot, \cdot))$ is specified by its mean function $\mu$ and a (bounded) covariance function $k: \cD\times \cD \to [0,\kappa^2]$.
Assume that choosing action $\x_t$ at round $t$ reveals a noisy observation: 
\begin{equation}
 y_t=f(\x_t)+\eta_t, \label{eq:GP_observation}  
\end{equation}
where $\eta_t\sim \cN(0,\lambda)$ is a zero-mean Gaussian noise with variance $\lambda>0$.
Standard GP algorithms implicitly use $\cG\cP(0,k(\cdot, \cdot))$ as the prior distribution over $f$. 
Then, given the observations $\y_t = [y_1, \dots, y_t]^{\top}$ corresponding to a sequence of actions $\X_t = [\x_1^{\top}, \dots, \x_t^{\top}]^{\top}$, the posterior distribution 
is also Gaussian with the mean and variance in the following closed-form:
\begin{align}
\mu_t(\x) &\triangleq \vk(\x, \X_t)^{\top} (\K_{\X_t\X_t} +\lambda \bI)^{-1}\y_t, \label{eq:standard_mu_update} \\
\sigma_t^{2}(\x) &\triangleq \vk(\x, \x) - \vk(\x, \X_t)^{\top}(\K_{\X_t\X_t}+\lambda \mathbf{I})^{-1}\vk(\x,\X_t), \label{eq:standard_sigma_update}
\end{align}
where $\vk(\x,\X_t) = [k(\x,\x_s)]^{\top}_{s=1,\cdots, t}\in \R^{t\times 1}$ and $\K_{\X_t\X_t} = [k(\x,\x')]_{\x,\x'\in \X_t}\in \R^{t\times t}$ is the corresponding kernel matrix. 
%

Next, we introduce an important kernel-dependent quantity called \emph{maximum information gain}~\cite{srinivas2009gaussian}:
\begin{equation}
    \gamma_t(k, \cD) \triangleq  \max_{\X \subseteq \cD: |\X|=t} \frac{1}{2}\log \det \left( \mathbf{I} + \lambda^{-1}\K_{\X\X}\right),
\end{equation}
which is often used to derive regret bounds. In addition, we have that $\gamma_t(k,\cD)$ scales sublinearly with $t$ for most commonly used kernels (see Appendix~\ref{app:kb_auxiliary_results}). For ease of notation, we often simply use $\gamma_t$ to denote $\gamma_t(k,\cD)$ when the kernel function $k$ and the dataset $\cD$ are clear from the context.

Thanks to the strong connection between RKHS functions and GP \cite{kanagawa2018gaussian} with the same kernel function $k$, one can use the above GP model to approximate unknown function $f \in \cH_k$ within a reliable confidence interval with high probability. 



\section{Algorithm Design}
\subsection{New Challenges and Main Ideas}
In Section~\ref{sec:problem}, we describe the learning with communication framework, which requires the distributed biased feedback to be communicated in phases and exhibits experimental scalability.  This framework naturally leads us to consider a phased elimination algorithm that gradually eliminates suboptimal actions by periodically aggregating and analyzing the local feedback from the participants. However, several new challenges arise in our setting compared to the standard phase elimination algorithm in linear bandits \cite{lattimore2020learning, lattimore2020bandit}.


\textbf{(i) How to select actions for each phase?} The standard phase elimination algorithm often relies on the so-called near-optimal experimental design (i.e., a probability distribution over the currently active set) that minimizes the worst-case variance~\cite{lattimore2020bandit}. However, due to the possible infinite feature dimension of RKHS functions, adapting this approach to kernelized bandits setting is nontrivial  even with the strong assumptions, requirements, and complicated algorithm design (e.g., \cite{zhu2021pure} and \cite{camilleri2021high}, see discussion in Section~\ref{sec:related_work}). 
We are wondering if there is a simple and efficient method of selecting actions in each phase for our kernelized bandits setting. (Challenge \textcircled{a}).

\textbf{(ii) How to use biased feedback?} In contrast to the standard phase elimination algorithm where feedback is unbiased, in our setting the local feedback from a particular user is biased. In order to reduce the impact of bias, an efficient user-sampling scheme is needed. However, how to incorporate this idea into the phase elimination algorithm is unclear (Challenge \textcircled{b}). 

\textbf{(iii) How to deal with scalability?} In our setting, scalability refers to both computation complexity and communication cost. On the one hand, it is well-known that standard GP bandits suffer a poor computation complexity (e.g., $O(T^3)$) due to the matrix inverse at each step for GP posterior update. On the other hand, due to the communication between the agent and the users, it is imperative to ensure a low communication cost (Challenge \textcircled{c}).  

\textbf{Our approach.} We propose a novel phase elimination algorithm that is able to simultaneously address all the above challenges. We highlight the main ideas as follows. (i) \emph{User-sampling} for distributed biased feedback. In each phase, a well-tuned subset of users is sampled to reduce the impact of bias from each individual user. (ii) \emph{Maximum variance reduction} for action selection. Upon selecting the next action within each phase, it simply selects the one that has the largest posterior variance. (iii) \emph{Batching strategy} for scalability. Instead of selecting a new action at each round within a phase, it consistently plays the same action for a batch of rounds before selecting the next one, i.e., \emph{rare-switching}. 
By reducing the number of times selecting a new action (which could be much smaller than the phase length), it also reduces the number of unique actions chosen within each phase, which can be utilized to improve the scalability in terms of both computation and communication through a proper design. Specifically, (a) \emph{Computation}: via a \emph{posterior reformulation} (specified in Section~\ref{sec:alg}), we convert the dimension of the matrix in the inverse operation from the total rounds to the number of batches in each phase; (b) \emph{Communication}: we let each participant  \emph{merge the local reward observations} in each batch before sending her feedback at the end of each phase. That is, the feedback $g(\{y_{u,t}\}_{t\in\cT_l})$  from each  participating user $u$  in phase $l$ is a vector, where each element corresponds to the average local reward of a batch. Then, the dimension of the feedback $g(\{y_{u,t}\}_{t\in\cT_l}$ becomes the number of batches. For example, consider a particular phase with a total of $10$ rounds. Without batching strategy, one requires to select an action for each round, i.e., $10$ actions for this phase. However, the batching strategy selects an action for each batch. If each batch has size two, there are $5$ batches in this phase,
and the dimension of the matrix in the inverse operation is shrunk from $10$ to $5$, which will reduce the computation complexity 
about $10^3/5^3=8$ times for matrix inverse operations! In addition, by merging local observations of each unique action, only  $5$, instead of $10$, (averaged) local rewards are communicated  at each user.

\smallskip 
	

 \begin{figure}[!t]
    \centering
    \includegraphics[width=0.8\linewidth]{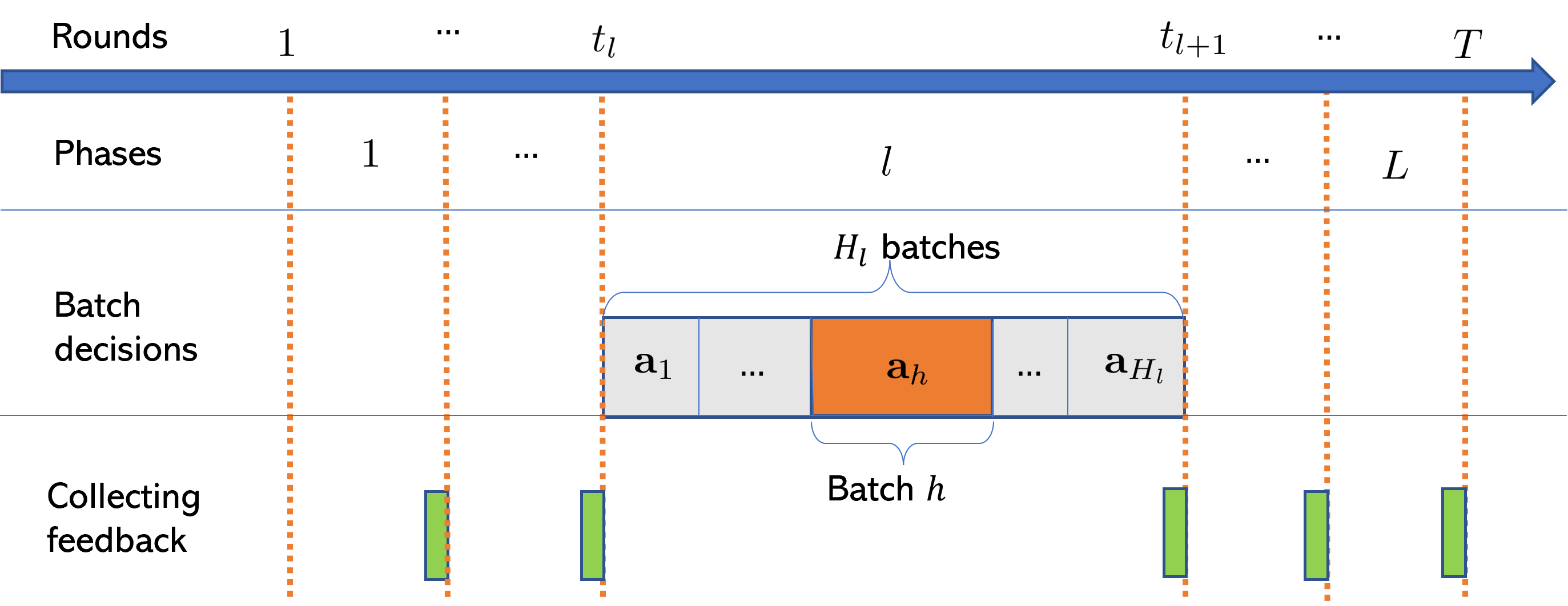}
    \caption{The phase-then-batch strategy: $T$ rounds are divided into $L$ phases; at the end of each phase, participants report their feedback, which is used for deciding actions in the next phase; within each phase $l$, decisions are made in a batched fashion, e.g., playing $\ca_h$ at all the rounds in the $h$-th batch.
    }
    \label{fig:schedule}
\end{figure}

\subsection{Distributed Phase-then-Batch-based Elimination (\texttt{DPBE})} \label{sec:alg}
Following the main ideas stated in the above section, we propose the phase-then-batch schedule strategy, shown in Figure~\ref{fig:schedule} and design the distributed phase-then-batch-based elimination (\texttt{DPBE}) algorithm  in Algorithm~\ref{alg:DPBE}.

The \texttt{DPBE} algorithm is a phased elimination algorithm, which maintains a set $\cD_l$ of active actions that are possible to be optimal and updates the active set after aggregating the distributed feedback. 

Consider a particular phase $l$, \texttt{DPBE} has three main steps: 1) action selection (Lines~\ref{alg_action_selection}-\ref{alg_action_slc_end}); 2) distributed feedback collection (Lines~\ref{alg_client_sample}-\ref{alg_collect_data}); 
and 3) action elimination (Lines~\ref{alg_aggregation}-\ref{alg_update}). 

Before describing the details of \texttt{DPBE}, we explain some additional notations used in the algorithm. Throughout this paper, we use another notation ``$\ca$'' to denote the specific chosen action under our algorithm to avoid too many subscripts or superscripts for all the batch, phase, or round indices. Consider the $l$-th phase. Let $t_l$ and $T_l$ be the time index right before the $l$-th phase and the length of the $l$-th phase, respectively.  Then, the round indices in the $l$-th phase can be represented as $\cT_l = \{t\in [T]: t_l+1\leq t \leq t_l+T_l\}$. In addition, $\cT_l(\ca)\triangleq \{t\in \cT_l: \x_t = \ca\}$ denotes the time indices when action $\ca$ 
is selected in this phase, and $H_l$ represents the number of batches in the $l$-th phase.

\textbf{1) Action selection (Lines~\ref{alg_action_selection}-\ref{alg_action_slc_end}):} 
In the $l$-th phase, actions are selected from the active set $\cD_l$. 
As mentioned before, each selection is based on \emph{maximum variance reduction} \cite{vakili2021optimal}, and we employ batch schedule for scalability.
Specifically, in the $h$-th batch, we find the action $\ca_h$ 
that maximizes a reformulated posterior variance $\Sigma_{h-1}(\cdot)$ defined in Eq.~\eqref{eq:sigma_update}  after $h-1$ batches (Eq.~\eqref{eq:decision}). This is possible because the posterior variance can be computed without knowing any reward observations (see Eq.~\eqref{eq:standard_sigma_update}). 
Then, play this action for $T_l(\ca_h)\triangleq \lfloor 	(C^2-1)/\Sigma_{h-1}^2(\ca_h)\rfloor$ rounds, which forms the $h$-th batch. Here,  the batch size schedule is inspired  by the \emph{rare-switching} idea in \cite{abbasi2011improved,calandriello2022scaling}. This batch schedule strategy enables us to merge rounds and thus shrink the dimensions of the matrix and vectors used for computing the variance in Eq.~\eqref{eq:standard_sigma_update}.  By the end of each batch, we update the variance function by incorporating the action in the current batch. Let $\A_h = [\ca_1^{\top}, \dots, \ca_h^{\top}]^{\top}\in \R^{h\times d}$ be the $h\times d$ matrix that contains the $h$ chosen actions so far. We reformulate the standard posterior variance in Eq.~\eqref{eq:standard_sigma_update} and update the posterior variance as follows:
\begin{equation}
\Sigma_h^{2}(\x) \triangleq k(\x, \x) - \vk(\x,\A_h)^{\top}(\K_{\A_h\A_h}+\lambda \W_{h}^{-1})^{-1}\vk(\x,\A_h)
\label{eq:sigma_update},
\end{equation}
where $\W_h\in \R^{h\times h}$ is a diagonal matrix with $[W_h]_{ii} = T_l(\ca_i)$ for any $i\in [h]$\High{, and $\lambda$ is set to be the noise variance of local observations, i.e., $\lambda = \sigma^2$}.   Here, we reformulate the standard posterior variance in Eq.~\eqref{eq:standard_sigma_update} with Eq.~\eqref{eq:sigma_update} in order to save computation complexity (especially for computing matrix inverse) while maintaining the same order of regret (sacrificing only a constant multiplier). 

\begin{algorithm}[!t]
\caption{Distributed Phase-then-Batch-based Elimination (\texttt{DPBE})}
\label{alg:DPBE}
\begin{algorithmic}[1]
\STATE \textbf{Input:}  $\cD\subseteq \R^d$, parameters $\alpha >0$, $\beta \in (0,1)$, $C$, and local noise $\sigma^2$
\STATE \textbf{Initialization:} $l=1$, 
$\cD_1=\cD$, $t_1=0$, and $T_1= 1$
	\WHILE{$t_l<T$}
	\STATE Set $\tau = 1$, $h=0$, $\tau_1=0$ and $\Sigma^2_0(\x)=k(\x,\x)$
	, for all $\x \in \cD_l$
	\WHILE{$\tau \leq T_l$} \label{alg_action_selection}
	\STATE $h = h+1$
	\STATE Choose action
	\begin{equation}
	    \ca_h \in \argmax_{\x\in \cD_l} \Sigma_{h-1}^2(\x) \label{eq:decision}
	\end{equation}
	\STATE Play action $\ca_h$ for $T_{l}(\ca_h)\triangleq \lfloor
	(C^2-1)/\Sigma_{h-1}^2(\ca_h)\rfloor$ rounds if not reaching $\min\{T,t_l+T_l\}$
	\STATE Update $\tau = \tau+T_l(\ca_h)$, and incorporate $\ca_h$ into the posterior variance $\Sigma^2_{h}(\cdot)$ (see Eq.~\eqref{eq:sigma_update})
	\ENDWHILE \label{alg_action_slc_end}
	\STATE Let $H_l = h$ denote the total number of batches in this phase
	\STATE Randomly select $\lceil 2^{\alpha l}\rceil$ participants $U_l$ \label{alg_client_sample}
	\item[] \deemph{\# Operations at each participant }
	\FOR{each participant $u  \in U_l$ }
	\STATE Collect and compute local average reward for every chosen action $\ca\in \A_{H_l}$:  $$y_l^u(\ca) = \frac{1}{T_{l}(\ca)}\sum_{t\in \cT_{l}(\ca)}y_{u,t}$$ 
	\STATE Send the (local) average reward for each chosen action $\y_l^u\triangleq [y_l^u(\ca)]_{\ca\in \A_{H_l}}$ to the agent \label{alg_communicate}
	\ENDFOR \label{alg_collect_data} 
    \STATE Aggregate local observations for each chosen action $\ca \in \A_{H_l}$: 
    \begin{equation}
        y_l(\ca) = \frac{1}{|U_l|}\sum_{u\in U_l} y_l^u(\ca)
    \end{equation}\label{alg_aggregation}
    \STATE Let $\bar{\y}_l = [y_l(\ca_1), \dots, y_l(\ca_{H_l})]$\label{alg_combine}  
	\STATE Update $\bar{\mu}_l(\cdot)$ according to Eq.~\eqref{eq:mu_update} \label{alg_update_mu}
	\STATE Eliminate low-rewarding actions from $\cD_l$ based on the confidence width $w_l(\cdot)$ in Eq.~\eqref{eq:confidence_width}:
	\begin{equation}
	\cD_{l+1} = \left\{\x\in \cD_l: \bar{\mu}_l(\x)+ w_l(\x)\geq \max_{\mathbf{b}\in \cD_l}( \bar{\mu}_l(\mathbf{b})-w_l(\mathbf{b}))\right\} \label{eq:action_elimination}
	\end{equation} \label{alg_elimination}
	\STATE $T_{l+1} = 2T_l$, 	$t = t+T_l$, $l = l+1$ \label{alg_update}
	\ENDWHILE
\end{algorithmic}
\end{algorithm}
\textbf{2) Distributed feedback collection (Lines~\ref{alg_client_sample}-\ref{alg_collect_data}):} To reduce the impact of bias from some specific user(s), the agent  randomly samples a subset of users (called participants) $U_l$ from $\cU$ to participate in the learning process (Line~\ref{alg_client_sample}). \High{We let $|U_l|=\lceil2^{\alpha l}\rceil$, where the user-sampling parameter $\alpha>0$ is an input of the algorithm.} Recall that $H_l$ denotes the number of batches in the $l$-th phase. 
Each participant $u \in U_l$ collects their local reward observations of each chosen action $\ca\in \A_{H_l}$ 
and send the average $y_l^u(\ca)$ for every chosen action $\ca\in\A_{H_l}$ as feedback to the agent, i.e., $g(\{y_{u,t}\}_{t\in\cT_l}) =\y_l^u \triangleq [y_l^u(\ca)]_{\ca\in \A_{H_l}} $. Note that the dimension of the feedback depends on the number of batches, which is also the communication cost associated with each participant (Eq.~\eqref{eq:comm_cost}). Therefore, \emph{by employing the idea of rare switching, we reduce both computation complexity and  communication cost (\textcircled{c})}. 
	
\textbf{3) Action elimination {(Lines~\ref{alg_aggregation}-\ref{alg_update})}:} Aggregate (specifically, average) the feedback from the participants 
for each action $\ca\in \A_{H_l}$ (Line~\ref{alg_aggregation}).  Then, using the aggregated feedback (i.e., the averaged local reward $\bar{\y}_l = [y_l(\ca_1), \dots, y_l(\ca_{H_l})]$  of the chosen actions $\ca\in \A_{H_l}$), the agent can compute the posterior mean function reformulated as follows (Line~\ref{alg_update_mu}): 
\begin{equation}
    \bar{\mu}_l(\x) \triangleq \vk(\x, \A_{H_l})^{\top}(\K_{\A_{H_l}\A_{H_l}}+\lambda \W_{H_l}^{-1})^{-1}\bar{\y}_l \label{eq:mu_update}.
\end{equation} 
Considering the bias in the feedback due to user heterogeneity (\textcircled{b}), we carefully construct a confidence width $w_l(\cdot)$  that incorporates both the noise and bias  as follows:
\begin{equation}
    w_l(\x) \triangleq \sqrt{\frac{2 k(\x, \x)\log(1/\beta)}{|U_l|}} +   \sqrt{\frac{2\Sigma^2_{H_l}(\x)\log(1/\beta)}{ |U_l|}}+B\Sigma_{H_l}(\x),\label{eq:confidence_width}
\end{equation}
where $B$ is the bound of $f$'s kernel norm, and $\beta$ is the confidence level from the input. 
%
Using this confidence width $w_l(\cdot)$ and the mean estimator function $\bar{\mu}_l(\cdot)$ in Eq.~\eqref{eq:mu_update}, we can identify suboptimal actions with high probability (\emph{w.h.p.}). Finally, we update the set of active actions $\cD_{l+1}$ by eliminating the suboptimal actions from $\cD_l$ (Line~\ref{alg_elimination}). 

\begin{remark}[Merge batches]\label{rk:shrink}
For implementation, we also merge 
different batches with the same chosen action in each phase. By doing this, we further shrink the dimension of the matrix in the inverse operation (thus reducing the time complexity) and also the dimension of local feedback (thus reducing the communication cost). 
\end{remark}

\begin{remark}[General decision set]\label{rk:infinite_set} Following the techniques used in \cite{li2022gaussian}, \texttt{DPBE} can also be extended from a finite domain to a continuous domain (e.g., $\cD = [0,1]^d$) via a simple discretization trick and 
 Lipschitz continuity of functions under commonly used kernels. 
\end{remark}

\section{Main Results} \label{sec:main_results}

In this section, we present the performance of our proposed \texttt{DPBE} algorithm in terms of regret, computation complexity, and communication cost, respectively. 


First, we analyze the regret performance of \texttt{DPBE} and present the upper bound in Theorem~\ref{thm:regret_upper_bound}. 
 While the \texttt{DPBE} algorithm uses GP tools to define and manage the uncertainty in estimating the unknown function $f$, the analysis of \texttt{DPBE} algorithm does not rely on any \emph{Bayesian} assumption about $f$ being drawn from the prior $\cG\cP(0,k(\cdot, \cdot))$, and it only requires $f$ to be bounded in the kernel norm associated with the RKHS $\cH_k$. 
\begin{theorem}[Regret]\label{thm:regret_upper_bound} Let $\beta=\frac{1}{|\cD|T}$. Under Assumptions~\ref{ass:global_func}, \ref{ass:local_func} and \ref{ass:noise}, the \texttt{DPBE} algorithm 
achieves the following expected regret:
\begin{equation}
\begin{aligned}
    \E[R(T)] &= O(T^{1-\alpha/2}\sqrt{\log (|\cD|T)})+O(\sqrt{\gamma_T T}) + O(\sqrt{\gamma_T T^{1-\alpha}\log (|\cD|T)}).
    \label{eq:regret_bound}
\end{aligned}
\end{equation}
\end{theorem}
We provide the detailed proof of Theorem~\ref{thm:regret_upper_bound} in Appendix~\ref{app:proof_reget}.  Bounds for $\gamma_T$ of different kernels can be found in Appendix~\ref{app:maximum_info_gain}.
In the following, we make two remarks about the above result.

\begin{remark} \label{rk:regret}
In the above regret upper bound (RHS of Eq.~\eqref{eq:regret_bound}), the first term, $O(T^{1-\alpha/2}\sqrt{\log (|\cD|T)})$, is due to the bias in the feedback at heterogeneous participants, and the last two terms, $O(\sqrt{\gamma_T T})$ + $O(\sqrt{\gamma_T T^{1-\alpha}\log (|\cD|T)})$, are from the noisy feedback of each action as in the standard kernelized bandits (cf.~\cite{srinivas2009gaussian}). 
\High{Note that the first term (i.e., the regret caused by the bias) can be improved if one increases the number of sampled users in the learning process (i.e., choosing a larger value of $\alpha$). However, this would also result in a larger communication cost.}

\end{remark}



\begin{remark}[Maximum Uncertainty Reduction]
Recall that \texttt{DPBE} selects actions that have maximum variance for each batch (Eq.~\eqref{eq:standard_sigma_update}). Intuitively, variance at action $\x$ indicates the uncertainty about $f(\x)$, and thus, maximum-variance selection leads to maximum uncertainty reduction, which promotes exploration.  
\end{remark}

\High{
\begin{remark}[(Sub-)optimality]\label{rk:sub-optimality}
We first note that one natural lower bound for our setting is the  one for the standard setting of kernelized bandits, where the agent receives unbiased feedback after taking an action. In this setting,  the state-of-the-art lower bounds under two commonly-used kernel functions (SE and Mat\'{e}rn)\footnote{Note that even for the standard setting of kernelized bandits, there only exist lower bounds for these specific kernel functions rather than a general one in terms of the maximum information gain $\gamma_T$.} are summarized in Table~\ref{tab:bounds} (see Appendix~\ref{app:maximum_info_gain}), which can also serve as valid lower bounds for the setting we consider.
Recall that $\alpha>0$ is the user-sampling parameter that one can choose. We discuss the (sub-)optimality of our upper bounds in two cases: $\alpha \ge 1$ (i.e., the high-communication regime) and $\alpha \in (0,1)$ (i.e., the low-communication regime). (i) In the high-communication regime, the upper bound in \eqref{eq:regret_bound} now becomes $O(\sqrt{\gamma_T T})$, which is \emph{near-optimal} under both SE and Mat\'{e}rn kernels. In particular, if one plugs the best-known bounds on $\gamma_T$ for SE and Mat\'{e}rn kernels (as listed in the first column in Table~\ref{tab:bounds}; also see~\cite{vakili2021information}) into the regret upper bound $O(\sqrt{\gamma_T T})$, one can now have explicit regret upper bounds (as listed in the third column in Table~\ref{tab:bounds}), which match the corresponding lower bounds, up to only a logarithmic factor. (ii) In the low-communication regime, the first term in the regret upper bound (see Eq.~\eqref{eq:regret_bound} in Theorem~\ref{thm:regret_upper_bound}) that depends on $\alpha$ may be dominant and cannot be ignored. On the other hand, the existing lower bounds do not depend on $\alpha$ since they are derived under the standard setting of kernelized bandits, where user sampling is irrelevant. Therefore, an important open problem is to close the gap by deriving tighter lower and/or upper bounds that capture the effect of user sampling in the new setting with distributed biased feedback we consider. We leave it as our future work.
\end{remark}}

As a critical bottleneck of kernelized bandits algorithms, the computation complexity of \texttt{DPBE} algorithm is analyzed  in the following Theorem~\ref{thm:complexity}.
\begin{theorem}[Computation complexity]\label{thm:complexity}
The computation complexity of \texttt{DPBE} is at most $ O(\gamma_T T^{\alpha} + (|\cD|\gamma_T^3+\gamma_T^4)\log T)$.
\end{theorem}
\begin{proof}
Recall that $H_l$ is the number of batches in the $l$-th phase. Then, the computation complexity of the central agent in the $l$-th phase is upper bounded by the following:
$$O(H_l\cdot H_l^3 + H_l\cdot |\cD_l| H_l^2 + |U_l|H_l+  |\cD_l| H_l^2).$$
Specifically, for each $h \in [H_l]$ within phase $l$, the agent would compute the matrix inverse in Eq.~\eqref{eq:sigma_update}, the complexity of which is at most $O(h^3) \le O(H_l^3)$. With this matrix inverse result ready, the agent can solve the maximum-variance problem in Eq.~\eqref{eq:decision} with at most $O(|\cD_l|H_l^2)$ for each batch and determine the batch length $\cT_l(\ca_h)$ with $O(1)$ after we have the posterior variance. 
Since there is a total of $H_l$ batches for phase $l$, the total complexity up to this stage is $O(H_l\cdot H_l^3 + H_l\cdot |\cD_l| H_l^2)$. 
Finally, in the elimination stage for phase $l$, the agent first loads/aggregates all the feedbacks with $O(|U_l|H_l)$ and can again reuse the matrix inverse result so that only $O(|\cD_l| H_l^2)$ is required to eliminate all the bad arms.

Putting the two stages together, we have the above result. Thus, it remains to bound the number of batches $H_l$ within each phase $l$. 
Fortunately, inspired by \cite{calandriello2022scaling}, we are able to show that $H_l$ can be upper bounded by the maximum information gain. We state this result in Lemma~\ref{lem:bound_h} and provide the proof in Appendix~\ref{app:proof_communication_complexity}.

\begin{lemma}[Bound on $H_l$] \label{lem:bound_h}
For any phase $l$, the number of batches $H_l$ is at most $\frac{4\sigma^2C^2}{C^2-1}\gamma_T$.
\end{lemma}


We can get that the total number of phases is $O(\log T)$ and the total number of participants satisfies $O(T^{\alpha})$. Armed with all the above results, we arrive at our final computation complexity. 
\end{proof}




\begin{table}[!t]
\centering
\caption{Comparison of computation complexity under \texttt{DPBE} and three state-of-the-art algorithms.} 
\label{tab:complexity}
\begin{tabular}{c|c} 
    \toprule
    Algorithms & Complexity \\  
    \hline
      \texttt{GP-UCB} \cite{chowdhury2017kernelized}    & $O(|\cD|T^3)$ \\  
      \texttt{BBKB} \cite{calandriello2020near}   & $O(|\cD|T\gamma_T^2)$ \\ 
      \textsc{mini}-GP-Opt \cite{calandriello2022scaling} & $O(T + |\cD|\gamma_T^3 + \gamma_T^4)$ \\ 
      \textbf{\texttt{DPBE} (this paper)}  & $O(\gamma_T T^{\alpha}+(|\cD|\gamma_T^3+\gamma_T^4)\log T)$\\ 
      \bottomrule
\end{tabular}
\end{table}
\begin{remark}[Complexity comparison] \label{rk:computation_analysis} For comparison, we list the computation complexity of  the state-of-the-art algorithms for standard kernelized bandits in Table~\ref{tab:complexity}.
As we already know, \texttt{GP-UCB} has a computation complexity of $O(|\cD|T^3)$, because it requires computing the posterior mean and variance using $O(T^2)$ and then finds the action that maximizes the UCB function per step. Recently, BBKB in \cite{calandriello2020near} improves the time complexity to ($|\cD|T\gamma_T^2$), and later \textsc{mini}-GP-Opt in \cite{calandriello2022scaling} further reduces computation complexity to $O(T + |\cD|\gamma_T^3 + \gamma_T^4)$, which is currently the fastest no-regret algorithm. 
Although more feedback is needed to address the additional bias 
in our setting, our algorithm can still achieve an improvement with the highest order term being $O(\gamma_T T^{\alpha})$. This improvement comes from the fact that the participants help preprocess local reward observations before sending them out. 
\end{remark}

Meanwhile, the bound on $H_l$ also allows us to achieve a meaningful communication cost.
\begin{theorem}[Communication cost]\label{thm:communication}
DPBE incurs at most $O(\gamma_T T^{\alpha})$ 
communication cost.
\end{theorem}
The proof for Theorem~\ref{thm:communication} is also provided in Appendix~\ref{app:proof_communication_complexity}.

\begin{remark}[Communication cost when merging batches]
By further merging batches 
according to Remark~\ref{rk:shrink}, the \texttt{DPBE} algorithm incurs $O(\min\{\gamma_T,|\cD|\}T^{\alpha})$ communication cost;  
We highlight that the batch schedule strategy plays a key role in obtaining the above bounds. Otherwise, even merging rounds as Remark~\ref{rk:shrink} with the reformulated representation in Eqs.~\eqref{eq:sigma_update} and \eqref{eq:mu_update}, the dimension of the local feedback at each participant is $O(\min\{T_l, |\cD_l|\})$ in order to distinguish different actions, which leads to $O(\min\{T,|\cD|\}T^{\alpha})$ (vs. ours $O(\min\{\gamma_T,|\cD|\}T^{\alpha})$).



\end{remark}

\section{Differentially Private \texttt{DPBE}}\label{sec:DP-DPBE}

As privacy is also an important factor in distributed learning, it is critical to protect users' sensitive data when collecting and aggregating their feedback. 
\High{For example, in the dynamic pricing application, it is required that an adversary cannot infer a customer's private information (e.g., purchase or not) by observing the pricing mechanism set by the company. 
Moreover, users may require more stringent privacy protection in some applications --- users are not willing to share their perceived Quality-of-Experience (QoE) directly with the central controller in the cellular network configuration problem; citizens are not willing to reveal the information about their preference for a certain policy to the government.}
Formally, we adopt the concept of \emph{differential privacy} (DP)~\cite{dwork2014algorithmic} as the privacy metric. Thanks to the phase-then-batch schedule strategy in our algorithm, different DP trust models (e.g., central~\cite{dwork2014algorithmic}, local~\cite{zhou2020local}, and shuffle~\cite{cheu2019distributed}) can be applied through proper designs. In this section, we describe how to ensure DP under \texttt{DPBE} with a trusted agent (the central DP model) and also analyze the regret under such a DP model. Extensions of the differentially private \texttt{DPBE} algorithms in other DP models (e.g., the stronger local DP model) are presented in Appendix~\ref{app:DP-DPBE}.


\subsection{DP Definition and Algorithm}
In the central DP model, we assume that each participating user trusts the agent, and hence, the agent can collect their raw data (i.e., the local reward ${\y}_l^u$ in our case). The privacy guarantee is that any adversary with arbitrary auxiliary information cannot infer a particular user's data by observing the decisions of the agent. To achieve this privacy protection, the central DP model requires that the decisions of the agent on two neighboring sets of users (differing in only one user) are indistinguishable~\cite{dwork2006calibrating}. 
Formally, we have the following definition.
\begin{definition} (Differential Privacy (DP)). For any $\epsilon \geq 0$ and $\delta\in [0,1]$, a randomized algorithm $\cM$ is $(\epsilon, \delta)$-differentially private (or $(\epsilon, \delta)$-DP) if for every pair of $U, U^{\prime} \subseteq \mathcal{U}$ differing on a single participant and for any subset of output actions $\Z= [\z_1^{\top}, \dots, \z_T^{\top}]^{\top}$, we have
\begin{equation}
	\mathbb{P}[\mathcal{M}(U) = \Z] \leq e^{\epsilon} \mathbb{P}[\mathcal{M}(U) = \Z] + \delta.
	\end{equation}
\end{definition}
The parameters $\epsilon$ and $ \delta$ indicate how private $\cM$ is; the smaller, the more private. 
According to the post-processing property of DP (cf. Proposition~2.1 in \cite{dwork2014algorithmic}), it suffices to guarantee that the aggregator (Line~\ref{alg_aggregation} in Algorithm~\ref{alg:DPBE}) is $(\epsilon, \delta)$-DP. 	
To achieve this, the standard Gaussian mechanism can be applied by adding Gaussian noise to the aggregated distributed feedback. 
Then, the \emph{private} aggregated feedback for the chosen actions in the $l$-th phase becomes
\begin{equation}
	\Tilde{\y}_l = 
		 \bar{\y}_{l}+(\rho_1, \dots, \rho_{H_l}),  \label{eq:privatizer_cdp} 
\end{equation}
where 
$\rho_j\overset{\emph{i.i.d.}}{\sim} \mathcal{N}(0,\sigma^2_{nc})$ \Hi{and the variance $\sigma^2_{nc}$ (specified in Eq.~\eqref{eq:gaussian_noise}) is based on the (high-probability) sensitivity of of the average vector $\bar{\y}_l$.} 
In addition, we replace $\bar{\y}_l$ with $\Tilde{\y}_{l}$ in Eq.~\eqref{eq:privatizer_cdp} to obtain the private mean estimator: 
\begin{equation}
    \Tilde{\mu}_l(\x) \triangleq \vk(\x, \A_{H_l})^{\top}(\K_{\A_{H_l}\A_{H_l}} +\lambda \W_{H_l}^{-1})^{-1}\Tilde{\y}_l. \label{eq:mu_update_dp}
\end{equation}
The confidence width function is also updated by counting the uncertainty introduced by privacy noise as follows:
\begin{equation}
    \Tilde{w}_l(\x) \triangleq \sqrt{\frac{2 k(\x, \x)\log(1/\beta)}{|U_l|}} +   \sqrt{\frac{2 \Sigma^2_{H_l}(\x)\log(1/\beta)}{ |U_l|}} + B\Sigma_{H_l}(\x) +\sqrt{2 \sigma_n^2\log(1/\beta)}, \label{eq:confidence_width_dp}
\end{equation}
where $\sigma_n$ is related to the overall privacy noise and will be specified in the algorithm.
We present the differentially private version of \texttt{DPBE}, called \texttt{DP-DPBE}, in Algorithm~\ref{alg:DP-DPBE} (see Appendix~\ref{app:DP-DPBE}).

\subsection{Performance Guarantees}
In the following, we provide the main results of the \texttt{DP-DPBE} algorithm in terms of privacy guarantee and regret. We start by stating an additional assumption in Assumption~\ref{ass:one-time}. This one-time participation assumption is commonly used in private bandits (see, e.g.,~\cite{mishra2015nearly,sajed2019optimal,tenenbaum2021differentially,dubey2021no,pmlr-v162-chowdhury22a,chowdhury2022distributed}). 
To handle multiple-times participation, one can use (adaptive) composition theorem of differential privacy or group privacy~\cite{dwork2014algorithmic}, depending on different cases of returning users.\footnote{More specifically, if the same user only participates across multiple phases, one can use advanced composition; if the same user participates multiple times in the same phase, one can carefully bound the sensitivity or use group privacy directly.}

\High{
\begin{assumption}
\label{ass:one-time}
Each sampled user only participates in one phase of the learning process.
\end{assumption}}

Then, we present the privacy guarantee in Theorem~\ref{thm:dp} and provide the proof in Appendix~\ref{app:proof_dp_guarantee}.

\begin{theorem}[Privacy Guarantee]\label{thm:dp} 
Under Assumptions~\ref{ass:global_func}, \ref{ass:local_func}, \ref{ass:noise}, and \ref{ass:one-time}, for any $\epsilon>0$ and $\delta\in (0,1)$, the \texttt{DP-DPBE} algorithm (Algorithm~\ref{alg:DP-DPBE})  guarantees $(\epsilon, \delta)$-DP.
\end{theorem} 
As an additional Gaussian noise is injected to protect privacy, \texttt{DP-DPBE} suffers additional regret cost. We present its regret upper bound in Theorem~\ref{thm:regret_cdp}.  

\Hig{
\begin{theorem}[Regret of \texttt{DP-DPBE}]\label{thm:regret_cdp} 
Under Assumptions~\ref{ass:global_func}, \ref{ass:local_func}, and \ref{ass:noise}, the \texttt{DP-DPBE} algorithm (Algorithm~\ref{alg:DP-DPBE}) with $\beta=\frac{1}{|\cD|T}$ 
achieves the following expected regret: 
\begin{equation}
\begin{aligned}
\E[R(T)] &= O(T^{1-\alpha/2}\sqrt{\log (|\cD|T)})+ O\left(\frac{\ln(1/\delta)\gamma_TT^{1-\alpha}\sqrt{\log(|\cD|T)}}{\epsilon}\right).
\end{aligned}
\end{equation}
\end{theorem}
}
The full proof of Theorem~\ref{thm:regret_cdp} is provided in Appendix~\ref{app:proof_cdp}. Regarding this regret result, we make the following remark.
\begin{remark}[Privacy ``for free''] Comparing Theorem~\ref{thm:regret_cdp} with Theorem~\ref{thm:regret_upper_bound}, we see that the additional regret cost introduced by privacy noise is $\tilde{O}\left(\frac{\ln(1/\delta)\gamma_TT^{1-\alpha}}{\epsilon}\right)$, which is a lower order term compared to the first non-private term. This implies that 
our \texttt{DP-DPBE} algorithm enables us to achieve a privacy guarantee ``for free'' in the kernelized bandits setting. The same observation of achieving privacy ``for free'' is also observed in a recent study \cite{li2022differentially} that only considers linear bandits. However, our result is a strict generalization in the sense that it holds for general functions and recovers their result when considering a linear kernel.
\end{remark}

\section{Numerical Experiments}\label{sec:experiment}
We now evaluate our proposed approach empirically on \High{three} types of functions: 
\High{1) synthetic functions in the RKHS with an SE kernel, 2) standard benchmark functions (with an unknown RKHS norm)~\cite{simulationlib} , and 3) functions from a real-world dataset.}
We implement the algorithms in \texttt{python} 
and run the numerical experiments on a Dell desktop (Processor: Intel\textregistered Core i7 CPU, 8 cores; Memory: 32GB).



\subsection{Synthetic Function} \label{sec:synthetic_func} 
We follow \cite{janz2020bandit} to construct the global function $f$ from the RKHS by sampling $m=30d$ independent points, $\hat{\x}_1, \dots, \hat{ \x}_m$, uniformly on $[0,1]^d$, and $\hat{a}_1, \dots, \hat{a}_m$, uniformly on $[-1, 1]$, and defining $f(\x) = \sum_{i=1}^m \hat{a}_i k(\hat{\x}_i, \x)$ for all $\x \in \cD$, where $k$ is SE kernel with length-scale $l_{SE}=0.2$. The RKHS norm is $\Vert f\Vert_k^2 = \sum_{i=1}^m \sum_{j=1}^m \hat{a}_i\hat{a}_j k(\hat{\x}_i, \hat{\x}_j)$, which is assumed to be known. Each local reward function $f_u$, a random function sampled from a given Gaussian process, is generated by following Algorithm~1 in~\cite{kanagawa2018gaussian}. 
\High{In the simulations, we evaluate the algorithms in a more general setting with $f_u\sim \cG\cP(f(\cdot), v^2k(\cdot, \cdot))$, where $v^2$ is a scaling parameter that can be used to set a reasonable level of local bias (see Footnote~2).}



\subsubsection{Ablation Studies and Analysis}
First, we show that the \texttt{DPBE} algorithm that selects actions according to maximum variance reduction achieves sublinear regret, as shown in Figure~\ref{fig:regret_compare}. Then,  we perform
numerous ablation studies to confirm the efficacy of other two key components in our algorithm: user-sampling and batching strategy. To this end, we consider the corresponding variants of our algorithm. 
In this simulation, we perform $20$ runs for each algorithm by setting 
$|\cD|=100, d=3, C=1.6$, $\sigma=0.01$, $v=0.1$, $T=40000$, $\alpha=0.7$, $\beta=1/(|\cD|T)$ and $\lambda=\sigma^2/v^2$ and present the regret performance in Figure~\ref{fig:regret_compare}
and communication cost and runtime in Table~\ref{tab:runtime}.

\begin{figure*}[!t]
\centering
	\begin{minipage}[t]{0.4\textwidth}
		\centering 
    \includegraphics[width=1\textwidth]{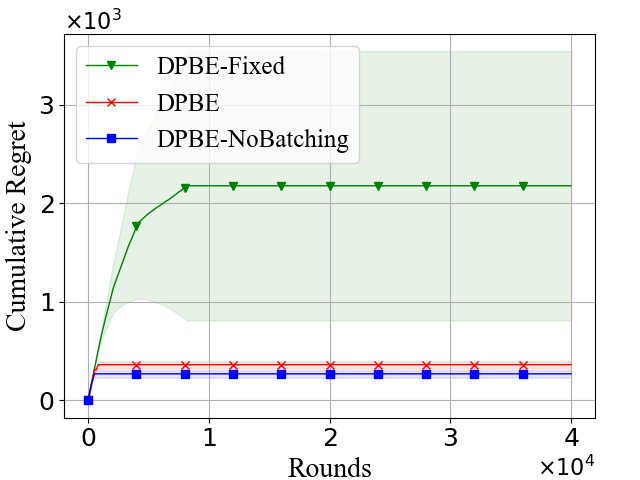} 
\caption{Comparison of regret performance 
on a synthetic function. The shaded area represents the standard deviation}
\label{fig:regret_compare}  
	\end{minipage}
	\qquad
	\begin{minipage}[t]{0.4\textwidth}
		\centering                  
	\includegraphics[width=1\textwidth]{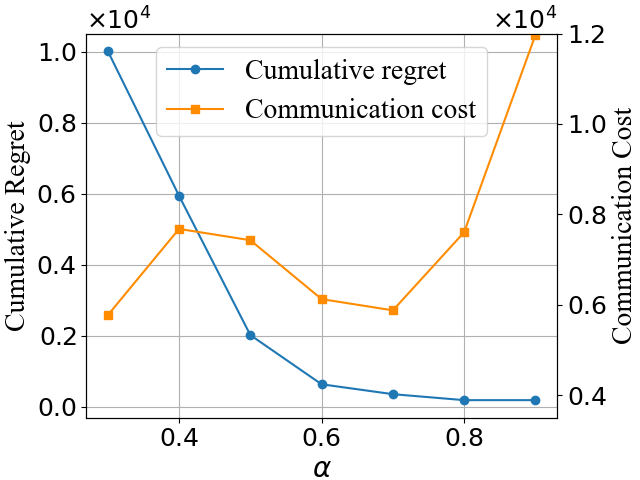} 
	\caption{The regret and communication cost under \texttt{DPBE} with different values of $\alpha$. 
	}
	 \label{fig:tradeoff}  
	\end{minipage}
\end{figure*}



\begin{table}[!t]
\centering
\caption{Comparisons of communication cost and running time under \texttt{DPBE}, \texttt{DPBE-Fixed}, and \texttt{DPBE-NoBatching} on a synthetic function.}
\label{tab:runtime}
\scalebox{0.99}{
\begin{tabular}{c|c|c} 
    \toprule
    Algorithms & Communication cost & Running time (seconds) \\  
    \hline
\texttt{DPBE}   & $5.87 \times 10^3$ & $0.12$ \\  
      \texttt{DPBE-Fixed} &  $5.87\times 10^3$   & $0.19$ \\ 
      \texttt{DPBE-NoBatching} & $1.81\times 10^4$ 
      &       $0.61$ \\ 
      \bottomrule
      \end{tabular} 
     }
\end{table}

\textbf{1) Importance of (exponentially-increasing) user-sampling}. To this end, we consider the first variation of \texttt{DPBE} with a fixed number of participants, called  \texttt{DPBE-Fixed}, where the number of participants in each phase is fixed at $|U| = \lfloor \frac{\sum_{l=1}^L |U_l|* N_{u,l} }{\sum_{l=1}^L N_{u,l}}\rfloor$ so as to have the same communication cost as \texttt{DPBE}.
From Figure~\ref{fig:regret_compare}, we observe that \texttt{DPBE} with exponentially-increasing user-sampling over phases performs much better than \texttt{DPBE-Fixed} with the same communication cost. It demonstrates that the exponentially-increasing user-sampling mechanism in \texttt{DPBE} is critical to striking a balance between regret and communication cost. From Table~\ref{tab:runtime}, we observe that \texttt{DPBE-Fixed} takes a little longer time than \texttt{DPBE}. This is mainly because \texttt{DPBE-Fixed} needs more phases to find the optimal action (i.e., $L$ is larger when $|\cD_L|=1$).


\textbf{2) Benefits of batching strategy. } 
To illustrate the impact of batching schedule strategy, we consider another variant of \texttt{DPBE} that does not employ batching strategy, called \texttt{DPBE-NoBatching}. In particular, it selects an action according to Eq.~\eqref{eq:decision} for each round in any phase. Without batching strategy, \texttt{DPBE-NoBatching} communicates local observations  directly without merging, and computes the posterior mean and variance according to standard update formula: Eq.~\eqref{eq:standard_mu_update} and  Eq.~\eqref{eq:standard_sigma_update} respectively; 
From Figure~\ref{fig:regret_compare}, we observe that \texttt{DPBE}, similar to other \emph{rare-switching} algorithms \cite{abbasi2011improved}, achieves a slightly worse regret performance than \texttt{DPBE-NoBatching}. However, as shown in Table~\ref{tab:runtime}, it significantly  saves  communication cost ($\sim3 \times$) by merging local observations in batches and reduces computation time  ($\sim5\times$) by shrinking the dimension of posterior reformulations. 

\subsubsection{Regret-communication Tradeoff}
We now turn to investigate the regret-communication tradeoff captured by the user-sampling parameter $\alpha$, as shown in Theorem~\ref{thm:regret_upper_bound}. 

Consider $\alpha \in \{0.3, 0.4, 0.5, 0.6, 0.7, 0.8, 0.9\}$. The cumulative regret and total communication cost of \texttt{DPBE} with different values of $\alpha$ are presented in Figure~\ref{fig:tradeoff}. As expected, while a larger value of $\alpha$ yields a lower regret, it generally results in a higher communication cost. Notice that \texttt{DPBE} incurs slightly higher communication cost when $\alpha = \{0.4, 0.5, 0.6\}$ compared to $\alpha=0.7$, this is mainly because \texttt{DPBE} with a smaller value of $\alpha$ needs more phases to find the optimal action (i.e., $L$ is larger when $|\cD_L|=1$). One can tune the user-sampling parameter $\alpha$ to achieve a better regret-communication cost accordingly, e.g., $\alpha=0.7$ for this synthetic function setting. 



\subsubsection{Regret-privacy Tradeoff} Finally, we evaluate the performance of the differentially private \texttt{DPBE}, i.e.,  \texttt{DP-DPBE}, 
and present the result in Figure~\ref{fig:dp_dpbe}. \Hi{Figure~\ref{fig:regret_epsilon} shows how the cumulative regret at the end of $T=10^6$ rounds varies with different values of privacy parameters $\epsilon \in \{5,10,15,20,25,30\}$ and $\delta=o(1/T)=10^{-6}$, which reveals a tradeoff between regret and privacy. Figures~\ref{fig:regret_time} 
shows the regret performance of \texttt{DPBE} and \texttt{DP-DPBE} with privacy parameters $\epsilon = 15$ and $\delta = 10^{-6}$.} 
We observe that although \texttt{DP-DPBE} adds extra noise to protect privacy, it can still achieve no-regret (i.e., $\lim_{T\to \infty}  \frac{R(T)}{T}\to 0$). Indeed, to protect privacy, \texttt{DP-DPBE} requires much more time to find the optimal action, which is the typical regret-privacy tradeoff. However, for a large $T$, the gap compared to the non-private one is small, which also validates the privacy ``for-free'' result.

\begin{figure*}[!t] 
\centering
\subfigure[
	]{
	\label{fig:regret_epsilon}
	\includegraphics[width=0.4\textwidth]{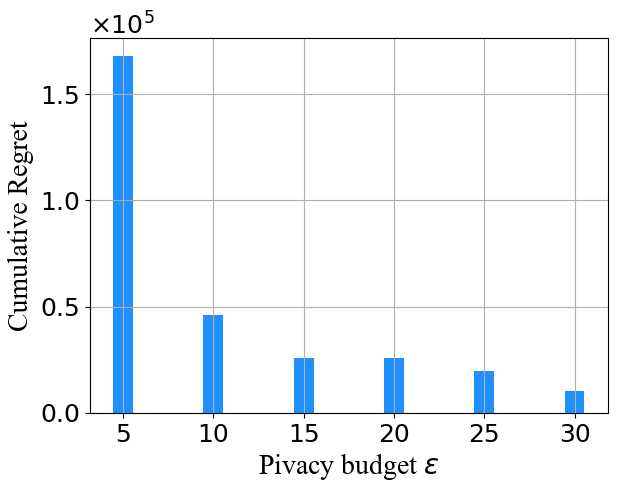}}
	\quad
	\subfigure[]{
	\label{fig:regret_time}
	\includegraphics[width=0.4\textwidth]{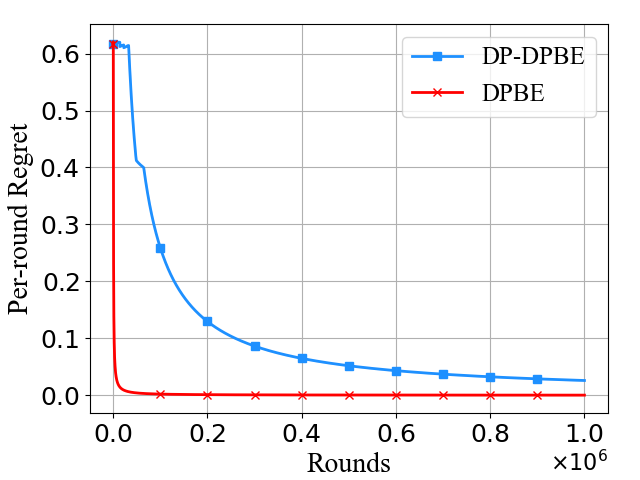}}
	\caption{Performance of \texttt{DP-DPBE}.   
	\Hi{(a) Final cumulative regret vs. the privacy budget $\epsilon$ with $\delta=10^{-6}$; (b) Per-round regret vs. time with parameters $\epsilon=15$ and $\delta=10^{-6}$}.
}
	\label{fig:dp_dpbe} 
	
\end{figure*}



\subsection{Standard Benchmark Functions} \label{sec:benchmark_function}
In addition, we study the performance of \texttt{DPBE} 
on standard optimization benchmark functions. 
This corresponds to a more realistic setting where the RKHS norm of the target function is unknown in advance. In particular, we use three common functions in global optimization problems~\cite{simulationlib}
: (a) Sphere function, (b) Six-hump Camel function, and (c) Michalewicz function, and provide the performance comparison of \texttt{DPBE-Fixed}, \texttt{DPBE}, and \texttt{DPBE-NoBatching} in Figure~\ref{fig:regret_optimization_baselines} and Table~\ref{tab:communication_time_baseline}. In the simulations, we scale the range of the function values to $[-1,1]$ and use RKHS norm $B=1$ in the algorithms as in \cite{janz2020bandit}. Without knowing the exact kernel of the target function, each local reward function $f_u$ is constructed by sampling a function from the GP $\cG\cP(f(\cdot), v^2k(\cdot, \cdot))$, where we choose $v^2=0.001$ and use the SE kernel with $l_{SE}=0.2$. 
In addition, we set $T=4\times10^4$ and $|\cD|=100$ and run each algorithm on each function for $20$ times.
\begin{figure*}[!t] 
\centering
	\subfigure[Sphere function]{
	\label{fig:sphere}
	\includegraphics[width=0.4\textwidth]{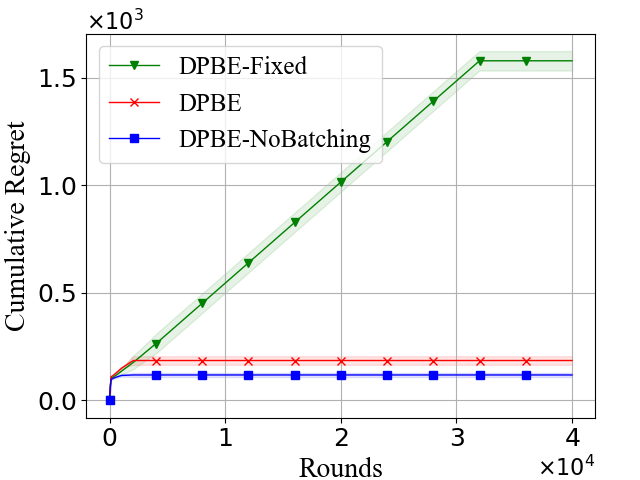}}
	\quad
	\subfigure[Six-Hump Camel function]{
	\label{fig:six-hump}
	\includegraphics[width=0.4\textwidth]{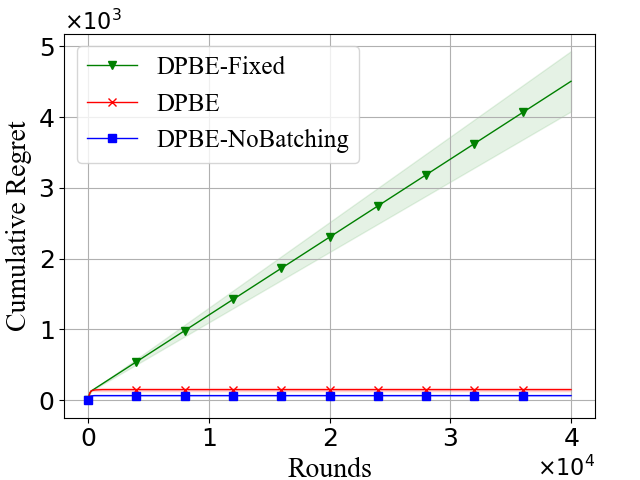}}
	\subfigure[Michalewicz function]{
	\label{fig:michalewicz}
	\includegraphics[width=0.4\textwidth]{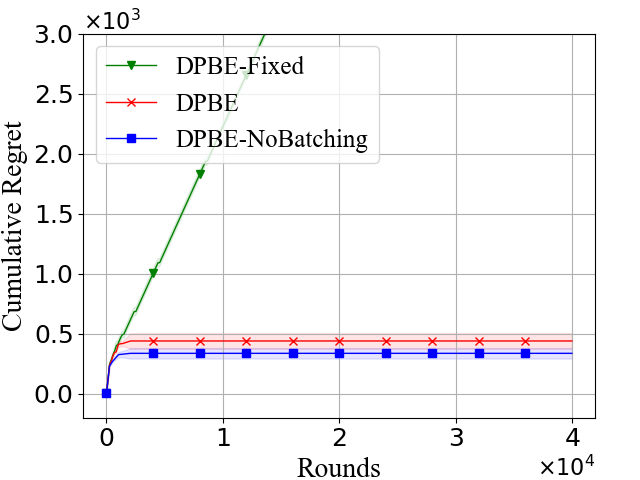}}
 \quad
	\subfigure[Light sensor data]{
	\label{fig:light}
	\includegraphics[width=0.4\textwidth]{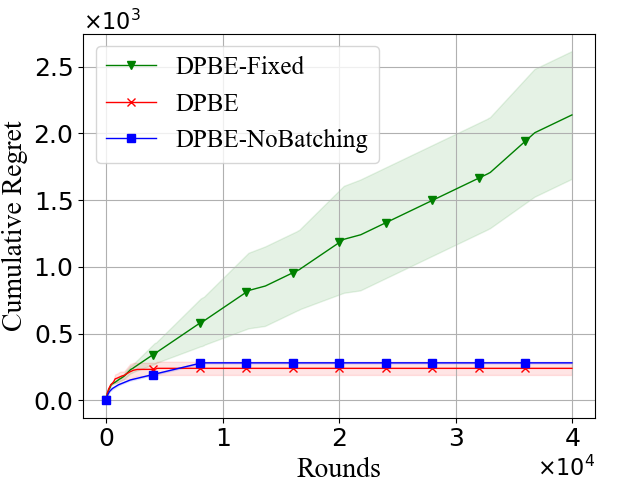}}
	\caption{Comparison of regret performance under \texttt{DPBE}, \texttt{DPBE-Fixed}, and \texttt{DPBE-NoBatching} on four functions. 
	(a) Sphere function. Settings: $d=3,  C=1.6, \sigma=0.01, \lambda=\sigma^2/v^2, \alpha=0.7$; (b) Six-Hump Camel function. Settings: $d=2, C=1.6, \sigma=0.01, \lambda=\sigma^2/v^2, \alpha=0.7$; (c) Michalewicz function. Settings: $d=2, {C=1.6}, \sigma=0.1,   \lambda=\sigma^2/v^2, \alpha=0.6$; (d) Function from light sensor data. Settings: $d=2, C=1.42, \sigma=0.01, \lambda=\sigma^2/v^2, \alpha=0.8$. 
	}
	\label{fig:regret_optimization_baselines}  
	\vspace{-5mm}
\end{figure*}

From Figure~\ref{fig:regret_optimization_baselines} and Table~\ref{tab:communication_time_baseline}, we observe similar results to those of the synthetic function with the same kernel. First, compared to \texttt{DPBE-Fixed} that incurs the same communication cost, \texttt{DPBE} might perform slightly worse at the very beginning (e.g., Figure~\ref{fig:sphere}) but eventually achieves a much smaller regret. Note that \texttt{DPBE-Fixed} may not be able to find the optimal action by the end of $T$ (e.g., Figure~\ref{fig:six-hump}). This phenomenon strengthens our argument on the exponentially-increasing user-sampling mechanism in \texttt{DPBE}. While \texttt{DPBE-NoBatching} has slightly better regret performance than \texttt{DPBE}, it incurs much higher communication cost ($5\sim13\times$) and requires a much longer time ($6\sim 23\times$, see running time column in Table~\ref{tab:communication_time_baseline}), which demonstrates the key benefits of the batching strategy in improving communication efficiency and computation complexity. 

In addition, we also evaluate the regret-privacy tradeoff under \texttt{DP-DPBE}. Due to space limitations, we present the numerical results in Appendix~\ref{app:experiments} \Hi{(see Figures~\ref{fig:dp_functions_epsilon}~and~\ref{fig:dp_functions_delta_106}).}
\begin{table}[!t]
\centering
\caption{Communication cost and running time under \texttt{DPBE}, \texttt{DPBE-Fixed}, and \texttt{DPBE-NoBatching}} 
\label{tab:communication_time_baseline}
\scalebox{0.99}{
\begin{tabular}{c|c|c|c} 
    \toprule
    Function & Algorithm & Communication cost & Running time (seconds) \\  
    \hline
     \multirow{4}*{Sphere} &
      \texttt{DPBE}   & $1.49 \times 10^3$ & $0.07$ \\  
      ~&\texttt{DPBE-Fixed} &  $1.49\times 10^3$   & $0.12$ \\ 
      ~&\texttt{DPBE-NoBatching} & $6.16\times 10^3$ & $0.69$ \\ 
    \hline
    \multirow{4}*{Six-Hump Camel} &
      \texttt{DPBE}   & $1.26\times 10^3$ & $0.03$ \\  
      ~&\texttt{DPBE-Fixed} &  $1.26\times 10^3$   & $0.12$ \\ 
      ~&\texttt{DPBE-NoBatching} & $1.45\times 10^4$ & $0.17$ \\ 
    \hline
    \multirow{4}*{Michalewicz} &
      \texttt{DPBE}   & $2.06 \times 10^3$ & $0.06$ \\  
      ~&\texttt{DPBE-Fixed} &  $2.06\times 10^3$   & $0.14$ \\ 
      ~&\texttt{DPBE-NoBatching} & $2.73\times 10^4$ & $0.49$ \\ 
    \hline
    \multirow{4}*{Light Sensor Data} &
      \texttt{DPBE}   & $5.17 \times 10^3$ & $0.22$ \\  
      ~&\texttt{DPBE-Fixed} &  $5.17\times 10^3$   & $0.28$ \\ 
      ~&\texttt{DPBE-NoBatching} & $2.73\times 10^4$ & $5.20$\\
      \bottomrule
      \end{tabular} 
    }
\end{table}
\High{
\subsection{Functions from Real-World Data}\label{sec:real_world_data}
We also evaluate the performance of \texttt{DPBE} on a function from a real-world dataset, where there is no explicit closed-form expression. 

\textbf{Light Sensor Data.} We use the light sensor data collected from the CMU Intelligent Workplace in November 2005, which is available online~\cite{lightsensor}. 
It contains locations of $41$ sensors, $601$ training samples, and $192$ testing samples. Following \cite{srinivas2009gaussian, chowdhury2017kernelized, zhou2022kernelized}, we compute the empirical covariance matrix of the training samples and use it as the kernel matrix in the algorithm. Here, for each location $\x$, we let $f(\x)$ be the average of the normalized sample readings at $\x$ and set $B=\max_{\x}f(\x)$ in the algorithm. For this function (from real data), we construct each local function $f_u$ by sampling a function from a Gaussian process with mean $f$ and the kernel constructed above, and set the noise in the local feedback as $\sigma=0.01$ and the bias in each local feedback as $v=0.1$. We run \texttt{DPBE} with input parameters $\alpha=0.7, \beta=1/(|\cD|T)$, and $\lambda=\sigma^2/v^2$, and present the regret performance in Figure~\ref{fig:light} and communication cost and running time in Table~\ref{tab:communication_time_baseline}. The observations are qualitatively similar to those made in simulations on other functions: \texttt{DPBE} outperforms \texttt{DPBE-Fixed} in regret  given the same communication cost and achieves a regret close to \texttt{DPBE-NoBatching}, which has much longer running time. Besides, we also run \texttt{DP-DPBE} on this real-world dataset and present the results in Appendix~\ref{app:experiments} \Hi{(see Figures~\ref{fig:light_dp_epsilon}~and~\ref{fig:light_dp_delta})}, which validates the regret-privacy tradeoff. 
}

\High{\section{Comparison with the State-of-the-Arts} \label{sec:compare_benchmarks}
}
\subsection{Discussion}\label{sec:discussion}
We now consider an alternative way of addressing kernelized bandits with distributed biased feedback.
\High{One may incorporate the local bias as another level of noise added to the noise in the rewards as a new noisy measurement of the global function $f$ with a larger variance.}
In this case, the state-of-the-art algorithms for the traditional kernelized bandits \cite{srinivas2009gaussian, chowdhury2017kernelized} 
may be adapted to our setting. However, they have some key limitations. 

\High{Consider two representative state-of-the-art algorithms: \texttt{GP-UCB}~\cite{chowdhury2017kernelized} and \texttt{BPE}~\cite{li2022gaussian}.}
\High{\texttt{GP-UCB} is one of the most commonly used algorithms for standard kernelized bandits, It was proposed in \cite{srinivas2009gaussian} and improved in \cite{chowdhury2017kernelized}. By resorting to the Gaussian process surrogate model (see Section~\ref{sec:preliminary}), \texttt{GP-UCB} adaptively selects the action with the maximal \emph{upper confidence bound} in each round based on historical observations up to the current round. 
\texttt{BPE} is a batch-based algorithm that eliminates suboptimal actions batch by batch,
and within each batch, actions are chosen independently from reward observations. 
In the following, we compare our proposed \texttt{DPBE} algorithm with  \texttt{GP-UCB} and \texttt{BPE} (adapted to our setting) and show their limitations in user-sampling, communication cost, and computation complexity.} 

First, both \texttt{GP-UCB} and \texttt{BPE} require to collect feedback from one user per step, which results in $T$ users involved in the learning process. 
In practice, even though there is a large population, not all users are willing to send their feedback. Hence, it may not be feasible to collect feedback from too many users. In our algorithm, instead of sampling more users to reduce the overall uncertainty, we ask each sampled user (who is more willing to participate) to participate in more rounds and send their feedback. In this way, we alleviate the user-sampling burden by letting the participating users collect more reward samples of the chosen actions. However, due to the bias in the feedback of each user, we could not just sample one user and then let her report the feedback during the entire horizon. We need to balance the tradeoff between sampling more users and letting the users participate in more rounds. 

Second, by collecting feedback in each round, \High{both \texttt{GP-UCB} and \texttt{BPE} incur a very high communication cost of $T$.} 
Instead, we employ a phase-based communication protocol where feedback corresponding to any particular action at each participant is averaged and only communicated at the end of each phase. Then, the total communication cost depends on the number of phases, the number of distinct actions in each phase, and the number of sampled users. The smaller each of these three factors, the smaller the communication cost.  By carefully designing the algorithm, we can reduce the communication cost to $O(\min\{\gamma_T, |\cD|\}T^{\alpha})$, where $\alpha\in (0,1)$ is the user-sampling parameter one can choose.
%

Finally, at each round $t$, \texttt{GP-UCB} finds the decision action $\x_t$ that maximizes an acquisition function (specifically, the UCB index, which is the sum of the posterior mean and variance). Note that obtaining the posterior mean and variance requires computing matrix inverse (see Eqs.~\eqref{eq:standard_mu_update} and \eqref{eq:standard_sigma_update}), which still has a computation complexity of $O(t^2)$ even using rank-one recursive updates \cite[Appendix~7]{chowdhury2017kernelized}. 
Hence, the overall computation complexity of \texttt{GP-UCB} is $O(|\cD|T^3)$. \High{Similarly, \texttt{BPE} may also compute the posterior variance using the rank-one recursive update within each batch, and then the total computation complexity depends on the batch size and the number of batches. As in \cite{li2022gaussian}, the batch size is updated as $N_i = \sqrt{T\sqrt{N_{i-1}}}$, initialized with $N_0=1$, which results in  $\lceil \log\log (T) \rceil $ batches in total. Therefore, the computation complexity of \texttt{BPE} is $O(|\cD|T^3)$.}
In our design, we employ the batch schedule strategy and reformulate the posterior mean and variance as Eqs.~\eqref{eq:mu_update} and \eqref{eq:sigma_update}, where the dimension of the matrix becomes much smaller. This leads to a much smaller overall computation complexity of $O(\gamma_T T^{\alpha}+(|\cD|\gamma_T^3+\gamma_T^4)\log T)$.
\subsection{Empirical Performance}
In this subsection, we evaluate the empirical performance of \texttt{DPBE} with different values of user-sampling parameter $\alpha$ compared to \texttt{GP-UCB} and \texttt{BPE}. 

The simulations are run on the same three types of functions as in the preceding section: the synthetic function in Section~\ref{sec:synthetic_func}, the standard benchmark functions in Section~\ref{sec:benchmark_function}, and the function from light sensor data in Section~\ref{sec:real_world_data}. 
Due to space limitation, we only present the results of the synthetic function here and put the results of the latter two types of functions in Appendix~\ref{app:experiments}. 


Consider\footnote{Note that the smaller the value of $\alpha$, the larger the cumulative regret. In Figure~\ref{fig:vs_gpucb_syn}, we omit the regret performance when $\alpha <0.4$ since they are much larger than others.} 
$\alpha \in \{0.4, 0.5, 0.6, 0.7,$ $ 0.8, 0.9\}$ for \texttt{DPBE}. We show the empirical regret performance of all algorithms in Figure~\ref{fig:vs_gpucb_syn} and the running time in Table~\ref{tab:wallclock}. From Figure~\ref{fig:vs_gpucb_syn}, we observe that the empirical regret performance of \texttt{DPBE} can be fairly close to or even better than that of \texttt{GP-UCB} and \texttt{BPE} via properly choosing parameter $\alpha$. However, it consumes much less time for \texttt{DPBE} with each $\alpha\in\{0.4, 0.5, 0.6, 0.7, 0.8, 0.9\}$ than both \texttt{GP-UCB} and \texttt{BPE}. For example, while \texttt{DPBE} takes about $0.15$  second in most scenarios, \texttt{GP-UCB} takes more than $5$ seconds, which is more than $30$ times slower. \texttt{BPE} takes around $27$ seconds, which is even slower.

Recall the empirical communication cost of \texttt{DPBE} with different values of $\alpha$ shown in Figure~\ref{fig:tradeoff}. While the communication cost of \texttt{GP-UCB} and \texttt{BPE} is $4\times 10^4$ (specifically, one feedback per round), \texttt{DPBE} incurs a much smaller communication cost even when $\alpha=0.9$ ($4\times 10^4$ vs. $1.19\times 10^4$). 


In summary, the comparison of empirical performance under \texttt{DPBE} with \texttt{GP-UCB} and \texttt{BPE} demonstrates the significant improvements of \texttt{DPBE} in terms of communication cost and computation complexity, although little regret performance is sacrificed when $\alpha$ is not big enough.  

\begin{table}[!t]
\centering
\caption{Comparison of running time (seconds) under \texttt{GP-UCB}, \texttt{BPE}, and \texttt{DPBE} with different values of $\alpha$.}
\label{tab:wallclock}
	\scalebox{0.99}{
\begin{tabular}{c|c|c|c|c|c|c|c|c} 
    \toprule
    \multirow{2}*{Algorithms}  & 
    \multicolumn{6}{c|}{ \texttt{DPBE}}
    &\multirow{2}*{\texttt{GP-UCB}} & \multirow{2}*{\texttt{BPE}}\\ 
    \cline{2-7}
    ~ & $\alpha=0.4$ & $\alpha=0.5$& $\alpha=0.6$& $\alpha=0.7$& $\alpha=0.8$ & $\alpha=0.9$&~&~ \\
    \hline
     \multirow{2}*{Running time}  &
     \multirow{2}*{$0.24$}   & \multirow{2}*{$0.19$} & \multirow{2}*{$0.14$} & \multirow{2}*{$0.12$} & \multirow{2}*{$0.13$}  &
     \multirow{2}*{$0.17$}&
     \multirow{2}*{$5.32$} & \multirow{2}*{$27.49$}\\
   & ~ & ~   & ~& ~ & ~& ~ & ~ & ~\\
      \bottomrule
\end{tabular}
}
\end{table}

\begin{figure}[!t]
\centering
\includegraphics[width=0.4\linewidth]{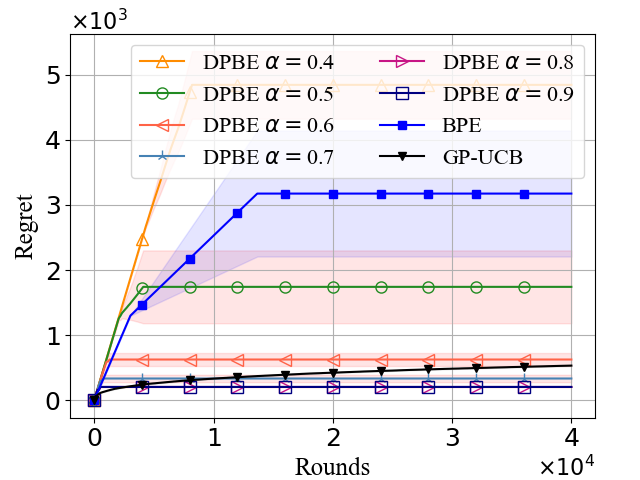}
\caption{Regret performance comparison of  \texttt{GP-UCB}, \texttt{BPE}, and \texttt{DPBE} with different values of $\alpha$. 
}
\label{fig:vs_gpucb_syn}
\vspace{-3mm}
\end{figure}

\section{Conclusion}\label{sec:conclusion}
In this paper, we studied a new kernelized bandit problem with distributed biased feedback, where the feedback of the unknown objective function is biased due to user heterogeneity. To learn and optimize the unknown function using distributed biased feedback, we proposed the learning with communication framework. Considering the communication cost for collecting feedback and the computational bottleneck of kernelized bandits, we carefully designed the distributed phase-then-batch-based elimination (\texttt{DPBE}) algorithm to address all the new challenges. Specifically, \texttt{DPBE} selects actions according to maximum variance reduction, reduces bias via user-sampling, and improves communication efficiency and computation complexity via the batching strategy. Furthermore, we showed that \texttt{DPBE} achieves a sublinear regret while being scalable in terms of communication efficiency and computation complexity. Finally, we generalized \texttt{DPBE} to incorporate various differential privacy models to ensure privacy guarantees for participating users.

\textbf{Future work.} While we proposed a new \texttt{DPBE} algorithm to address the new challenges that arise in our problem setup, it would be worthwhile to explore other batch-based algorithms and investigate whether one can further improve the tradeoff among regret, communication efficiency, and computation complexity. 
In addition, as discussed in Remark~\ref{rk:sub-optimality}, 
the lower bound derived for the standard kernelized bandits is also a valid lower bound for our problem. We show that our algorithm, if sampling a sufficient number of users, can achieve this lower bound. In general, however, it is 
an important open problem to close the gap by deriving tighter lower and/or upper bounds that capture the effect of user sampling in our new setting. We leave it as our future work.

\section{ACKNOWLEDGMENTS}
We thank our shepherd, Giulia Fanti, and the anonymous paper reviewers for their insightful feedback. We also thank Duo Cheng for fruitful discussions. This work is supported in part by the NSF grants under CNS-2112694 and
CNS-2153220.



\appendix

\printbibliography 
\newpage
\begin{appendix}
\section{Kernelized Bandits: Useful Definitions and Useful Results} \label{app:kb_auxiliary_results}
\subsection{Example Kernel Functions} 
In the following, we list some commonly used kernel functions $k: \cD \times \cD \to \R$:
\begin{itemize}
    \item Linear kernel: $k_{\text{lin}}(\x, \x^{\prime}) = \x^{\top}\x^{\prime}$,
    \item Squared exponential kernel: $k_{\text{SE}}(\x, \x^{\prime}) = \exp{\left(-\frac{\Vert \x -\x^{\prime}\Vert}{2l^2}\right)}$,
    \item Mat\'{e}rn kernel: $k_{\text{Mat}}(\x, \x^{\prime}) = \frac{2^{1-\nu}}{\Gamma(\nu)}
    \left(\frac{\sqrt{2\nu}\Vert \x -\x^{\prime}\Vert}{l}\right)J_{\nu}\left(\frac{\sqrt{2\nu}\Vert \x -\x^{\prime}\Vert}{l}\right)$,
\end{itemize}
where $l$ denotes the length-scale hyperparameter, $\nu>0$ is an additional hyperparameter that dictates the smoothness, and $J_{\nu}$ and $\Gamma_{\nu}$ denote the modified Bessel function and the Gamma function, respectively \cite{rasmussen2006gaussian}.

\subsection{Maximum Information Gain for Different Kernels}\label{app:maximum_info_gain}

We present the bounds on $\gamma_T$ and regret under two common kernels below in Table~\ref{tab:bounds}.
\begin{table}[!t]
\centering
\caption{Bounds on $\gamma_T$ and Regret under Two Common Kernels \cite{vakili2021information} }
\label{tab:bounds}
\scalebox{0.95}{
\begin{tabular}{c|c|c|c} 
    \toprule
    Kernel & \texttt{Upper Bound on $\gamma_T$} & \texttt{Regret Lower Bound} & \texttt{Regret Upper Bound $O(\sqrt{\gamma_T T})$} \\ 
    \hline
    SE &$O\left(\log^{d+1} (T)\right)$ & $\Omega \left(\sqrt{T\log^{\frac{d}{2}} (T)} \right)$ & ${O} \left(\sqrt{T\log^{d+1} (T)} \right)$ \\  
    Mat\'{e}rn-$\nu$   & $O \left(T^{\frac{d}{2\nu+d}}\log^{\frac{2\nu}{2\nu+d}} (T)\right)$ & $\Omega \left(T^{\frac{\nu + d}{2\nu +d}} \right)$ &  ${O} \left(T^{\frac{\nu + d}{2\nu +d}}\log^{\frac{\nu}{2\nu+d}}(T) \right)$\\ 
    \bottomrule
\end{tabular}
}
\end{table}
\subsection{Useful Results}
\begin{lemma}[Sum of variance. Lemma~6 in \cite{ray2019bayesian}]\label{lem:information_gain}
Let $\X_t=[\x^{\top}_1, \dots,\x^{\top}_t]^{\top}$, and $\sigma^2_t(\x) \triangleq k(\x, \x) - \vk(\x, \X_{t})^{\top}(\K_{\X_{t}\X_{t}}+\lambda \I)^{-1}\vk(\x,\X_{t})$ for any $\x\in\cD$. Then, we have
\begin{equation}
    \sum_{s=1}^t\sigma^2_s(x_s)     \leq
    \lambda \ln |\lambda^{-1}\K_{\X_t\X_t}+\I|\leq 2\lambda\gamma_t. 
\end{equation}

\end{lemma}
\begin{lemma}[Proposition A.1 in \cite{calandriello2022scaling}/Lemma~4 in \cite{calandriello2020near}]\label{lem:variance_ratio}
For any kernel $k$, set of points $\X_{\tau}$, $\x\in\cD$, and $\tau^{\prime} < \tau$, we have
\begin{equation}
    1\leq \frac{\sigma^2_{\tau^{\prime}}(\x)}{\sigma^2_{\tau}(\x)} \leq 1+ \sum_{s=\tau^{\prime}+1}^{\tau} \sigma^2_{\tau^{\prime}}(\x_{s}).
\end{equation}
\end{lemma}
\subsection{Formulation in Feature Space}
For several of the proofs, it will be useful to introduce the so-called feature space (RKHS) formulation of any point in the primal space $\R^d$. In particular, we define a
feature map $\varphi(\x) = k(\x, \cdot)$ where $\varphi: \cD \to \cH_k$ with $\cH_k$ being the reproducing kernel Hilbert space (RKHS) associated with kernel function $k$. According to the properties of RKHS, we have the following observations:
\begin{itemize}
    \item For any $\x, \x^{\prime}$, $k(\x, \x^{\prime}) = \varphi(\x)^{\top}\varphi( \x^{\prime})$.
    \item For any function $f\in \cH$, $f(\x) = \langle f, \varphi(\x)\rangle = \varphi(\x)^{\top} f$.
    \item Fundamental linear algebra equality
    \begin{equation}
        (BB^{\top}+\lambda \I)^{-1}B = B(B^{\top}B+\lambda \I)^{-1}.
    \end{equation}
    \item Define $\Phi_h \triangleq [\varphi(\ca_1)^{\top}, \dots, \varphi(\ca_h)^{\top}]^{\top}$. Then, the kernel matrix $\K_{\A_h\A_h} = \Phi_h\Phi_h^{\top}$ and $k(\x, \A_h)=\Phi_{h}\varphi(\x)$, and the variance function $\Sigma_h^2(\cdot)$ represented in the feature space is the following: 
\begin{equation}
 \begin{aligned}
  \Sigma_{h}^{2}(\x) &= k(\x, \x) - \vk(\x, \A_h)^{\top}(\K_{\A_h\A_h}+\lambda \W_{h}^{-1})^{-1}\vk(\x,\A_h) \\
& = \varphi(\x)^{\top}\varphi(\x) - \varphi(\x)^{\top} \Phi_{h}^{\top}(\Phi_{h}\Phi_{h}^{\top}+\lambda \W_{h}^{-1})^{-1} \Phi_{h}\varphi(\x).
 \end{aligned}
\end{equation}
    \item Consider any phase $l$. 
    Recall that $H_l$ is the number of batches in the $l$-th phase. Define $\Phi_{H_l}\triangleq [\varphi(\ca_1)^{\top}, \dots, \varphi(\ca_{H_l})^{\top}]^{\top}$. Then, the kernel matrix $\K_{\A_{H_l}\A_{H_l}} = \Phi_{H_l}\Phi_{H_l}^{\top}$ and $k(\x, \A_{H_l})=\Phi_{H_l}\varphi(\x)$.
    \item  Define $\Phi_{\tau}\triangleq [\varphi(\x_{t_l+1})^{\top}, \dots, \varphi(\x_{t_l+\tau})^{\top}]^{\top}$. Then, the kernel matrix $\K_{\X_{\tau}\X_{\tau}} = \Phi_{\tau} \Phi_{\tau}^{\top}$, 
    $k(\x, \X_{\tau}) = \Phi_{\tau}\varphi(\x)$, and the variance function $\sigma_{\tau}^2(\cdot)$ represented in the feature space is the following:
 \begin{equation}
 \begin{aligned}
  \sigma_{\tau}^{2}(\x) &= k(\x, \x) - \vk(\x, \X_{\tau})^{\top}(\K_{\X_{\tau}\X_{\tau}}+\lambda \I)^{-1}\vk(\x,\X_{\tau}) \\
& = \varphi(\x)^{\top}\varphi(\x) - \varphi(\x)^{\top} \Phi_{\tau}^{\top}(\Phi_{\tau}\Phi_{\tau}^{\top}+\lambda \I)^{-1} \Phi_{\tau}\varphi(\x)\\
& = \varphi(\x)^{\top}\varphi(\x) - \varphi(\x)^{\top} (\Phi_{\tau}^{\top}\Phi_{\tau}+\lambda \I)^{-1} \Phi_{\tau}^{\top}\Phi_{\tau}\varphi(\x)\\
& = \varphi(\x)^{\top}(\Phi_{\tau}^{\top}\Phi_{\tau}+\lambda \I)^{-1}(\Phi_{\tau}^{\top}\Phi_{\tau}+\lambda \I)\varphi(\x) - \varphi(\x)^{\top} (\Phi_{\tau}^{\top}\Phi_{\tau}+\lambda \I)^{-1} \Phi_{\tau}^{\top}\Phi_{\tau}\varphi(\x)\\
& = \lambda  \varphi(\x)^{\top} (\Phi_{\tau}^{\top}\Phi_{\tau}+\lambda \I)^{-1}\varphi(\x). \label{eq:sigma_feature_space}
 \end{aligned}
\end{equation}
\item  Define $\Phi_{T_l}\triangleq [\varphi(\x_{t_l+1})^{\top}, \dots, \varphi(\x_{t_l+T_l})^{\top}]^{\top}$. Then, the kernel matrix $\K_{\X_{T_l}\X_{T_l}} = \Phi_{T_l} \Phi_{T_l}^{\top}$, 
    $k(\x, \X_{T_l}) = \Phi_{T_l}\varphi(\x)$, and $f(\X_{T_l}) = \Phi_{T_l}f$.
\end{itemize}

\section{Auxiliary Results and Proofs for Regret Analysis}
\subsection{Equivalent Representations} 
Consider any phase $l$. 
We use $\tau$ to denote the within-phase time index, i.e., $\tau  \in \{1, \cdots, T_l\}$. Define $\tau_h$ as the last within-phase time index of the $h$-th batch, i.e., $\tau_h \triangleq \max\{\tau: t_l+\tau \in \cT_l(\ca_h)\} $. 
Then, after playing $\tau_h$ actions, the posterior variance in the traditional GP model is the following:
\begin{equation}
\sigma_{\tau_h}^2(\x) = k(\x, \x) - \vk(\x, \X_{\tau_h})^{\top}(\K_{\X_{\tau_h}\X_{\tau_h}}+\lambda \I)^{-1}\vk(\x,\X_{\tau_h}).\label{eq:stand_sigma_update_l}
\end{equation}
For the posterior mean, without the observations $\y_{T_l} = [y_{t_l+1}, \dots, y_{t_l+T_l}]^{\top}$ corresponding to the actions $\X_{T_l} = [\x_{t_l+1}^{\top} \dots, \x_{t_l+T_l}^{\top}]^{\top}$, we replace $\y_{T_l}$ with $\frac{1}{|U_l|}\sum_{u\in U_l} \y_{l,u}$ where $\y_{l,u} = [y_{u,t_l+1}, \dots, y_{u,t_l+T_l}]^{\top}$ 
in the traditional GP model. Then, the posterior mean  becomes the following:
\begin{equation}
\mu_{T_l}(\x) = \frac{1}{|U_l|}\sum_{u\in U_l} \vk(\x,\X_{T_l})^{\top} (\K_{\X_{T_l}\X_{T_l}} +\lambda \I)^{-1}\y_{l,u}. \label{eq:standard_mu_update_l}
\end{equation}
In our algorithm, in order to save computation complexity and communication cost, we use Eq.~\eqref{eq:sigma_update} and Eq.~\eqref{eq:mu_update} instead of the above formula. In the following lemma, we show that they are equivalent. 
\begin{lemma}
[Equivalent representations]\label{lem:equivalent_representation} Consider any phase $l$. By the end of the $h$-th phase, 
the posterior variance Eq.~\eqref{eq:stand_sigma_update_l} in the traditional GP model is equivalent to Eq.~\eqref{eq:sigma_update} used in our \texttt{DPBE} algorithm. That is, for any $\x\in \cD$, we have
\begin{equation}
\begin{aligned}
\sigma^2_{\tau_h}(\x) = \Sigma^2_{h}(\x), \quad \forall h =1, \dots, H_l.
\end{aligned}
\end{equation}
Moreover, we have the two representations (Eq.~\eqref{eq:standard_mu_update_l} and Eq.~\eqref{eq:mu_update}) for the posterior mean function are equivalent. That is, for any $\x\in \cD$, we have
\begin{equation}
    \mu_{T_l}(\x) =\bar{\mu}_l(\x).
\end{equation}
\end{lemma}

\begin{proof}
First, we have the following result, which helps connect the two representations of mean and variance functions:
\begin{equation}
\begin{aligned}
\Phi_{\tau_h}^{\top}\Phi_{\tau_h} &= \sum_{t=t_l+1}^{t_l+\tau_h} \varphi(\x_t)\varphi(\x_t)^{\top}\\
& \overset{(a)}{=} \sum_{i=1}^{h} T_l(\ca_i) \varphi(\ca_i) \varphi(\ca_i)^{\top}\\
& =  \Phi_{h}^{\top}\W_{h}\Phi_{h}, \label{eq:useful_rs_1}
\end{aligned}
\end{equation}
where $(a)$ is due to our algorithm decisions: $\x_t = \ca_i$ for any $t\in \mathcal{T}_l(\ca_i)=\{t_l+\tau_{i-1}+1, t_l+\tau_{i-1}+T_l(\ca_i)\}$ and the last step holds because $\W_h$ is a diagonal matrix with $(\W_h)_{ii} = T_l(\ca_i)$ for any $i \in [h]$.

Then, we are ready to derive the equivalence of two representations of the mean function.

1) Variance representation equivalence: $
\sigma^2_{\tau_h}(\x) = \Sigma^2_{h}(\x)
$ for $h=1,\dots, H_l$. 
This implies
\begin{equation*}
\begin{aligned}
&k(\x, \x) - \vk(\x, \X_{\tau_h})^{\top}(\K_{\X_{\tau_h}\X_{\tau_h}}+\lambda \I)^{-1}\vk(\x,\X_{\tau_h}) \\
&=k(\x, \x) - \vk(\x, \A_h)^{\top}(\K_{\A_h\A_h}+\lambda \W_{h}^{-1})^{-1}\vk(\x,\A_h).
\end{aligned}
\end{equation*}
It remains to show the following:
\begin{equation}
\vk(\x, \X_{\tau_h})^{\top}(\K_{\X_{\tau_h}\X_{\tau_h}}+\lambda \I)^{-1}\vk(\x,\X_{\tau_h})  = \vk(\x, \A_h)^{\top}(\K_{\A_h\A_h}+\lambda \W_{h}^{-1})^{-1}\vk(\x,\A_h). 
\end{equation}
Using the feature space formulations, we have
\begin{equation}
\begin{aligned}
&\vk(\x, \A_h)^{\top}(\K_{\A_h\A_h}+\lambda \W_{h}^{-1})^{-1}\vk(\x,\A_h)\\
= & \varphi(\x)^{\top} \Phi_h^{\top} (\Phi_{h}\Phi_{h}^{\top}+\lambda \W_h^{-1})^{-1}\Phi_{h}\varphi(\x) \\
= & \varphi(\x)^{\top} \Phi_h^{\top} \W_h^{1/2} (\W_h^{1/2}\Phi_{h}\Phi_{h}^{\top}\W_h^{1/2}+\lambda \I)^{-1}\W_h^{1/2}\Phi_{h}\varphi(\x) \\
= & \varphi(\x)^{\top}  (\Phi_{h}^{\top}\W_h^{1/2}\W_h^{1/2}\Phi_{h}+\lambda \I)^{-1}\Phi_h^{\top} \W_h^{1/2}\W_h^{1/2}\Phi_{h}\varphi(\x) \\
= & \varphi(\x)^{\top}  (\Phi_{h}^{\top}\W_h \Phi_{h}+\lambda \I)^{-1}\Phi_h^{\top} \W_h\Phi_{h}\varphi(\x) \\
\overset{(a)}{=} & \varphi(\x)^{\top}  (\Phi_{\tau_h}^{\top} \Phi_{\tau_h}+\lambda \I)^{-1}\Phi_{\tau_h}^{\top} \Phi_{\tau_h}\varphi(\x) \\
= & \varphi(\x)^{\top} \Phi_{\tau_h}^{\top}(\Phi_{\tau_h}\Phi_{\tau_h}^{\top}+\lambda \I)^{-1}\Phi_{\tau_h}\varphi(\x)\\
= &\vk(\x, \X_{\tau_h})^{\top}(\K_{\X_{\tau_h}\X_{\tau_h}}+\lambda \I)^{-1}\vk(\x,\X_{\tau_h}),
\end{aligned}
\end{equation}
where $(a)$ is from Eq.~\eqref{eq:useful_rs_1}. Then, we have $\sigma_{\tau_h}^2(\x) = \Sigma_h^2(\x)$.

2) Mean representation equivalence: $\mu_{T_l}(\x) =\bar{\mu}_l(\x)$, i.e., 
\begin{equation}
\begin{aligned}
\frac{1}{|U_l|}\sum_{u\in U_l} \vk(\x,\X_{T_l})^{\top} (\K_{\X_{T_l}\X_{T_l}} +\lambda \I)^{-1}\y_{l,u} &=  k(\x, \A_{H_l})^{\top}(\K_{\A_{H_l}\A_{H_l}} +\lambda \W_{H_l}^{-1})^{-1}\bar{\y}_l.
\end{aligned}
\end{equation}

For the last within-phase index $\tau_{H_l}=T_l$, we also have the following:

\begin{equation}
\begin{aligned}
\frac{1}{|U_l|}\sum_{u\in U_l} \Phi_{T_l}^{\top}\y_{l,u} & =\frac{1}{|U_l|}\sum_{u\in U_l} \sum_{t=t_l+1}^{t_l+T_l} y_{u,t}\varphi(\x_t) \\
&  =  \frac{1}{|U_l|}\sum_{u\in U_l} \sum_{h=1}^{H_l} \sum_{t\in \cT_{l}(\ca_h)} y_{u,t}\varphi(\x_t) \\
&  = \frac{1}{|U_l|}\sum_{u\in U_l} \sum_{h=1}^{H_l} \varphi(\ca_h)\sum_{t\in \cT_{l}(\ca_h)} y_{u,t} \\
& = \frac{1}{|U_l|}\sum_{u\in U_l} \sum_{h=1}^{H_l} \varphi(\ca_h) T_l(\ca_h) y_l^u(\ca_h) \\
& =  \sum_{h=1}^{H_l} T_l(\ca_h) y_l(\ca_h) \varphi(\ca_h)  \\
& =\Phi_{H_l}^{\top}\W_{H_l}\bar{\y}_{l}. \label{eq:useful_rs_2}
\end{aligned}
\end{equation}
Then, we are ready to derive the equivalence of two representations of the mean function: 
\begin{equation}
\begin{aligned}
&\frac{1}{|U_l|}\sum_{u\in U_l} \vk(\x,\X_{T_l})^{\top} (\K_{\X_{T_l}\X_{T_l}} +\lambda \I)^{-1}\y_{l,u}  \\
 &=\frac{1}{|U_l|}\sum_{u\in U_l}  \varphi(\x)^{\top} \Phi_{T_l}^{\top}(\Phi_{T_l}\Phi_{T_l}^{\top}+\lambda \I)^{-1}\y_{l,u}\\
 & = \frac{1}{|U_l|}\sum_{u\in U_l}  \varphi(\x)^{\top} (\Phi_{T_l}^{\top}\Phi_{T_l}+\lambda \I)^{-1}\Phi_{T_l}^{\top}\y_{l,u}\\
 & = \varphi(\x)^{\top} (\Phi_{T_l}^{\top}\Phi_{T_l}+\lambda \I)^{-1}\cdot \frac{1}{|U_l|}\sum_{u\in U_l}  \Phi_{T_l}^{\top}\y_{l,u}\\
 &\overset{(a)}{=} \varphi(\x)^{\top} (\Phi_{H_l}^{\top}\W_{H_l}\Phi_{H_l}+\lambda \I)^{-1}\Phi_{H_l}^{\top}\W_{H_l}\bar{\y}_{l}\\ 
 &= \varphi(\x)^{\top} (\Phi_{H_l}^{\top}\W_{H_l}^{1/2}\W_{H_l}^{1/2}\Phi_{H_l}+\lambda \I)^{-1}\Phi_{H_l}^{\top}\W_{H_l}^{1/2}\W_{H_l}^{1/2}\bar{\y}_{l}\\ 
  &= \varphi(\x)^{\top} ((\W_{H_l}^{1/2}\Phi_{H_l})^{\top}(\W_{H_l}^{1/2}\Phi_{H_l})+\lambda \I)^{-1}(\W_{H_l}^{1/2}\Phi_{H_l})^{\top}\W_{H_l}^{1/2}\bar{\y}_{l}\\ 
  &= \varphi(\x)^{\top}\Phi_{H_l}^{\top} \W_{H_l}^{1/2} (\W_{H_l}^{1/2}\Phi_{H_l}\Phi_{H_l}^{\top}\W_{H_l}^{1/2}+\lambda \I)^{-1}\W_{H_l}^{1/2}\bar{\y}_{l}\\ 
  &= \varphi(\x)^{\top}\Phi_{H_l}^{\top}  (\Phi_{H_l}\Phi_{H_l}^{\top}+\lambda\W_{H_l}^{-1})^{-1}\bar{\y}_{l}\\ 
 &= \vk(\x, \A_{H_l})^{\top}(\K_{\A_{H_l}\A_{H_l}} +\lambda \W_{H_l}^{-1})^{-1}\bar{\y}_l = \bar{\mu}_l(\x),
\end{aligned}
\end{equation}
where $(a)$ is from Eq.~\eqref{eq:useful_rs_1} with $\tau_{H_l}=T_l$ and the result in Eq.~\eqref{eq:useful_rs_2}.
\end{proof}

\subsection{Impact of Batch Schedule Strategy on Posterior Variance}
In our batch schedule strategy,  the decision $\x_t$ does not change for $T_l(\ca_h)$ rounds when starting choosing $\ca_h$ after $\tau_{h-1}$ rounds within the $l$-th phase. Applying Lemma~\ref{lem:variance_ratio} to our setting with $\tau^{\prime} = \tau_{h-1}$, we obtain the following corollary. 
\begin{corollary}\label{cor:variance_ratio}
Consider any phase $l$. Recall that $\tau_{h-1}$ is the within-phase time index before starting choosing $\ca_h$. Then, give any set of chosen actions $\A_{{h-1}}$ for the first $h-1$ batches, for any kernel $k$, any  $\x\in \cD$, and any $\tau \in [\tau_{h-1} +1, \tau_{h-1} +T_l(\ca_h)]$, we have 
\begin{equation}
    1\leq \frac{\Sigma_{h-1}(\x)}{\sigma_{\tau}(\x)} \leq C. \label{eq:variance_ratio_col}
\end{equation}
\end{corollary}
\begin{proof}
Applying Lemma~\ref{lem:variance_ratio} to our setting, we have
\begin{equation}
    1\leq \frac{\sigma^2_{\tau_{h-1}}(\x)}{\sigma^2_{\tau}(\x)} \leq 1+ T_l(\ca_h) \sigma^2_{\tau_h}(\ca_h).
\end{equation}
Moreover, by selecting $T_l(\ca_h) = \lfloor (C^2-1)/\Sigma^2_{h-1}(\ca_h)\rfloor = \lfloor (C^2-1)/\sigma^2_{\tau_{h-1}}(\ca_h)\rfloor$ (Lemma~\ref{lem:equivalent_representation}) in our algorithm, we derive the result in Eq.~\eqref{eq:variance_ratio_col}.
\end{proof}

One key step to getting the regret upper bound is to bound the confidence width, which is related to the maximal value of the posterior variance by the end of each phase. (See Eq.~\eqref{eq:confidence_width}. In the following, we provide a bound for the maximal value of the posterior variance.  
\begin{lemma}\label{lem:max_variance}
The posterior variance after $H_l$ batches (decisions) in the $l$-th phase satisfies
\begin{equation}
     \max_{\x\in\cD_l} \Sigma_{H_l}(\x) \leq  \sqrt{\frac{2\sigma^2C^2\gamma_{T_l}}{T_l}}. \label{eq:max_sigma}
\end{equation}
\end{lemma}
\begin{proof} 
Recall that \texttt{DPBE} plays action $\ca_h$ when $\tau \in [\tau_{h-1}+1, \tau_{h-1}+T_l(\ca_h)]$ within the $l$-th phase. First, we have for any $\x\in\cD_l$, any $h\leq H_l$, 
\begin{equation}
  \Sigma_{H_l}(\x) 
  \overset{(a)}{\leq} \Sigma_{h-1}(\x) 
  \overset{(b)}{\leq} \Sigma_{h-1} (\ca_h)
  = \sigma_{\tau_{h-1}}(\ca_h),
\end{equation}
where $(a)$ holds because $\Sigma_{h}(\cdot)$ is non-increasing in $h$, $(b)$ is based on our decision, and the last step is due to the equivalent representation result. Then, we have the following:
\begin{equation}
\begin{aligned}
 \max_{\x\in\cD_l} \Sigma_{H_l}(\x) 
 & \leq \frac{1}{T_l} \sum_{h=1}^{H_l} T_l(\ca_h)\Sigma_{h-1}(\ca_h) \\
 &= \frac{1}{T_l} \sum_{h=1}^{H_l} \sum_{\tau=\tau_{h-1}+1}^{\tau_{h-1}+T_l(\ca_h)}\Sigma_{h-1}(\ca_h)\\
 &=\frac{1}{T_l} \sum_{h=1}^{H_l} \sum_{\tau=\tau_{h-1}+1}^{\tau_{h-1}+T_l(\ca_h)} \frac{\Sigma_{h-1}(\ca_h)}{\sigma_{\tau}(\ca_h)}\cdot \sigma_{\tau}(\ca_h)\\
  &\overset{(a)}{\leq }\frac{1}{T_l} \sum_{h=1}^{H_l} \sum_{\tau=\tau_{h-1}+1}^{\tau_{h-1}+T_l(\ca_h)} C \sigma_{\tau}(\ca_h)\\
  &\overset{(b)}{= }\frac{C}{T_l} \sum_{h=1}^{H_l} \sum_{\tau=\tau_{h-1}+1}^{\tau_{h-1}+T_l(\ca_h)}  \sigma_{\tau}(\x_{t_l+\tau})\\
  & =\frac{C}{T_l}  \sum_{\tau=1}^{T_l}  \sigma_{\tau}(\x_{t_l+\tau})\\
 & \overset{(c)}{\leq} \frac{C}{T_l} \sqrt{T_l\sum_{\tau=1}^{T_l}  \sigma^2_{\tau}(\x_{t_l+\tau})} \\
 & \overset{(d)}{\leq} \frac{C}{T_l}\sqrt{T_l\cdot 2\lambda  \gamma_{T_l}}=\sqrt{\frac{2\lambda C^2\gamma_{T_l}}{T_l}},  
\end{aligned}
\end{equation}
where the inequality $(a)$ is from Corollary~\ref{cor:variance_ratio}, $(b)$ is based on our algorithm decision: $\x_{t_l+\tau} = \ca_h$ for any $\tau \in [\tau_{h-1}+1, \tau_{h-1}+T_l(\ca_h)]$, $(c)$ is by Cauchy-Schwartz inequality, and $(d)$ is from Lemma~\ref{lem:information_gain}. 
\end{proof}



\subsection{Other Useful Results}
\begin{lemma}\label{lem:f_diff}
Consider any particular phase $l$. In the traditional GP models, without noise in the reward observations, the difference between the ground truth and regression estimator satisfies
\begin{equation}
\begin{aligned}
\left|f(\x)-\vk(\x, \X_{T_l})^{\top}(\K_{\X_{T_l}\X_{T_l}}+\lambda \I)^{-1}f(\X_{T_l})\right|\leq  B\sigma_{T_l}(\x).
\end{aligned}
\end{equation}
\end{lemma}
\begin{proof}
Representing $f(\x)$ in the feature space, we have 
\begin{equation}
\begin{aligned}
&\left|f(\x)-\vk(\x, \X_{T_l})^{\top}(\K_{\X_{T_l}\X_{T_l}}+\lambda \I)^{-1}f(\X_{T_l})\right| \\
&= \left| \varphi(\x)^{\top} f - \varphi(\x)^{\top} \Phi_{T_l}^{\top}(\Phi_{T_l}\Phi_{T_l}^{\top}+\lambda \I)^{-1} \Phi_{T_l}f\right|\\
&= \left| \varphi(\x)^{\top} f - \varphi(\x)^{\top} (\Phi_{T_l}^{\top}\Phi_{T_l}+\lambda \I)^{-1}\Phi_{T_l}^{\top} \Phi_{T_l}f\right|\\
&= \left| \lambda \varphi(\x)^{\top} (\Phi_{T_l}^{\top}\Phi_{T_l}+\lambda \I)^{-1}f\right|\\
& \leq \Vert f\Vert_k \Vert \lambda  (\Phi_{T_l}^{\top}\Phi_{T_l}+\lambda \I)^{-1}\varphi(\x)\Vert_k \\
& \leq B\sqrt{\lambda \varphi(\x)^{\top} (\Phi_{T_l}^{\top}\Phi_{T_l}+\lambda \I)^{-1}\lambda \I  (\Phi_{T_l}^{\top}\Phi_{T_l}+\lambda \I)^{-1}\varphi(\x) }  \\
& \leq B\sqrt{\lambda \varphi(\x)^{\top} (\Phi_{T_l}^{\top}\Phi_{T_l}+\lambda \I)^{-1}(\Phi_{T_l}^{\top}\Phi_{T_l}+\lambda \I) (\Phi_{T_l}^{\top}\Phi_{T_l}+\lambda \I)^{-1}\varphi(\x) }  \\
& \leq B\sqrt{\lambda \varphi(\x)^{\top}  (\Phi_{T_l}^{\top}\Phi_{T_l}+\lambda \I)^{-1}\varphi(\x) }  \\
& = B\sigma_{T_l}(\x),
\end{aligned}
\end{equation}
where the last step is from Eq.~\eqref{eq:sigma_feature_space}.
\end{proof}

\section{Proofs of Theorem~\ref{thm:regret_upper_bound}}\label{app:proof_reget}

Before proving Theorem~\ref{thm:regret_upper_bound}, we first provide the key concentration inequality under \texttt{DPBE} in Theorem~\ref{thm:concentration_ineq}. 
\begin{theorem}\label{thm:concentration_ineq}
For any particular phase $l$, with probability at least $1-4\beta$, the following holds
\begin{equation}
    |f(\x) - \bar{\mu}_l(\x)| \leq   w_l(\x),\label{eq:concentration_ineq}
\end{equation}
where mean function $\bar{\mu}_l(\x)$ and confidence width function $w_l(\x)$ are defined in Eq.~\eqref{eq:mu_update} and Eq.~\eqref{eq:confidence_width}.
\end{theorem}
\begin{proof}
In this proof, we will show the following concentration inequality holds for any $\x\in \cD$
\begin{equation}
    \bP[|f(\x) - \bar{\mu}_l(\x)|\geq  w_l(\x)] \leq 4\beta.
\end{equation}

For any $\x\in \cD$, we let $w_{l}(\x) = w_{l,1}(\x) + w_{l,2}(\x)$, where
\begin{equation*}
    w_{l,1}(\x) \triangleq \sqrt{\frac{2k(\x, \x)\log(1/\beta)}{|U_l|}} \quad \text{and } w_{l,2}(\x) \triangleq \Sigma_{H_l}(\x) \left( \sqrt{\frac{2\log(1/\beta)}{ |U_l|}}+B\right).
\end{equation*}
First, for any $\x\in \cD$, we have the following inequality:
\begin{equation*}
     |f(\x) - \bar{\mu}_l(\x)|\leq 
   \left|f(\x)- \frac{1}{|U_l|}\sum_{u\in U_l} f_u(\x)\right| + \left|\frac{1}{|U_l|}\sum_{u\in U_l} f_u(\x)-\bar{\mu}_l(\x)\right|. 
\end{equation*}
Then, we have
\begin{equation}
    \begin{aligned}
    &\bP\left[|f(\x) - \bar{\mu}_l(\x)|\geq  w_l(\x)\right] \\
    \leq & \bP\left[\left|f(\x)- \frac{1}{|U_l|}\sum_{u\in U_l} f_u(\x)\right| +  \left|\frac{1}{|U_l|}\sum_{u\in U_l} f_u(\x)-\bar{\mu}_l(\x)\right| \geq  w_{l,1}(\x) + w_{l,2}(\x)\right]\\
    \leq & \bP\left[\left|f(\x)- \frac{1}{|U_l|}\sum_{u\in U_l} f_u(\x)\right| \geq  w_{l,1}(\x) \right] + \bP\left[  \left|\frac{1}{|U_l|}\sum_{u\in U_l} f_u(\x)-\bar{\mu}_l(\x)\right| \geq   w_{l,2}(\x)\right],
    \label{eq:gap_bound}
    \end{aligned}
\end{equation}
where the last inequality is from union bound. 

In the following, we try to bound the above two terms, respectively. 

\smallskip

\textbf{i)} Recall that each user $u$ is associated with a local reward function $f_u \sim \cG\cP(f(\cdot),  k(\cdot))$. Hence,
\begin{equation}
f_u(\x) \sim \cN(f(\x),  k(\x, \x)), \quad \forall \x \in \cD. \label{eq:prior_distr}
\end{equation}
Note that $U_l$ is a set of $\lceil 2^{\alpha l} \rceil$ independently sampled random users. Then, we have
\begin{equation*}
\frac{1}{|U_l|}\sum_{u\in U_l} f_u(\x) \sim  \cN\left(f(\x), \frac{k(\x, \x)}{|U_l|}\right), \quad \forall \x \in \cD.
\end{equation*}
Combining the concentration inequality for Gaussian random variables, we have 
\begin{equation}
    \bP\left[\left| \frac{1}{|U_l|}\sum_{u\in U_l} f_u(\x) - f(\x)\right| \geq  w_{l,1}(\x) \right] \leq 2\exp \left(-\frac{|U_l|w_{l,1}^2(\x)}{2k(\x,\x))}\right) = 2\beta.
\end{equation}

\textbf{ii)} Then, we want to bound the second term in Eq.~\eqref{eq:gap_bound}: 
\begin{equation*}
\begin{aligned}
&\bP\left[ \left|\frac{1}{|U_l|}\sum_{u\in U_l} f_u(\x)-\bar{\mu}_l(\x)\right| \geq   w_{l,2}(\x)\right] \\
= & \sum_{\Lambda}\bP[\Lambda=\{\y_{l,u}\}_{u\in U_l}]\cdot \bP\left[  \left|\frac{1}{|U_l|}\sum_{u\in U_l} f_u(\x)-\bar{\mu}_l(\x)\right| \geq   w_{l,2}(\x)\middle| \{\y_{l,u}\}_{u\in U_l}\right],
\end{aligned}    
\end{equation*}
where $\y_{l,u} = [y_{u,t_l+1}, \dots, y_{u,t_l+T_l}]^{\top}$ denotes the realization of the local reward observations at user $u$ in the $l$-th phase. 
According to our assumption, the participant user $u$ is associated with a local reward function $f_u$ sampled from Gaussian Process $\cG\cP(f(\cdot), k(\cdot, \cdot))$. Given the points $\X_{T_l} = [x^{\top}_{t_l+1}, \dots, \x^{\top}_{t_l+T_l}]^{\top}$ in $\cD$, the corresponding vector of local rewards $\y_{l,u} =[y_{u,t_l+1}, \dots, y_{u,t_l+T_l}]^{\top}$ has the multivariate Gaussian distribution $\cN(f(\X_{T_l}), (\K_{\X_{T_l}\X_{T_l}} + \lambda \I))$ where $f(\X_{T_l}) = [f(\x_{t_l+1}), \cdots,$ $f(\x_{t_l+T_l})]^{\top}$ and $\K_{\X_{T_l}\X_{T_l}}=[k(\x,\x^{\prime})]_{\x,\x^{\prime}\in \X_{T_l}}$ is the kernel matrix for the $T_l$ selected actions in the $l$-th phase. Due to the properties of GPs, we have that $\y_{l,u}$ and $f_u(\x)$ are jointly Gaussian given $\X_{T_l}$:
\begin{equation}
   \Bigg[
   \begin{aligned}
   &f_u(\x)\\
   &\y_{l,u}
   \end{aligned}
   \Bigg] \sim \cN 
   \Bigg(
   \Bigg[
   \begin{aligned}
   &f(\x)\\
   &f(\X_{T_l})
   \end{aligned}
   \Bigg], 
   \Bigg[
   \begin{aligned}
   &k(\x,\x)  \quad  &\vk(\x,\X_{T_l})^{\top} \\
   &\vk(\x, \X_{T_l})  &\K_{\X_{T_l}\X_{T_l}}+\lambda \I
   \end{aligned}
   \Bigg]
   \Bigg),
\end{equation}
where $k(\x, \X_{T_l}) = [k(\x, \x_{t_l+1}), \dots, k(\x, \x_{t_l+T_l})]^{\top}$. According to the basic formula for conditional distributions of Gaussian random vectors (see \cite[Appendix A.2]{rasmussen2003gaussian} or \cite[Proposition 3.2]{kanagawa2018gaussian}), we have that conditioned on $\y_{l,u}$ (corresponding to the points $\X_{T_l}$), the following holds: 
\begin{equation*}
    f_u(\x) | \y_{l,u} \sim \cN(m_u(\x), \sigma_{T_l}^2(\x)),
\end{equation*}
where we have
\begin{align}
m_u(\x) &\triangleq f(\x) + \vk(\x,\X_{T_l})^{\top} (\K_{\X_{T_l}\X_{T_l}} +\lambda \I)^{-1}(\y_{l,u}-f(\X_{T_l})), \label{eq:posterior_mean_u}\\
\sigma_{T_l}^{2}(\x) &= k(\x, \x) - \vk(\x, \X_{T_l})^{\top}(\K_{\X_{T_l}\X_{T_l}}+\lambda \I)^{-1}\vk(\x,\X_{T_l}).
\end{align}
Note that we sample the participants $U_l$ independently and that the local reward noise is also independent across participants. Then, we have the following result: 
\begin{equation*}
\left(\frac{1}{|U_l|}\sum_{u\in U_l} f_u(\x) \right) \Bigm| \{\y_{l,u}\}_{u\in U_l}  = \frac{1}{|U_l|}\sum_{u\in U_l} (f_u(\x)\mid \y_{l,u}) \sim \cN\left(\frac{1}{|U_l|}\sum_{u\in U_l}m_u(\x), \frac{\sigma_{T_l}^2(\x)}{|U_l|}\right). 
\end{equation*}
Combining the Gaussian concentration inequality, we have the following result
\begin{equation}
\begin{aligned}
\bP\left[  \left|\frac{1}{|U_l|}\sum_{u\in U_l} f_u(\x)-\frac{1}{|U_l|}\sum_{u\in U_l}m_u(\x)\right| \geq   \sqrt{\frac{2\sigma_{T_l}^2(\x)\log(1/\beta)}{|U_l|}}\middle| \{\y_{l,u}\}_{u\in U_l}\right]\leq 2\beta. \label{eq:local_posterior_con}
\end{aligned}    
\end{equation}

From Lemma~\ref{lem:equivalent_representation}, we have the following equation:
\begin{equation}
\begin{aligned}
\frac{1}{|U_l|}\sum_{u\in U_l} \vk(\x,\X_{T_l})^{\top} (\K_{\X_{T_l}\X_{T_l}} +\lambda \I)^{-1}\y_{l,u} &= \vk(\x, \A_h)^{\top}(\K_{\A_h\A_h} +\lambda \W_h^{-1})^{-1}\bar{\y}_l = \bar{\mu}_l(\x),
\end{aligned}
\end{equation}
which implies 
\begin{equation}
    \frac{1}{|U_l|}\sum_{u\in U_l}m_u(\x) = \bar{\mu}_l(\x) + f(\x)-\vk(\x, \X_{T_l})^{\top}(\K_{\X_{T_l}\X_{T_l}}+\lambda \I)^{-1}f(\X_{T_l}).
\end{equation}
Then, the gap between the average local function $\frac{1}{|U_l|}\sum_{u\in U_l}f_u(\cdot)$ and the estimator $\bar{\mu}_l(\cdot)$ satisfies
\begin{equation}
\begin{aligned}
&\left|\frac{1}{|U_l|}\sum_{u\in U_l} f_u(\x)-\bar{\mu}_l(\x)\right|\\
 \leq &\left|\frac{1}{|U_l|}\sum_{u\in U_l} f_u(\x)-\frac{1}{|U_l|}\sum_{u\in U_l}m_u(\x)\right| + \left|f(\x)-\vk(\x, \X_{T_l})^{\top}(\K_{\X_{T_l}\X_{T_l}}+\lambda \I)^{-1}f(\X_{T_l})\right|\\
\overset{(a)}{\leq} &\left|\frac{1}{|U_l|}\sum_{u\in U_l} f_u(\x)-\frac{1}{|U_l|}\sum_{u\in U_l}m_u(\x)\right| + B\sigma_{T_l}(\x),
\end{aligned}
\end{equation}
where $(a)$ is from Lemma~\ref{lem:f_diff}. Combining the result in Eq.~\eqref{eq:local_posterior_con}, we have
\begin{equation}
\begin{aligned}
&\bP\left[  \left|\frac{1}{|U_l|}\sum_{u\in U_l} f_u(\x)-\bar{\mu}_l(\x)\right| \geq   w_{l,2}(\x)\middle|\{\y_{l,u}\}_{u\in U_l}\right]   \\
\leq &\bP\left[ \left|\frac{1}{|U_l|}\sum_{u\in U_l} f_u(\x)-\frac{1}{|U_l|}\sum_{u\in U_l}m_u(\x)\right| + B\sigma_{T_l}(x_0) \geq   w_{l,2}(\x)\middle| \{\y_{l,u}\}_{u\in U_l}\right] \\
\overset{(a)}{=}&\bP\left[  \left|\frac{1}{|U_l|}\sum_{u\in U_l} f_u(\x)-\frac{1}{|U_l|}\sum_{u\in U_l}m_u(\x)\right| \geq   \sqrt{\frac{2\sigma_{T_l}^2(\x)\log(1/\beta)}{|U_l|}}\middle| \{\y_{l,u}\}_{u\in U_l}\right]\leq 2\beta,
\end{aligned}
\end{equation}
where  $(a)$ is from $\sigma_{T_l}^2(\x) = \Sigma_{H_l}^2(\x)$ according to Lemma~\ref{lem:equivalent_representation}.
Therefore, we derive the desired result:
\begin{equation}
\begin{aligned}
&\bP\left[  \left|\frac{1}{|U_l|}\sum_{u\in U_l} f_u(\x)-\bar{\mu}_l(\x)\right| \geq   w_{l,2}(\x)\right] \\
= & \sum_{\Lambda}\bP[\Lambda=\{\y_{l,u}\}_{u\in U_l}]\cdot \bP\left[  \left|\frac{1}{|U_l|}\sum_{u\in U_l} f_u(\x)-\bar{\mu}_l(\x)\right| \geq   w_{l,2}(\x)\middle| \{\y_{l,u}\}_{u\in U_l}\right]\\
\leq & \sum_{\Lambda}\bP[\Lambda=\{\y_{l,u}\}_{u\in U_l}]\cdot 2\beta = 2\beta.
\end{aligned}    
\end{equation}
\end{proof}

To prove Theorem~\ref{thm:regret_upper_bound}, we first present three main conclusions when the concentration inequality in Theorem~\ref{thm:concentration_ineq} holds, then get an upper bound for the regret incurred in a particular phase~$l$ with high probability, and finally sum up the regret over all phases. 

Define a ``good'' event when Eq.~\eqref{eq:concentration_ineq} holds in the $l$-th phase as:
\begin{equation*}
    \cE_l \triangleq \left\{\forall \x\in \cD_l,  \left|f(\x)-\bar{\mu}_l(\x)\right|\leq w_l(\x)  \right\}.  
\end{equation*}
We have $\bP[\cE_l] \geq 1-4|\cD|\beta$ via the union bound.
Then, under event $\mathcal{E}_l$ in the $l$-th phase, we have the following three observations:
\begin{itemize}
    \item[\textbf{1.}] For any optimal action $\x^*\in \argmax_{\x\in \cD} f(\x)$, if $\x^*\in \cD_l$, then $\x^* \in \cD_{l+1}$. 
    \item[\textbf{2.}] Let $f^* = \max_{\x\in \cD} f(\x)$. Supposed that $x^*\in \cD_l$. For any $\x \in \cD_{l+1}$, its reward gap from the optimal reward is bounded by $4\max_{\x\in \cD_l}w_l(\x)$, i.e., 
    $$f^* - f(\x)\leq 4\max_{\x\in \cD_l}w_l(\x).$$ 
    \item[\textbf{3.}] The confidence width function 
    satisfies 
    $$\max_{\x\in\cD_l} w_l(\x)\leq \sqrt{\frac{2\kappa^2\log(1/\beta)}{|U_l|}}+\sqrt{\frac{4\sigma^2C^2\gamma_{T_l}\log(1/\beta)}{T_l|U_l|}}+\sqrt{\frac{2\sigma^2B^2C^2\gamma_{T_l}}{T_l}}.$$
\end{itemize}

\begin{proof}
\textbf{Observation 1:}
Let $\mathbf{b}\in \argmax_{\x\in\cD_l}  (\bar{\mu}_l(\x)-w_l(\x))$. Then under event $\mathcal{E}_l$, we have 
\begin{equation}
\begin{aligned}
 \bar{\mu}_l(\x^*)+w_l(\x^*)
  &\geq f(\x^*)  \geq f(\mathbf{b}) \geq \bar{\mu}_l(\mathbf{b})-w_l(\mathbf{b}),
\end{aligned}
\end{equation}
which indicates  $\x^* \in \cD_{l+1}$ according to Eq.~\eqref{eq:action_elimination}.

\textbf{Observation 2: }
For any $\x \in \cD_{l+1}$, we have $\x \in \cD_l$ and 
\begin{equation}
    \bar{\mu}_l(\x) + w_l(\x) \geq\bar{\mu}_l(\mathbf{b})-w_l(\mathbf{b})\geq \bar{\mu}_l(\x^*) -w_l(\x^*). \label{eq:estimation_regret}
\end{equation}
Then, we have the regret of choosing any action $\x\in\cD_{l+1}$ satisfying
\begin{equation}
\begin{aligned}
 f(\x^*) -f(\x)  &\overset{(a)}{\leq} \bar{\mu}_l(\x^*) + w_l(\x^*) -\bar{\mu}_l(\x)+w_l(\x) \\
  &\overset{(b)}{\leq} 2(w_l(\x) +w_l(\x^*)) \\
  &\leq 4\max_{\x\in\cD_l} w_l(\x),
\end{aligned}
\end{equation}
where $(a)$ holds under event $\cE_l$ and the second inequality $(b)$ is from Eq.~\eqref{eq:estimation_regret}.
Then, we derive  Observation 2. 

\textbf{Observation 3: } 
Based on the result in Lemma~\ref{lem:max_variance}, we have
\begin{equation}
\begin{aligned}
\max_{\x\in\cD_l} w_l(\x) = &\max_{\x\in\cD_l} \left( \sqrt{\frac{2k(\x, \x)\log(1/\beta)}{|U_l|}} + \Sigma_{H_l}(\x) \left( \sqrt{\frac{2\log(1/\beta)}{ |U_l|}}+B\right)\right)\\
& \leq \sqrt{\frac{2\kappa^2\log(1/\beta)}{|U_l|}} + \max_{\x\in\cD_l}\Sigma_{H_l}(\x) \left( \sqrt{\frac{2\log(1/\beta)}{ |U_l|}}+B\right) \\  
&\leq \sqrt{\frac{2\kappa^2\log(1/\beta)}{|U_l|}}+\sqrt{\frac{4\lambda C^2\gamma_{T_l}\log(1/\beta)}{T_l|U_l|}}+\sqrt{\frac{2\lambda B^2C^2\gamma_{T_l}}{T_l}}\\
&= \sqrt{\frac{2\kappa^2\log(1/\beta)}{|U_l|}}+\sqrt{\frac{4\sigma^2C^2\gamma_{T_l}\log(1/\beta)}{T_l|U_l|}}+\sqrt{\frac{2\sigma^2B^2C^2\gamma_{T_l}}{T_l}}.
\end{aligned}
\end{equation}
\end{proof}
Then, we are ready to prove Theorem~\ref{thm:regret_upper_bound}.
\begin{proof}[Proof of Theorem~\ref{thm:regret_upper_bound}]
Let the regret in the $l$-th phase be $r_l \triangleq \sum_{t\in \cT_l} (\max_{\x\in \cD} f(\x) - f(\x_t))$. For any $l\geq 2$, we assume event $\cE_{l-1}$ holds. Then, we have the following result
\begin{equation}
\begin{aligned}
r_l &=  \sum_{t \in \cT_l} (\max_{\x\in \cD} f(\x) - f(\x))\\
    &\leq \sum_{t\in\cT_l} 4\max_{\x\in \cD_{l-1}}w_{l-1}(\x)\\
    &\leq  4T_l\max_{\x\in \cD_{l-1}}w_{l-1}(\x)\\
    & \leq 4 T_l \left(\sqrt{\frac{2\kappa^2\log(1/\beta)}{|U_{l-1}|}}+\sqrt{\frac{4\sigma^2C^2\gamma_{T_{l-1}}\log(1/\beta)}{T_{l-1}|U_{l-1}|}}+\sqrt{\frac{2\sigma^2B^2C^2\gamma_{T_{l-1}}}{T_{l-1}}}\right)\\
     & \overset{(a)}{\leq }  4 \cdot 2^{l-1} \left(\sqrt{\frac{2\kappa^2\log(1/\beta)}{2^{\alpha (l-1)}}}+\sqrt{\frac{4\sigma^2C^2\gamma_{T}\log(1/\beta)}{2^{(1+\alpha)(l-1)-1}}}+\sqrt{\frac{2\sigma^2B^2C^2\gamma_T}{2^{l-2}}}\right)\\
     & \leq 4\sqrt{2\kappa^2\log (1/\beta)}\sqrt{2^{(2-\alpha)(l-1)}} + 8\sigma C \sqrt{2\gamma_T\log(1/\beta)}\sqrt{2^{(1-\alpha)(l-1)}} + 8\sigma BC\sqrt{\gamma_T 2^{l-1}}, \label{eq:simple_regret}
\end{aligned}
\end{equation}
where $(a)$ is from $\gamma_{T_{l-1}}\leq \gamma_T$ and $|U_l|\geq 2^{\alpha l}$. 

Define $\cE_g$ as the event where the ``good'' event occurs in every phase, i.e., $\cE_g \triangleq \bigcap_{l=1}^L \cE_l$. It is not difficult to obtain  $\bP[\cE_g] \geq 1- 4|\cD|\beta L$ by applying union bound. At the same time, let $R_g$ be the regret under event $\cE_g$, and $R_b$ be the regret if event $\mathcal{E}_g$ does not hold. Then, the expected total regret in $T$ is $\E[R(T)] = \bP[\cE_g]R_g + (1-\bP[\cE_g]) R_b$.

Under event $\mathcal{E}_g$, the regret in the $l$-th phase $r_l$ satisfies Eq.~\eqref{eq:simple_regret} for any $l\geq 2$. Note that $r_1 \leq 2T_1B \kappa\leq 2B\kappa$  since $T_1=1$ and for any $\x\in\cD$,
$$
    |f(\x)|=|\langle f, k(\x,\cdot)\rangle_k| \leq \Vert f\Vert_k\langle k(\x,\cdot), k(\x,\cdot)\rangle_k^{1/2}\leq Bk(\x,\x)^{1/2}\leq B \kappa.
$$

Then, we have
\begin{equation}
    \begin{aligned}
    R_g &= \sum_{l=1}^L r_l  \\
    & \leq 2B\kappa + \sum_{l=2}^L 4\sqrt{2\kappa^2\log (1/\beta)}\sqrt{2^{(2-\alpha)(l-1)}} \\
    &+ \sum_{l=2}^L 8\sigma C \sqrt{2\gamma_T\log(1/\beta)}\sqrt{2^{(1-\alpha)(l-1)}} \\
    &+ \sum_{l=2}^L 8\sigma BC\sqrt{\gamma_T2^{l-1}} \\
    & \leq 2B\kappa + 4\sqrt{2\kappa^2\log (1/\beta)} \cdot 4 \sqrt{2^{(L-1)(2-\alpha)}} \\
    &+ 8\sigma C\sqrt{2\gamma_T\log(1/\beta)}\cdot C_1\sqrt{2^{(1-\alpha)(L-1)}}  \quad \quad  \quad \quad \quad \hfill \left(C_1 = \sqrt{2^{1-\alpha}}/(\sqrt{2^{1-\alpha}}-1)\right)\\  
    &+ 8\sigma BC\sqrt{\gamma_T}\cdot 4\sqrt{2^{L-1}} \\
    &\leq 2B\kappa+  16\sqrt{2\kappa^2\log (1/\beta)}T^{1-\alpha/2}+8\sigma C_1C  \sqrt{2\gamma_T\log(1/\beta)T^{1-\alpha}} + 32\sigma BC\sqrt{\gamma_T T},
    \end{aligned}
\end{equation}
where the last step is due to $2^{L-1} \leq T$ and $L\leq \log(2T)$ since $\sum_{l=1}^{L-1} T_l +1 \leq T$. 

On the other hand, $R_b \leq 2B\kappa T$ 
since $|\max_{\x\in\cD} f(\x)-f(\x)|\leq 2B\kappa$ for all $\x \in \cD$. 
Choose $\beta = 1/(|\cD|T)$ in Algorithm~\ref{alg:DPBE}. Finally, we have the following results:
\begin{equation}
\begin{aligned}
    &\E[R(T)] \\&= \bP[\cE_g]R_g + (1-\bP[\cE_g]) R_b \\
    & \leq R_g + 4|\cD|\beta L \cdot 2B\kappa T\\
    & \leq 2B\kappa+  16\sqrt{2\kappa^2\log (1/\beta)}T^{1-\alpha/2}+8\sigma C_1C  \sqrt{2\gamma_T\log(1/\beta)T^{1-\alpha}} + 32\sigma BC\sqrt{\gamma_T T} \\
    &+ 8B\kappa|\cD|\beta LT\\
    & = 2B\kappa+  16T^{1-\alpha/2}\sqrt{2\kappa^2\log (|\cD|T)}+8\sigma C_1C  \sqrt{2\gamma_TT^{1-\alpha}\log (|\cD|T)} \\
    &+ 32\sigma BC\sqrt{\gamma_T T} + 8B\kappa\log (2T)\\
    & = O(T^{1-\alpha/2}\sqrt{\log (|\cD|T)})+O(\sqrt{\gamma_T T^{1-\alpha}\log (|\cD)T)} +O(\sqrt{\gamma_T T}).
\end{aligned}
\end{equation}
\end{proof}

\section{Proofs for Communication and Computation Results} \label{app:proof_communication_complexity}
The results regarding computation complexity and communication cost highly depend on the number of batches $H_l$ in each phase $l$. Hence, we first provide the proof for Lemma~\ref{lem:bound_h}. 
\begin{proof}[Proof of Lemma~\ref{lem:bound_h}]
To bound the number of batches in the $l$-th phase, we follow a similar line to the proof of Lemma~4.3 in \cite{calandriello2022scaling}. For any $1\leq h\leq H_l$, we have
\begin{equation}
\begin{aligned}
&T_l(\ca_h) = \left\lfloor \frac{C^2 - 1}{\Sigma_{h-1}^2(\ca_h)}\right\rfloor \geq  \frac{C^2 - 1}{\Sigma_{h-1}^2(\ca_h)} - 1 \\
&\Rightarrow \quad \Sigma_{h-1}^2(\ca_h)(T_l(\ca_h)+1) \geq C^2-1 \\
& \Rightarrow \quad 2\Sigma_{h-1}^2(\ca_h)T_l(\ca_h)\geq C^2-1.
\end{aligned}
\end{equation}
Recall that we use $\tau_h$ to denote the last within-phase time index in the $h$-th batch. 
Then, summing the above inequality across all batches up to $H_l$, we have 
\begin{equation}
\begin{aligned}
H_l(C^2-1) 
&\leq \sum_{h=1}^{H_l}2\Sigma_{h-1}^2(\ca_h)T_l(\ca_h) \\
& \leq 2\sum_{h=1}^{H_l}\sum_{\tau=\tau_{h-1}+1}^{\tau_{h-1}+T_l(\ca_h)}\Sigma_{h-1}^2(\ca_h) \\
& = 2\sum_{h=1}^{H_l}\sum_{\tau=\tau_{h-1}+1}^{\tau_{h-1}+T_l(\ca_h)} \frac{\Sigma_{h-1}^2(\ca_h)}{\sigma_{\tau}^2(\ca_h)} \cdot \sigma_{\tau}^2(\ca_h)\\
& \overset{(a)}{\leq} 2\sum_{h=1}^{H_l}\sum_{\tau=\tau_{h-1}+1}^{\tau_{h-1}+T_l(\ca_h)} C^2 \cdot \sigma_{\tau}^2(\ca_h)\\
&\overset{(b)}{= } 2 C^2 \sum_{h=1}^{H_l} \sum_{\tau=\tau_{h-1}+1}^{\tau_{h-1}+T_l(\ca_h)}  \sigma^2_{\tau}(\x_{t_l+\tau})\\
& = 2C^2\sum_{\tau=1}^{T_l}\sigma_\tau^2(\x_{t_l+\tau})\\
& \overset{(c)}{\leq}  4\sigma^2C^2\gamma_{T_l},
\end{aligned}
\end{equation}
where $(a)$ is from Corollary~\ref{cor:variance_ratio}, $(b)$ is based on our algorithm decision: $\x_{t_l+\tau} = \ca_h$ for any $\tau \in [\tau_{h-1}+1, \tau_{h-1}+T_l(\ca_h)]$, $(c)$ is from Lemma~\ref{lem:information_gain} where $\X_{T_l}=[\x_{t_l+1}^{\top}, \dots, \x^{\top}_{t_l+T_l}]^{\top}$ for any phase $l$ and $\lambda=\sigma^2$. 
Hence, we derive 
\begin{equation}
    H_l \leq \frac{4\sigma^2C^2}{C^2-1} \gamma_{T_l}.
\end{equation}
\end{proof}

We already analyze how to derive the computation complexity for \texttt{DPBE} in Remark~\ref{rk:computation_analysis}. In the following, we prove Theorem~\ref{thm:communication}, which tells the result regarding communication cost: $O(\gamma_T T^{\alpha})$.
\begin{proof}[Proof of Theorem~\ref{thm:communication}]
Note that the communicating data in each phase between participants and the agent is the local average performance $y_l^u(\ca)$ for each action $\ca$ chosen in the corresponding batch. 
That is, $N_{u,l} \leq H_l$ for every participant $u$. (Here, the inequality holds when merging batches as Remark~\ref{rk:shrink}). Combining the bound of $H_l$ in Lemma~\ref{lem:bound_h}, we derive the total communication cost satisfying
\begin{equation}
    \sum_{l=1}^L |U_l| H_l \leq \sum_{l=1}^L  \frac{4\sigma^2C^2}{C^2-1} \gamma_{T_l} \cdot (2^{\alpha l}+1) = O\left(\frac{\sigma^2C^2}{C^2-1} \cdot \gamma_T T^{\alpha}\right),
\end{equation}
where the last step is due to $2^{L-1} \leq T$ and $L\leq \log(2T)$ since $\sum_{l=1}^{L-1} T_l +1 \leq T$.
\end{proof}

\section{Differentially Private \texttt{DPBE} Extensions}\label{app:DP-DPBE}
In this section, we extend the differentially private \texttt{DPBE} in Section~\ref{sec:DP-DPBE} to two other celebrated DP models: the local model and the shuffle model. 

\Hig{To begin with, we present the details of the \texttt{DP-DPBE} algorithm (see Algorithm~\ref{alg:DP-DPBE}) in the central DP model discussed in Section~\ref{sec:DP-DPBE}. Recall that in the central DP model, with a trusted agent, data privacy is protected by privatizing the aggregated feedback so that the output of the algorithm is indistinguishable between any two users. In a particular phase $l$, the aggregated feedback for each chosen action becomes $\Tilde{\y}_l = \bar{\y}_{l}+(\rho_1, \dots, \rho_{H_l})$ (see Eq.~\eqref{eq:privatizer_cdp}), where $\rho_j\overset{\emph{i.i.d.}}{\sim} \mathcal{N}(0,\sigma^2_{nc})$ is the injected Gaussian noise for ensuring the required $(\epsilon, \delta)$-DP and is chosen according to the (high-probability) sensitivity of $\bar{\y}_l$. Specifically, we set the variance of the injected Gaussian noise to the following: 
\begin{equation}
 \sigma_{nc} = \frac{2\sqrt{2(\kappa^2+\sigma^2)H_l\log(2H_l/\delta_1)\ln(1.25/\delta_2)}}{\epsilon |U_l|},  \label{eq:gaussian_noise}
\end{equation}
where $\delta_1\in (0, \delta)$ is the probability of sensitivity concentration of $\bar{\y}_l$ (i.e., Eq.~\eqref{eq:sensitivity_lemma} holds with probability at least $1-\delta_1$) and $\delta_2=\delta-\delta_1$.
By accounting for privacy noise, we update the confidence width function in Eq.~\eqref{eq:confidence_width_dp} with $\sigma_n = \sigma_{nc}\sqrt{2C^2\gamma_T}$, where $C$ is the rare-switching parameter.}

\begin{algorithm}[!t]
\caption{Differentially Private Distributed Phase-then-Batch-based Elimination (DP-DPBE)}
\label{alg:DP-DPBE}
\begin{algorithmic}[1]
\STATE \textbf{Input:}  $\cD\subseteq \R^d$, $\alpha \in (0,1)$, $\beta \in (0,1)$, rare-switching parameter $C$, local noise $\sigma^2$, privacy parameters $\epsilon$ and $\delta$
\STATE \textbf{Initialization:} $l=1$, 
$\cD_1=\cD$, $t_1=0$, and $T_1= 1$
	\WHILE{$t_l<T$}
	\STATE Set $\tau = 1$, $h=0$, $\tau_1=0$ and $\Sigma^2_0(\x)=k(\x,\x)$
	, for all $\x \in \cD_l$
	\WHILE{$\tau \leq T_l$} 
	\STATE $h = h+1$
	\STATE Choose 
	\begin{equation}
	    \ca_h \in \argmax_{\x\in \cD_l} \Sigma_{h-1}^2(\x) 
	\end{equation}
	\STATE Play action $\ca_h$ for $T_{l}(\ca_h)\triangleq \lfloor
	(C^2-1)/\Sigma_{h-1}^2(\ca_h)\rfloor$ times if not reaching $\min\{T,t_l+T_l\}$
	\STATE Update $\tau = \tau+T_l(\ca_h)$, and the posterior variance $\Sigma^2_{h}(\cdot)$ by including $\ca_h$  according to Eq.~\eqref{eq:sigma_update}.
	\ENDWHILE 
	\STATE Let $H_l = h$ denote the total number of batches in this phase.
	\STATE Randomly select $\lceil 2^{\alpha l}\rceil$ participants $U_l$
	\item[] \deemph{\# Operations at each participant }
	\FOR{each participant $u  \in U_l$ }
	\STATE Collect and compute local average reward for every chosen action $\ca\in \A_{H_l}$:  $$y_l^u(\ca) = \frac{1}{T_{l}(\ca)}\sum_{t\in \cT_{l}(\ca)}y_{u,t}$$ 
	\STATE Send the local average reward for every chosen action $\y_l^u\triangleq [y_l^u(\ca)]_{\ca\in \A_{H_l}}$ to the agent 
	\ENDFOR
    \STATE Aggregate local observations for each chosen action $\ca \in \A_{H_l}$: 
    \begin{equation*}
        y_l(\ca) = \frac{1}{|U_l|}\sum_{u\in U_l} y_l^u(\ca)
    \end{equation*}\label{alg_aggregation_dp}
    \STATE Let $\bar{\y}_l = [y_l(\ca_1), \dots, y_l(\ca_{H_l})]$ and \begin{equation*}
	\Tilde{\y}_l = 
		 \bar{\y}_{l}+(\rho_1, \dots, \rho_{H_l}),  
\end{equation*}
where 
$\rho_j\overset{\emph{i.i.d.}}{\sim} \mathcal{N}(0,\sigma^2_{nc})$ \Hi{and  $\sigma_{nc}$ is specified in Eq.~\eqref{eq:gaussian_noise}}.
\label{alg_combine_dp}
	\STATE Update $\Tilde{\mu}_l(\cdot)$:	\begin{equation}
    \Tilde{\mu}_l(\x) \triangleq \vk(\x, \A_{H_l})^{\top}(\K_{\A_{H_l}\A_{H_l}} +\lambda \W_{H_l}^{-1})^{-1}\Tilde{\y}_l 
\end{equation}
	\STATE Eliminate low-rewarding actions from $\cD_l$ based on the confidence width function $\Tilde{w}_l(\cdot)$ in Eq.~\eqref{eq:confidence_width_dp} \Hig{with $\sigma_n = \sigma_{nc}\sqrt{2C^2\gamma_T}$}:
	\begin{equation} 
	\cD_{l+1} = \left\{\x\in \cD_l: \Tilde{\mu}_l(\x)+ \Tilde{w}_l(\x)\geq \max_{\mathbf{b}\in \cD_l}( \Tilde{\mu}_l(\mathbf{b})-\Tilde{w}_l(\mathbf{b}))\right\}. \label{eq:action_elimination_dp}
	\end{equation} \label{alg_elimination_dp}
	\STATE $T_{l+1} = 2T_l$, 	$t = t+T_l$; $l = l+1$
	\ENDWHILE
\end{algorithmic}
\end{algorithm}

\textbf{Differentially Private \texttt{DPBE} in the Local DP Model.}
In the local model, the users do not trust the agent, and thus, each is equipped with a local randomizer $\cR$ to protect its own local reward. Let $Y$ be the set of all possible values of the local reward. Formally, a local randomizer $\cR$ is $(\epsilon, \delta)$-local differentially private (or $(\epsilon, \delta)$-LDP) if for any two user inputs, the probability that $\cR$ outputs two values in $Y$ that are not different by more than a multiplicative factor of $e^{\epsilon}$ and an additive factor of $\delta$. To guarantee LDP, the local randomizer $\cR$ at each user $u$ injects Gaussian noise before sending the local reward observations out to the central agent. That is, 
\begin{equation}
    \cR(\y_l^u) = \y_l^u + (\rho_{u,1},\dots, \rho_{u, H_l}), \label{eq:randomizer_ldp} 
\end{equation}
\Hig{where $\rho_{u,j}{\sim}  \mathcal{N}(0,\sigma^2_{nl})$ is \emph{i.i.d.} across both users and actions and the variance $\sigma^2_{nl}$ is chosen according to the (high-probability) sensitivity of $\y_l^u$ (see Eq.~\eqref{eq:local_sensitivity}).}
Then, the \emph{private} aggregated feedback for the chosen actions in the $l$-th phase in the local DP model becomes
\begin{equation}
\Tilde{\y}_l = \frac{1}{|U_l|}\sum_{u\in U_l} \cR(\y_l^u) =  \frac{1}{|U_l|}\sum_{u\in U_l} \left(	{\y}_{l}^u+(\rho_{u,1}, \dots, \rho_{u,H_l})\right).\label{eq:privatizer_ldp}
\end{equation}

\Hig{We call the differentially private version of \texttt{DPBE} in the local DP model \texttt{LDP-DPBE}. Specifically, we extend the \texttt{DPBE} algorithm (Algorithm~\ref{alg:DP-DPBE}) to \texttt{LDP-DPBE} by employing a local randomizer $\cR$ as in Eq.~\eqref{eq:randomizer_ldp} at each participant in the $l$-th phase and then using the privately aggregated feedback in Eq.~\eqref{eq:privatizer_ldp} to estimate the mean function $\Tilde{\mu}_l(\cdot)$ in Eq.~\eqref{eq:mu_update_dp}. The injected Gaussian noise at each participant is
$\sigma_{nl} = \frac{2\sqrt{2(\kappa^2+\sigma^2)H_l\log(2H_l/\delta_1)\ln(1.25/\delta_2))}}{\epsilon}$,
where $\delta_1\in (0, \delta)$ is the probability of sensitivity concentration of $\bar{\y}_l$ (i.e., Eq.~\eqref{eq:local_sensitivity} holds with probability at least $1-\delta_1$) and $\delta_2=\delta-\delta_1$.
In \texttt{LDP-DPBE}, we update the confidence width function in Eq.~\eqref{eq:confidence_width_dp} with $\sigma_n = \sqrt{\frac{2C^2\sigma_{nl}^2\gamma_T}{|U_l|}}$, where $C$ is the rare-switching parameter. }

\textbf{Differentially Private \texttt{DPBE} in the Shuffle DP Model.} While local DP  provides a more stringent privacy guarantee, it usually incurs larger regret cost \cite{zhou2020local}. The shuffle model is recently proposed to achieve a better tradeoff between regret and privacy \cite{cheu2019distributed}. In the shuffle model, between the users and the agent, there exists a shuffler that permutes the local feedback from the participants before they are observed by the agent so that the agent cannot distinguish between two users' feedback. Thus, an additional layer of randomness is introduced via shuffling, which can often be easily implemented using Cryptographic primitives (e.g., mixnets) due to its simple operation~\cite{bittau2017prochlo}. Specifically, the shuffle DP model consists of three components: a local randomizer $\cR$ at each user side, a shuffler $\cS$ between the users and the agent, and an analyzer $\cA$ at the agent side. Let $\cU_T\triangleq (U_l, \cdots, U_l)$ be the participants throughout the $T$ rounds. Define the (composite) mechanism $\cM_s(\cU_T)\triangleq ((\cS \circ \cR)(U_1), (\cS \circ \cR)(U_2),\ldots, (\cS \circ \cR)(U_L))$, where $(\cS\circ \cR)(U_l)\triangleq \cS (\{\cR(\y_l^{u}) \}_{u\in U_l})$. Formally,  We say the \texttt{DP-DPBE} algorithm  satisfies the shuffle differential privacy (SDP) if the composite mechanism $\cM_s$ is DP, which leads to the following formal definition. 
\begin{definition} (Shuffle Differential Privacy (SDP)). For any $\epsilon \geq 0$ and $\delta\in [0,1]$, the \texttt{DP-DPBE} is $(\epsilon, \delta)$-shuffle differential privacy (or $(\epsilon, \delta)$-SDP) if for any pair $\mathcal{U}_T$ and $\mathcal{U}^{'}_T$ that differ by one user, and for any $Z\in Range(\cM_s)$\footnote{$Rang(\cM)$ denotes the range of the output of the mechanism $\cM$.}:
\begin{equation}
	\mathbb{P}[\cM_s(\mathcal{U}_T)\in Z] \leq e^{\epsilon} \mathbb{P}[\cM_s(\mathcal{U}^{'}_T)\in Z] + \delta.
	\end{equation}
\end{definition}
In our case, we apply a shuffle model to the feedback from participants of every particular phase. That is,
the \emph{private} aggregated feedback for the chosen actions in the $l$-th phase in the shuffle DP model becomes
\begin{equation}
\Tilde{\y}_l = \cA\left(\cS\left(\{\cR(\y_l^u\}_{u\in U_l})\right)\right), \label{eq:privatizer_sdp} 
\end{equation}
where the local randomizer injects a sub-Gaussian noise with variance $\sigma_{ns}^2$, which is \emph{i.i.d.} across both users and actions. Thanks to our phase-then-batch strategy, 
the recently proposed vector summation protocol \cite{cheu2021shuffle} can be extended to our algorithm as \cite{li2022differentially}. 
We present the concrete pseudocodes of $\cR$, $\cS$, and $\cA$ in Algorithm~\ref{alg:shuffler}.

\Hig{We call the differentially private version of \texttt{DPBE} in the shuffle model \texttt{SDP-DPBE}, which is extended from \texttt{DP-DPBE} by using the privately aggregated feedback in Eq.~\eqref{eq:privatizer_sdp}, where $\cR$, $\cS$, and $\cA$ are specified in Algorithm~\ref{alg:shuffler}. For any $\delta_1\in (0, \delta)$, let $\Delta \triangleq B\kappa\sqrt{H_l} +\sqrt{2(\kappa^2+\sigma^2)H_l\log(2H_l/\delta_1)}$. It is not difficult to show that $\Vert \y_l^u\Vert_2\leq \Delta$ with probability at least  $1-\delta_1$. \texttt{SDP-DPBE} employs Algorithm~\ref{alg:shuffler} 
in each phase $l$ with input  $\{\y_l^u\}_{u\in U_l}$, $\Delta$,
    and privacy parameters $\epsilon$ and $\delta_2=\delta-\delta_1$. According to \cite{li2022differentially},
the introduced error for privacy is sub-Gaussian with variance $\sigma_{ns}^2 = O\left(\frac{H_l(\kappa^2+\sigma^2)\log(H_l/\delta_1)\ln(H_l/\delta_2)^2}{\epsilon^2|U_l|^2 }\right)$. 
In \texttt{SDP-DPBE}, we update the confidence width function in Eq.~\eqref{eq:confidence_width_dp} with $\sigma_n = \sigma_{ns}\sqrt{2C^2\gamma_T}$, where $C$ is the rare-switching parameter.
}
\begin{algorithm}[!t]
	\caption{ $\cM:$ Shuffle Protocol for a Set of Vectors with  Users $U$ \cite{li2022differentially}} 
	\label{alg:shuffler}
	\begin{algorithmic}[1]
	\STATE \textbf{Input:} $\{\y^u\}_{u\in U}$, where each $\y^u\in \R^s$, $\Vert \y^u\Vert_2\leq \Delta$, privacy parameters $\epsilon$, $\delta_2\in (0,1)$
	\STATE Let 
	\begin{equation}
	\begin{cases}
		\hat{\epsilon}=\frac{\epsilon}{18\sqrt{\log(2/\delta_2)}} \\
		g\triangleq \max\{\hat{\epsilon}\sqrt{|U|}/(6\sqrt{5\ln{((4s)/\delta_2)}}), \sqrt{s}, 10\}\\
		b\triangleq \lceil \frac{180g^2\ln{(4s/\delta_2)}}{\hat{\epsilon}^2|U|}\rceil\\
		p \triangleq \frac{90g^2\ln{(4s/\delta_2)}}{b\hat{\epsilon}^2|U|} \label{eq:gbp_set}
	\end{cases}    
	\end{equation}
	\item[] \deemph{// Local Randomizer}
	\item[] \textbf{function} 
	$\cR(\y^u)$ \\
	 \begin{ALC@g}
			\FOR{coordinate $j\in [s]$} 
	\STATE Shift data to enforce non-negativity: $w_{u,j} = (\y^u)_j + \Delta, \forall u\in U$
	\item[] //randomizer for each entry
	\STATE Set $\Bar{w}_{u,j} \gets \lfloor w_{u,j}g/(2\Delta)\rfloor$  \hfill \deemph{//$\max|(\y^u)_j+\Delta|\leq 2\Delta$}
	\STATE Sample rounding value $\gamma_1 \sim \textbf{Ber}(w_{u,j}g/(2\Delta) - \Bar{w}_{u,j})$
	\STATE Sample privacy noise value $\gamma_2 \sim \textbf{Bin}(b,p)$
	\STATE Let $\phi_j^u$ be a multi-set of $(g+b)$ bits associated with the $j$-th coordinate of user $u$, where $\phi_j^u$ consists of $\Bar{w}_{u,j}+\gamma_1+\gamma_2$ copies of 1 and $g+b-(\Bar{w}_{i,j}+\gamma_1+\gamma_2)$ copies of 0
	\ENDFOR
	\STATE Report $\{(j,\phi^u_j)\}_{j\in [s]}$ to the shuffler
	\end{ALC@g}
    \item[] \textbf{end function}
    \item[] \deemph{// Shuffler}
    \item[] \textbf{function} 
    $\mathcal{S}(\{(j, \bm{\phi}_j)\}_{j\in[s]})$  \quad //$\bm{\phi}_j = (\phi^u_j)_{u\in U}$\\
    \begin{ALC@g}
    \FOR{each coordinate $j\in [s]$}
    \STATE Shuffle and output all  $(g+b)|U|$ bits in $\bm{\phi}_j$ 
    \ENDFOR
    \end{ALC@g}
    \item[] \textbf{end function} \\
    \item[] \deemph{// Analyzer}
    \item[] \textbf{function} 
    $\mathcal{A}(\mathcal{S}(\{(j, \bm{\phi}_j)\}_{j\in[s]})$
    \begin{ALC@g}
	\FOR{coordinate $j\in [s]$}
	\STATE Compute $z_j \gets \frac{2\Delta}{g|U|} ((\sum_{i=1}^{(g+b)|U|} (\bm{\phi}_{j})_i)-b|U|p) $   \quad // $(\bm{\phi}_j)_i$ denotes the $i$-th bit in $\bm{\phi}_j$
	\STATE Re-center: $o_j \gets z_j - \Delta$
	\ENDFOR
	\STATE  Output the estimator of vector average ${o}=(o_j)_{j\in [s]}$
	\end{ALC@g}
	\textbf{end function}
	\end{algorithmic}
\end{algorithm}

\subsection{Performance Guarantee}
For the \texttt{DP-DPBE} algorithm incorporated with the above local and shuffle DP models, we provide the DP guarantee and regret in the following. 
\Hig{
\begin{theorem}[DP guarantee]\label{thm:dp_app} Under Assumptions~\ref{ass:global_func}, \ref{ass:local_func}, \ref{ass:noise}, and \ref{ass:one-time}, for any $\epsilon>0$ and     $\delta\in (0,1)$,
\begin{itemize}
    \item[i)]\texttt{LDP-DPBE}   guarantees $(\epsilon, \delta)$-LDP;
    \item[ii)]  \texttt{SDP-DPBE}   guarantees $(\epsilon, \delta)$-SDP. 
\end{itemize}
\end{theorem} 
}
We achieve the above LDP guarantee of i) directly by employing the Gaussian mechanism given the (high-probability) sensitivity of $\y_l^u$. 
In the shuffle model, we follow the shuffle protocol for each phase in \cite{li2022differentially} and derive the corresponding SDP guarantee from Theorem~A.2 therein.

From the above results, we derive that compared to the local model  the shuffle model injects much less noise ($\sigma_{ns}^2$ vs. $\sigma_{nl}^2$)  without requiring a trusted agent.  In the following, we present the regret performance of \texttt{DP-DPBE} in these two DP models.

\Hig{
\begin{theorem}[LDP-DPBE]\label{thm:regret_ldp} 
Under Assumptions~\ref{ass:global_func}, \ref{ass:local_func}, and \ref{ass:noise}, the \texttt{LDP-DPBE} algorithm with $\beta=\frac{1}{|\cD|T}$ achieves the following expected regret: 
\begin{equation}
\begin{aligned}
\E[R(T)] &= O(T^{1-\alpha/2}\sqrt{\log (|\cD|T)})+ O\left(\frac{\ln(1/\delta)\gamma_T T^{1-\alpha/2}\sqrt{\log(|\cD| T)}}{\epsilon}\right).
\end{aligned}
\end{equation}
\end{theorem}

\begin{theorem}[SDP-DPBE]\label{thm:regret_sdp} 
Under Assumptions~\ref{ass:global_func}, \ref{ass:local_func}, and \ref{ass:noise}, the \texttt{SDP-DPBE} algorithm with $\beta=\frac{1}{|\cD|T}$ achieves the following  expected regret:
\begin{equation}
\begin{aligned}
\E[R(T)] &= O(T^{1-\alpha/2}\sqrt{\log (|\cD|T)})+ O\left(\frac{\ln^{3/2}(\gamma_T/\delta)\gamma_T T^{1-\alpha}\sqrt{\log(|\cD| T)}}{\epsilon}\right).
\end{aligned}
\end{equation}
\end{theorem}
}

\begin{table}[!t]
\caption{Regret of \texttt{DP-DPBE} in Different DP Models}
\label{tab:private_regret}
\begin{tabular}{c|c} 
    \toprule
    Algorithms & Regret \\  \hline
      \texttt{DPBE}    & $O(T^{1-\alpha/2}\sqrt{\log (|\cD|T)})$ \\  
      \texttt{CDP-DPBE}   & $O(T^{1-\alpha/2}\sqrt{\log (|\cD|T)}) + O\left(\frac{\ln(1/\delta)\gamma_TT^{1-\alpha}\sqrt{\log(|\cD|T)}}{\epsilon}\right)$ \\ 
      \texttt{LDP-DPBE} & $O(T^{1-\alpha/2}\sqrt{\log (|\cD|T)}) + O\left(\frac{\ln(1/\delta)\gamma_TT^{1-\alpha/2}\sqrt{\log(|\cD|T)}}{\epsilon}\right)$ \\ 
      \texttt{SDP-DPBE}  & $O(T^{1-\alpha/2}\sqrt{\log (|\cD|T)}) + O\left(\frac{\ln^{3/2}(\gamma_T/\delta)\gamma_TT^{1-\alpha}\sqrt{\log(|\cD|T)}}{\epsilon}\right)$\\ 
      \bottomrule
      \multicolumn{2}{l}{\footnotesize Notes: CDP-DPBE, LDP-DPBE, and SDP-DPBE represent the \texttt{DP-DPBE} algorithm in the central, local,}\\ 
      \multicolumn{2}{l}{\footnotesize  and shuffle models, respectively, which guarantee $(\epsilon, \delta)$-DP, $(\epsilon, \delta)$-LDP, and $(\epsilon, \delta)$-SDP, respectively.} \\
\end{tabular}
\end{table}
We omit the proofs for the above two theorems because they can be derived by directly replacing $\sigma_n$ of the central model with $\sigma_n = \sqrt{\frac{2C^2\sigma_{nl}^2\gamma_T}{|U_l|}}$  of the local model and $\sigma_n = \sigma_{ns}\sqrt{2C^2\gamma_T} $ of the shuffle model. See Appendix~\ref{app:proof_cdp}.
\subsection{Proofs for DP Guarantees}\label{app:proof_dp_guarantee}
Before providing the DP guarantee of the \texttt{DPBE} algorithm in the three DP models, we first show the $\ell_2$ sensitivity of 
$\bar{\y}_l$, which is a key parameter to decide the Gaussian noise. 

\Hig{
\begin{lemma}
\label{lem:global_sensitivity} 
Let $\cU_T, \cU_T^{\prime} \subseteq \mathcal{U}$ be two sets of participants in \texttt{DPBE} differing on a single user that is participating in the $l$-th phase, and let $\bar{\y}_l$ and $\bar{\y}_l^{\prime}$ be the corresponding average local reward. For any $\delta_1\in (0,1)$, we have that with probability at least $1-\delta_1$, the maximal $\ell_2$ distance between $\bar{\y}_l$ and $\bar{\y}_l^{\prime}$ is bounded by
\begin{equation}
\begin{aligned}
\max |\bar{\y}_l^{\prime} - \bar{\y}_l| \leq 2\frac{\sqrt{(\kappa^2+\sigma^2)H_l\log (2H_l/\delta_1)}}{|U_l|},  \label{eq:sensitivity_lemma} 
\end{aligned}
\end{equation}
where $H_l$ denotes the dimension of $\bar{\y}_l$ and $\sigma^2$ is the variance of the noisy observations. 
\end{lemma}
}

\begin{proof}
Let $U_l$, $U_l^{\prime}$ be the sets of participating users in $l$-th phase corresponding to $\cU_T$ and $\cU_T^{\prime}$ respectively. We have $|U_l| = |U_l^{\prime}|$ and the maximal $\ell_2$ distance between $\bar{\y}_l$, $\bar{\y}_l^{\prime}$ is the following: 
\begin{equation}
\begin{aligned}
\max |\bar{\y}_l^{\prime} - \bar{\y}_l|
&=  \max_{\cU_T, \cU_T^{\prime}} \left\Vert \frac{1}{| U_l|}\sum_{u\in  U_l^{\prime}}{\y}_{l}^u - \frac{1}{| U_l|}\sum_{u\in U_l}\y_{l}^u \right\Vert_2\\
&= \frac{1}{| U_l|}\max_{u,u' \in \cU}\Vert \y_{l}^{u^{\prime}} - \y_{l}^{u}\Vert_2. \label{eq:sensitivity}
\end{aligned}
\end{equation}
For any chosen action $\ca \in \A_{H_l}$, we have the following result:
\begin{equation*}
\begin{aligned}
 |y_l^{u^{\prime}}(\ca) - y_l^u(\ca)|
 & = \left |\frac{1}{T_l(\ca)}\sum_{t\in \cT_l(\ca)} y_{u^{\prime},t}- \frac{1}{T_l(\ca)}\sum_{t\in \cT_l(\ca)} y_{u,t}\right |\\
 & = \left |\frac{1}{T_l(\ca)}\sum_{t\in \cT_l(\ca)}( y_{u^{\prime},t}-  y_{u,t})\right |\\
 & = \left |\frac{1}{T_l(\ca)}\sum_{t\in \cT_l(\ca)} (f_{u^{\prime}}(\x_t) + \eta_{u^{\prime},t}-  f_{u}(\x_t) - \eta_{u,t})\right |\\
  & \leq \frac{1}{T_l(\ca)}\sum_{t\in \cT_l(\ca)} \left |f_{u^{\prime}}(\x_t) + \eta_{u^{\prime},t}-  f_{u}(\x_t) - \eta_{u,t}\right |.
\end{aligned}
\end{equation*}
Note that $f_u(\x)\sim \cN(f(\x), k(\x,\x))$, $\eta_{u,t} \sim \cN(0, \sigma^2)$, and the participating users are independent from each other. We have $(f_{u^{\prime}}(\x_t) + \eta_{u^{\prime},t}-  f_{u}(\x_t) - \eta_{u,t}) \sim \cN(0, 2(k(\x_t,\x_t)+\sigma^2))$. According to the concentration property of Gaussian distribution, we have with probability at least $1-\delta_1$,
\begin{equation}
 |f_{u^{\prime}}(\x_t) + \eta_{u^{\prime},t}-  f_{u}(\x_t) + \eta_{u,t})|\leq 2\sqrt{(k(\x_t,\x_t)+\sigma^2)\log(2/\delta_1)}\leq 2\sqrt{(\kappa^2+\sigma^2)\log(2/\delta_1)} ,
\end{equation}
which results in $ |y_l^{u^{\prime}}(\ca) - y_l^u(\ca)|\leq 2\sqrt{(\kappa^2+\sigma^2)\log(2/\delta_1)} $ for any particular $\ca \in \A_{H_l}$ with probability at least $1-\delta_1$. By substituting the above result into Eq.~\eqref{eq:sensitivity} and applying union bound, we have that with probability at least $1-\delta_1$, the following is satisfied:
\Hig{
\begin{equation}
    \max_{u,u' \in \cU}\Vert \y_{l}^{u^{\prime}} - \y_{l}^{u}\Vert_2 \leq 
    2\sqrt{H_l(\kappa^2+\sigma^2)\log(2H_l/\delta_1)}, \label{eq:local_sensitivity}
\end{equation}
}
and then with probability at least $1-\delta_1$, the $\ell_2$ distance between $\bar{\y}_l$ and $\bar{\y}_l^{\prime}$ is bounded by
\begin{equation}
    \max |\bar{\y}_l^{\prime} - \bar{\y}_l| \leq 
    \frac{\max_{u,u' \in \cU}\Vert \y_{l}^{u^{\prime}} - \y_{l}^{u}\Vert_2}{|U_l|} 
    \leq
    \frac{2\sqrt{H_l(\kappa^2+\sigma^2)\log(2H_l/\delta_1)}}{|U_l|},
\end{equation}
where the last step is because $H_l$ is the dimension of $\y_l^u$ and also the number of actions in $\A_{H_l}$.
\end{proof}
For both the central model and the local model, we employ the Gaussian mechanism in the differential privacy literature, which is described in the following. 
\begin{theorem}\label{thm:gaussian_mech}
({Gaussian Mechanism} \cite{dwork2014algorithmic}).  
Given any vector-valued function\footnote{We use the superscript $^*$ to indicate that the length could be varying.} $f: \cU^{*} \to \R^s$, define $\Delta_2\triangleq \max_{\cU_1, \cU_2^{\prime} \subseteq \cU} \Vert f(\cU_1)-f(\cU_2)\Vert_2$. Let $\sigma = \Delta_2 \sqrt{2\ln (1.25/\delta)}/\epsilon$. The Gaussian mechanism, which adds independently drawn random noise from $\mathcal{N}(0,\sigma^2)$ to each output of $f(\cdot)$, i.e. returning $f(\cU)+(\rho_1, \dots, \rho_s)$ with $\rho_j\overset{\text{i.i.d.}}{\sim} \mathcal{N}(0,\sigma^2)$, ensures $(\epsilon, \delta)$-DP.
\end{theorem}

\begin{proof}[Proof of Theorem~\ref{thm:dp}]
Let $E$ denote the event that 
Eq.~\eqref{eq:sensitivity_lemma} holds, and thus, $\bP[E]\geq 1-\delta_1$. Let $\Delta_2 \triangleq \max |\bar{\y}_l^{\prime} - \bar{\y}_l| $. If $E$ holds, adding independently drawn noise from $\cN\left(0,\frac{2\Delta_2^2\ln(1.25/\delta_2)}{\epsilon}\right)$ to each element of $\bar{\y}_l$, i.e., returning $\bar{\y}_l + (\rho_1, \cdots, \rho_{H_l})$ with $\rho_j\overset{i.i.d.}{\sim} \cN\left(0,\frac{2\Delta_2^2\ln(1.25/\delta_2)}{\epsilon}\right)$, ensures $(\epsilon, \delta_2)$-DP . Specifically, the following inequality holds
\begin{equation}
    \bP[\cM(\cU_T)\in Z| E] \leq e^{\epsilon}\cP[\cM(\cU_T^{\prime})\in Z| E] + \delta_2.
\end{equation}
Then, we have 
\begin{equation}
\begin{aligned}
 \bP[\cM(\cU_T)\in Z] 
 &\leq  \bP[\cM(\cU_T)\in Z| E]\bP[E] + 1- \bP[E] \\
 & \leq (e^{\epsilon}\cP[\cM(\cU_T^{\prime})\in Z| E] + \delta_2) \bP[E] + \delta_1\\
 &\leq e^{\epsilon}\cP[\cM(\cU_T^{\prime})\in Z| E]\bP[E] + \delta_2 + \delta_1 \\
 &\leq e^{\epsilon}\cP[\cM(\cU_T^{\prime})\in Z| E]\bP[E] + \delta_2 + \delta_1\\
 &\leq e^{\epsilon}\cP[\cM(\cU_T^{\prime})\in Z, E] + \delta_2 + \delta_1\\
  &\leq e^{\epsilon}\cP[\cM(\cU_T^{\prime})\in Z] + \delta, \label{eq:delta1+delta2}
\end{aligned}
\end{equation}
where $\delta=\delta_1+\delta_2$.
\end{proof}
Similarly, we can derive the $(\epsilon, 
\delta)$-LDP. Meanwhile, we can achieve $(\epsilon, \delta)$-SDP by combining the analysis in Eq.~\eqref{eq:delta1+delta2} and the proof for Theorem~A.2 in \cite{li2022differentially}.

\subsection{Proof of Theorem~\ref{thm:regret_cdp}}\label{app:proof_cdp}
Following a similar line to the proof for Theorem~\ref{thm:regret_upper_bound}, we first provide the key concentration inequality under \texttt{DP-DPBE} in Theorem~\ref{thm:concentration_ineq_dp}. 
\begin{theorem}\label{thm:concentration_ineq_dp}
For any particular phase $l$, with probability at least $1-6\beta$, the following holds
\begin{equation}
    |f(\x) - \Tilde{\mu}_l(\x)| \leq   \Tilde{w}_l(\x),\label{eq:concentration_ineq_dp}
\end{equation}
where mean function $\Tilde{\mu}_l(\x)$ and confidence width function $\Tilde{w}_l(\x)$ are defined in Eq.~\eqref{eq:mu_update_dp} and Eq.~\eqref{eq:confidence_width_dp}.
\end{theorem}
\begin{proof}
In this proof, we will show the following concentration inequality holds for any $\x\in \cD$
\begin{equation}
    \bP[|f(\x) - \Tilde{\mu}_l(\x)| \geq  \Tilde{w}_l(\x)] \leq 6\beta.
\end{equation}
Let $\bm{\rho} \triangleq (\rho_1, \dots, \rho_{H_l})$. Note that 
\begin{equation}
\begin{aligned}
\Tilde{\mu}_l(\x) &= \vk(\x, \A_{H_l})^{\top}(\K_{\A_{H_l}\A_{H_l}} +\lambda \W_{H_l}^{-1})^{-1}\Tilde{\y}_l \\
&= \vk(\x, \A_{H_l})^{\top}(\K_{\A_{H_l}\A_{H_l}} +\lambda \W_{H_l}^{-1})^{-1}(\bar{y}_l+\bm{\rho} ) 
\\  
&= \bar{\mu}_l(x) + \vk(\x, \A_{H_l})^{\top}(\K_{\A_{H_l}\A_{H_l}} +\lambda \W_{H_l}^{-1})^{-1}\bm{\rho} .
\end{aligned}
\end{equation}
Then, we have 
\begin{equation}
    |f(\x) - \Tilde{\mu}_l(\x)| \leq |f(\x) - \bar{\mu}_l(\x)| + | \vk(\x, \A_{H_l})^{\top}(\K_{\A_{H_l}\A_{H_l}} +\lambda \W_{H_l}^{-1})^{-1}\bm{\rho} |.
\end{equation}
For any $\x\in \cD$, we have
\begin{equation}
    \begin{aligned}
    &\bP\left[|f(\x) - \Tilde{\mu}_l(\x)|\geq  \Tilde{w}_l(\x)\right] \\
    \leq & \bP\left[\left|f(\x)- \bar{\mu}_l(x)\right| +  \left| \vk(\x, \A_{H_l})^{\top}(\K_{\A_{H_l}\A_{H_l}} +\lambda \W_{H_l}^{-1})^{-1}\bm{\rho} \right| \geq  w_l(\x) + 2C\sqrt{\gamma_{T} \sigma_n^2\log(1/\beta)}\right]\\
    \leq & \bP\left[\left|f(\x)- \bar{\mu}_l(x)\right| \geq  w_{l}(\x) \right] + \bP\left[  \left| \vk(\x, \A_{H_l})^{\top}(\K_{\A_{H_l}\A_{H_l}} +\lambda \W_{H_l}^{-1})^{-1}\bm{\rho} \right| \geq    2C\sqrt{\gamma_{T} \sigma_n^2\log(1/\beta)}\right]\\
    \leq & 4\beta + \bP\left[  \left|\vk(\x, \A_{H_l})^{\top}(\K_{\A_{H_l}\A_{H_l}} +\lambda \W_{H_l}^{-1})^{-1}\bm{\rho} \right| \geq    2C\sqrt{\gamma_{T} \sigma_n^2\log(1/\beta)}\right],
    \end{aligned} \label{eq:union_bound_dp}
\end{equation}
where the first inequality is due to $\Tilde{w}_l(x) = w_l(x) + \sqrt{{2 \sigma_n^2\log(1/\beta)}}$ from Eq.~\eqref{eq:confidence_width_dp}, the second inequality is from union bound, and the last one is from Theorem~\ref{thm:concentration_ineq}. Hence, it remains to bound the second probability in Eq.~\eqref{eq:union_bound_dp}. 

Recall that $\bm{\rho} =(\rho_1, \dots, \rho_{H_l})$ where $\rho_{j} \overset{\emph{i.i.d.}}{\sim} \cN(0, \sigma_{nc}^2)$. Then, $\vk(\x, \A_{H_l})^{\top}(\K_{\A_{H_l}\A_{H_l}} +\lambda \W_{H_l}^{-1})^{-1}\bm{\rho} $ is the sum of $H_l$ \emph{i.i.d.} Gaussian variables, and the total variance (denoted by $\sigma_{\text{sum}}^2$) is
\begin{equation}
 \sigma^2_{\mathrm{sum}}   = \vk(\x, \A_{H_l})^{\top}(\K_{\A_{H_l}\A_{H_l}} +\lambda \W_{H_l}^{-1})^{-1} (\K_{\A_{H_l}\A_{H_l}} +\lambda \W_{H_l}^{-1})^{-1} \vk(\x, \A_{H_l}) \sigma_{nc}^2. 
 \label{eq:sigma_sum}
\end{equation}
Notice that
\begin{equation}
    \begin{aligned}
   & \vk(\x, \A_{H_l})^{\top}(\K_{\A_{H_l}\A_{H_l}} +\lambda \W_{H_l}^{-1})^{-1} (\K_{\A_{H_l}\A_{H_l}} +\lambda \W_{H_l}^{-1})^{-1} \vk(\x, \A_{H_l}) \\
    =&  \varphi(\x)^{\top} \Phi_{H_l}^{\top} (\Phi_{H_l}\Phi_{H_l}^{\top}+\lambda\W_{H_l}^{-1})^{-1}(\Phi_{H_l}\Phi_{H_l}^{\top}+\lambda\W_{H_l}^{-1})^{-1}\Phi_{H_l}\varphi(\x)  \\
    =& \varphi(\x)^{\top} \Phi_{H_l}^{\top} \W_{H_l}^{1/2} (\W_{H_l}^{1/2}\Phi_{H_l}\Phi_{H_l}^{\top}\W_{H_l}^{1/2}+\lambda \I)^{-1}\W_{H_l}^{1/2} \cdot \W_{H_l}^{1/2} (\W_{H_l}^{1/2}\Phi_{H_l}\Phi_{H_l}^{\top}\W_{H_l}^{1/2}+\lambda \I)^{-1}\W_{H_l}^{1/2} \Phi_{H_l}\varphi(\x) \\
    =& \varphi(\x)^{\top}  (\Phi_{H_l}^{\top}\W_{H_l}^{1/2}\W_{H_l}^{1/2}\Phi_{H_l}+\lambda \I)^{-1}\Phi_{H_l}^{\top} \W_{H_l}^{1/2} \cdot  \W_{H_l} \cdot 
    \W_{H_l}^{1/2}\Phi_{H_l} (\Phi_{H_l}^{\top}\W_{H_l}^{1/2}\W_{H_l}^{1/2}\Phi_{H_l}+\lambda \I)^{-1} \varphi(\x) \\
    = & \varphi(\x)^{\top}  (\Phi_{H_l}^{\top}\W_{H_l}\Phi_{H_l}+\lambda \I)^{-1}\Phi_{H_l}^{\top}\W_{H_l}^{2}\Phi_{H_l} (\Phi_{H_l}^{\top}\W_{H_l}\Phi_{H_l}+\lambda \I)^{-1}\varphi(\x) \\
    \overset{(a)}{\leq }  & T_l \varphi(\x)^{\top}  (\Phi_{H_l}^{\top}\W_{H_l}\Phi_{H_l}+\lambda \I)^{-1}\Phi_{H_l}^{\top}\W_{H_l}\Phi_{H_l} (\Phi_{H_l}^{\top}\W_{H_l}\Phi_{H_l}+\lambda \I)^{-1}\varphi(\x) \\
    = & T_l \varphi(\x)^{\top}  (\Phi_{H_l}^{\top}\W_{H_l}\Phi_{H_l}+\lambda \I)^{-1}(\Phi_{H_l}^{\top}\W_{H_l}\Phi_{H_l} +\sigma_n^2\I) (\Phi_{H_l}^{\top}\W_{H_l}\Phi_{H_l}+\lambda \I)^{-1}\varphi(\x) \\
    & - \lambda T_l \varphi(\x)^{\top}  (\Phi_{H_l}^{\top}\W_{H_l}\Phi_{H_l}+\lambda \I)^{-1} (\Phi_{H_l}^{\top}\W_{H_l}\Phi_{H_l}+\lambda \I)^{-1}\varphi(\x) \\
    \leq & T_l \varphi(\x)^{\top}   (\Phi_{H_l}^{\top}\W_{H_l}\Phi_{H_l}+\lambda \I)^{-1}\varphi(\x) \\
    \overset{(b)}{=} & T_l \varphi(\x)^{\top}   (\Phi_{\tau_{H_l}}^{\top}\Phi_{\tau_{H_l}}+\lambda \I)^{-1}\varphi(\x) \\
    \overset{(c)}{=}&\frac{T_l\sigma_{\tau_{H_l}}^2(\x)}{\lambda} \\
    \overset{(d)}{=} &
    \frac{T_l\Sigma_{H_l}^2(\x)}{\lambda} =\frac{T_l\Sigma_{H_l}^2(\x)}{\sigma^2}  \leq 2C^2\gamma_{T_l},
    \end{aligned}
\end{equation}
where $(a)$ is from $\Phi_{H_l}^{\top}\W_{H_l}^2\Phi_{H_l}\leq \Phi_{H_l}^{\top}(T_l\I)\W_{H_l}\Phi_{H_l}=T_l\Phi_{H_l}^{\top}\W_{H_l}\Phi_{H_l}$ because each diagonal entry of $W_{H_l}$ satisfies $[W_{H_l}]_{hh} = T_l(\ca_h)\leq T_l$, $(b)$ is based on Eq.~\eqref{eq:useful_rs_1}, $(c)$ is from Eq.~\eqref{eq:sigma_feature_space}, and $(d)$ is according to the equivalence representation in Lemma~\ref{lem:equivalent_representation}. The last step is from the result in Lemma~\ref{lem:max_variance}.

Substituting the above result into Eq.~\eqref{eq:sigma_sum}, we have 
\begin{equation}
    \sigma_{\mathrm{sum}}^2 \leq 2C^2\gamma_{T_l}\sigma_{nc}^2 = \sigma_n^2.
\end{equation}
According to the tail bound of Gaussian variables, we have 
\begin{equation*}
    \bP\left[  \left|\vk(\x, \A_{H_l})^{\top}(\K_{\A_{H_l}\A_{H_l}} +\lambda \W_{H_l}^{-1})^{-1}\bm{\rho} \right| \geq    \sqrt{2\sigma_n^2 \log(1/\beta)}\right]\leq 2\exp \left\{-\frac{4C^2\gamma_T\sigma_{nc}^2\log(1/\beta)}{2\sigma_{\mathrm{sum}}^2}\right\} \leq 2\beta.
\end{equation*}
\end{proof}

\begin{proof}[Proof of Theorem~\ref{thm:regret_cdp}]
Similar to the proof of Theorem~\ref{thm:regret_upper_bound}, we, to prove Theorem~\ref{thm:regret_cdp}, first present three results when the concentration inequality in Theorem~\ref{thm:concentration_ineq_dp} holds, then obtain an upper bound for the regret incurred in a particular phase $l>2$ with high probability, and finally sum up the regret over all phases.

\smallskip
\textbf{1) Three observations when Eq.~\eqref{eq:concentration_ineq_dp} holds}

Define a ``good'' event when Eq.~\eqref{eq:concentration_ineq_dp} holds in the $l$-th phase as:
\begin{equation*}
    \Tilde{\cE}_l \triangleq \left\{\forall \x\in \cD_l,  \left|f(\x)-\tilde{\mu}_l(\x)\right|\leq \tilde{w}_l(\x)  \right\}.  
\end{equation*}
We have $\bP[\Tilde{\cE}_l] \geq 1-6|\cD|\beta$ via the union bound.
Then, similar to the non-private case, under event $\Tilde{\cE}_l$ in the $l$-th phase, we have the following three observations:
\begin{itemize}
    \item[\textbf{1.}] For any optimal action $\x^*\in \argmax_{\x\in \cD} f(\x)$, if $\x^*\in \cD_l$, then $\x^* \in \cD_{l+1}$. 
    \item[\textbf{2.}] Let $f^* = \max_{\x\in \cD} f(\x)$. Supposed that $x^*\in \cD_l$. For any $\x \in \cD_{l+1}$, its reward gap from the optimal reward is bounded by $4\max_{\x\in \cD_l}\tilde{w}_l(\x)$, i.e., 
    $$f^* - f(\x)\leq 4\max_{\x\in \cD_l}\tilde{w}_l(\x).$$ 
    \item[\textbf{3.}] The confidence width function in the private setting     satisfies 
    \begin{equation}
    \max_{\x\in\cD_l} \tilde{w}_l(\x) \leq  \max_{\x\in \cD_l} w_l(\x) + \frac{G_1\gamma_T\sqrt{2\log(1/\beta)}}{|U_l|},
    \end{equation}
    where $G_1\triangleq \frac{8C^2\sqrt{2(\kappa^2+\sigma^2)\sigma^2\log(1/\delta_1)\ln(1.25/\delta_2)}}{\epsilon \sqrt{C^2-1}}$.
\end{itemize}
The first two observations can be derived similar to the non-private case. Regarding the third observation, we have the confidence width function in the private setting $\tilde{w}_l(\x) = w_l(\x) + \sqrt{2\sigma_n^2\log(1/\beta)} $ and     \begin{equation*}
\begin{aligned}
& \sqrt{2\sigma_n^2\log(1/\beta)} \\
& = 2C\sqrt{\gamma_T\sigma_{nc}^2\log(1/\beta)}\\
&=  \frac{4C\sqrt{2(\kappa^2+\sigma^2)H_l\gamma_T\log(1/\delta_1)\ln(1.25/\delta_2)\log(1/\beta)}}{\epsilon |U_l|}\\
& \overset{(a)}{\leq} \frac{8C^2\gamma_T\sqrt{2(\kappa^2+\sigma^2)\sigma^2\log(1/\delta_1)\ln(1.25/\delta_2)\log(1/\beta)}}{\epsilon |U_l| \sqrt{C^2-1}} \\
& \leq \underbrace{\frac{8C^2\sqrt{2(\kappa^2+\sigma^2)\sigma^2\log(1/\delta_1)\ln(1.25/\delta_2)}}{\epsilon \sqrt{C^2-1}}}_{G_1} \cdot \frac{\gamma_T\sqrt{2\log(1/\beta)}}{ |U_l|},
\end{aligned}
\end{equation*}
where $(a)$ is from Lemma~\ref{lem:bound_h}.

\smallskip 
\textbf{2) Regret in a specific phase $l>2$.}
 
Under event $\tilde{\cE}_{l-1}$, the regret incurred in the $l$-th phase is
\begin{equation*}
\begin{aligned}
 &\sum_{t \in \cT_l} f^* - f(\x_t)  \\
 \leq & \sum_{t\in\cT_l} 4\max_{\x\in \cD_{l-1}}\tilde{w}_{l-1}(\x)\\
 \leq &  4T_l\max_{\x\in \cD_{l-1}}w_{l-1}(\x)\\
 \overset{(a)}{\leq}& 4T_l\max_{\x\in \cD_{l-1}}w_{l-1}(\x) + 4T_l  \cdot \frac{G_1\gamma_T\sqrt{2\log(1/\beta)}}{|U_{l-1}|}  \\ 
\leq & 4T_l \max_{\x\in \cD_{l-1}}w_{l-1}(\x) +4G_1\gamma_T\sqrt{2\log(1/\beta)}2^{(1-\alpha)(l-1)} \\
 \leq & 4\sqrt{2\kappa^2\log (1/\beta)}\sqrt{2^{(2-\alpha)(l-1)}} + 8\sigma C \sqrt{2\gamma_T\log(1/\beta)}\sqrt{2^{(1-\alpha)(l-1)}} + 8\sigma BC\sqrt{\gamma_T 2^{l-1}} \\
 &+ 4G_1\gamma_T\sqrt{2\log(1/\beta)}2^{(1-\alpha)(l-1)},
\end{aligned}
\end{equation*}
where $(a)$ is from Observation~3 and the last step is from Eq.~\eqref{eq:simple_regret}.

\smallskip
\textbf{3) Total regret.}

Define $\tilde{\cE}_g$ as the event where the ``good'' event occurs in every phase in the private setting, i.e., $\tilde{\cE}_g \triangleq \bigcap_{l=1}^L \tilde{\cE}_l$. It is not difficult to obtain  $\bP[\cE_g] \geq 1- 6|\cD|\beta L$ by applying union bound. At the same time, the total regret under event $\tilde{\cE}_g$ becomes
\begin{equation}
    \begin{aligned}
    R_g &= \sum_{l=1}^L \sum_{t \in \cT_l} (f^* - f(\x_t)  )\\
    & \leq 2B\kappa + \sum_{l=2}^L 4\sqrt{2\kappa^2\log (1/\beta)}\sqrt{2^{(2-\alpha)(l-1)}} \\
    &+ \sum_{l=2}^L 8\sigma C \sqrt{2\gamma_T\log(1/\beta)}\sqrt{2^{(1-\alpha)(l-1)}} \\
    &+ \sum_{l=2}^L 8\sigma  BC\sqrt{\gamma_T2^{l-1}}  + \sum_{l=2}^L 4G_1\gamma_T\sqrt{2\log(1/\beta)}2^{(1-\alpha)(l-1)}\\
    & \leq 2B\kappa + 4\sqrt{2\kappa^2\log (1/\beta)} \cdot 4 \sqrt{2^{(L-1)(2-\alpha)}} \\
    &+ 8\sigma C\sqrt{2\gamma_T\log(1/\beta)}\cdot C_1\sqrt{2^{(1-\alpha)(L-1)}}  \quad \quad  \quad \quad \quad \hfill \left(C_1 \triangleq \frac{\sqrt{2^{1-\alpha}}}{\sqrt{2^{1-\alpha}}-1}\right)\\  
    &+ 8\sigma BC\sqrt{\gamma_T}\cdot 4\sqrt{2^{L-1}} \\
    & + 4G_1\gamma_T\sqrt{2\log(1/\beta)}\cdot C_22^{(1-\alpha)(L-1)} \quad \quad \quad \quad \quad \quad \quad \left(C_2\triangleq \frac{2^{1-\alpha}}{2^{1-\alpha}-1}\right)\\
    &{\leq} 2B\kappa+  16\sqrt{2\kappa^2\log (1/\beta)}T^{1-\alpha/2}+8\sigma C_1C  \sqrt{2\gamma_T\log(1/\beta)T^{1-\alpha}} \\
    &+ 32\sigma BC\sqrt{\gamma_T T}+4C_2G_1\gamma_T\sqrt{2\log(1/\beta)}T^{1-\alpha},
    \end{aligned}
\end{equation}
where the last step is is due to $2^{L-1} \leq T$ and $L\leq \log(2T)$ since $\sum_{l=1}^{L-1} T_l +1 \leq T$. 
On the other hand, $R_b \leq 2B\kappa T$ 
since $|\max_{\x\in\cD} f(\x)-f(\x)|\leq 2B\kappa$ for all $\x \in \cD$. 
Choose $\beta = 1/(|\cD|T)$ in Algorithm~\ref{alg:DP-DPBE}. Then, the expected regret is:
\begin{equation}
\begin{aligned}
    &\E[R(T)] = \bP[\tilde{\cE}_g]R_g + (1-\bP[\tilde{\cE}_g]) R_b \\
    & \leq R_g + 6|\cD|\beta L \cdot 2B\kappa T\\
    & \leq 2B\kappa+  16\sqrt{2\kappa^2\log (1/\beta)}T^{1-\alpha/2}+8\sigma C_1C  \sqrt{2\gamma_T\log(1/\beta)T^{1-\alpha}} + 32\sigma BC\sqrt{\gamma_T T} \\
    &+ 4C_2G_1\gamma_T\sqrt{2\log(1/\beta)}T^{1-\alpha}+ 12B\kappa|\cD|\beta LT\\
    & = 2B\kappa+  16T^{1-\alpha/2}\sqrt{2\kappa^2\log (|\cD|T)}+8\sigma C_1C  \sqrt{2\gamma_T T^{1-\alpha}\log (|\cD|T)} + 32\sigma BC\sqrt{\gamma_T T} \\
    & + 4C_2G_1\gamma_T\sqrt{2\log(|\cD|T)}T^{1-\alpha}+ 12B\kappa\log (2T)\\
    & = O(T^{1-\alpha/2}\sqrt{\log (|\cD|T)})+O(\sqrt{\gamma_T T^{1-\alpha}\log (|\cD)T)}
    +O(G_1\gamma_T T^{1-\alpha}\sqrt{\log(|\cD|T)})
    +O(\sqrt{\gamma_T T}).
\end{aligned}
\end{equation}
Finally, substituting $G_1$ with $\delta_1=\delta_2=\delta/2$, we have the total expected regret under the \texttt{DP-DPBE} with the central model  is
\begin{equation}
\begin{aligned}
    \E[R(T)] &= O(T^{1-\alpha/2}\sqrt{\log (|\cD|T)})  + O\left(\frac{\ln(1/\delta)\gamma_TT^{1-\alpha}\sqrt{\log(kT)}}{\epsilon}\right) \\
    &+O(\sqrt{\gamma_T T^{1-\alpha}\log (|\cD)T)} +O(\sqrt{\gamma_T T}).
\end{aligned}
\end{equation}
\end{proof}
 While the \texttt{DPBE} algorithm uses GP tools to define and manage the uncertainty in estimating the unknown function $f$, the analysis of \texttt{DPBE} algorithm does not rely on any \emph{Bayesian} assumption about $f$ being actually drawn from the prior $\cG\cP(0,k)$, and it only requires $f$ to be bounded in the kernel norm associated with the RKHS $\cH_k$. 
\section{Additional Numerical Results}\label{app:experiments}

\subsection{Evaluation of \texttt{DP-DPBE}}
In Section~\ref{sec:experiment}, we evaluated \texttt{DP-DPBE} on the synthetic function. In this subsection, we present additional numerical results for \texttt{DP-DPBE} on the standard benchmark functions and the function from real-world (light-sensor) data. \Hi{By considering the same setting as for the synthetic function, we run $T=10^6$ rounds and present how the cumulative regret at the end of $T$ varies with privacy budget $\epsilon \in \{5,10,15,20,25,30\}$ and $\delta=10^{-6}$ in Figure~\ref{fig:dp_functions_epsilon}. Then, by choosing privacy parameters $\delta=10^{-6}$ and $\epsilon=15$, 
we also compare the per-round regret of \texttt{DP-DPBE} and \texttt{DPBE} for the three benchmark functions and the real-world (light-sensor) data and present the results in Figure~\ref{fig:dp_functions_delta_106}.} 
We  perform $20$ runs for each simulation. From these results, we make similar observations to those for the synthetic function: the privacy-regret tradeoff and achieving privacy ``for free''. 
\begin{figure}[!t] 
\centering
	\subfigure[Sphere function]{
	\label{fig:sphere_dp_epsilon}
	\includegraphics[width=0.4\textwidth]{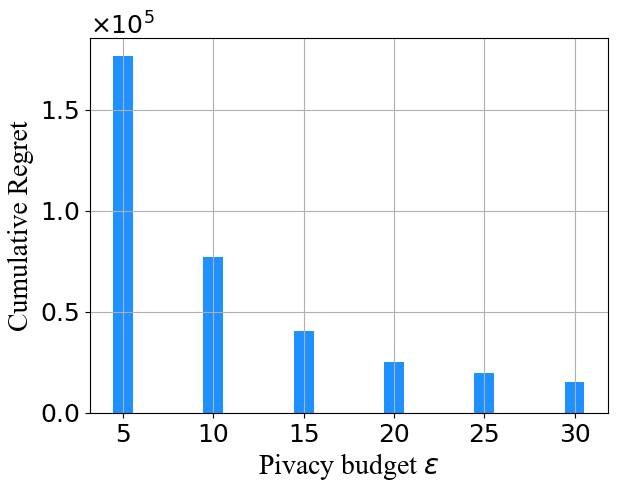}}
	\quad
	\subfigure[Six-Hump Camel function]{
	\label{fig:six-hump_dp_epsilon}
	\includegraphics[width=0.4\textwidth]{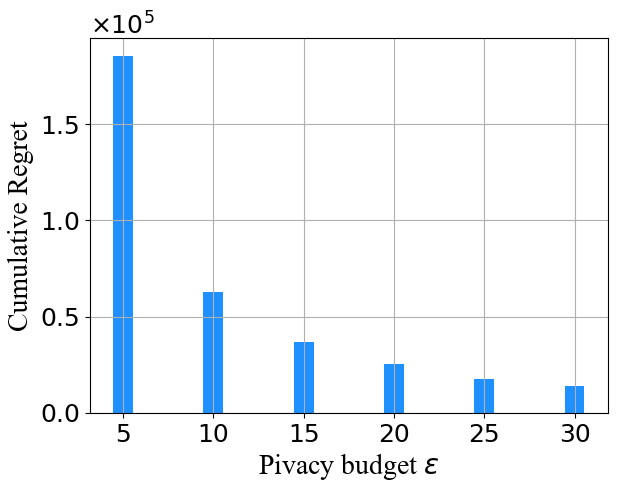}}
	\subfigure[Michalewicz function]{
	\label{fig:michalewicz_dp_epsilon}
	\includegraphics[width=0.4\textwidth]{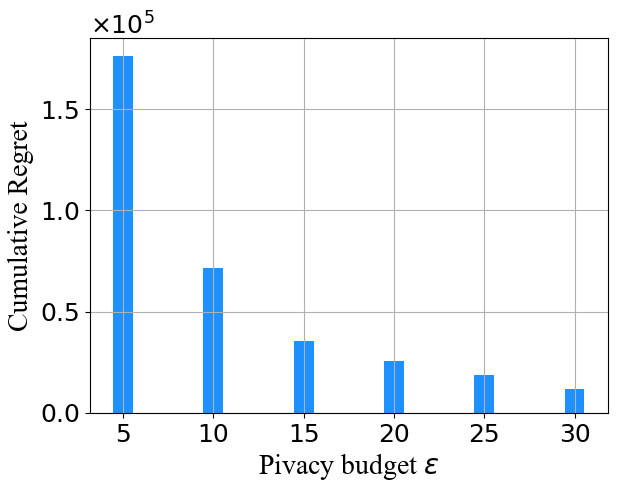}}
 \quad
	\subfigure[Function from light sensor data]{
	\label{fig:light_dp_epsilon}
	\includegraphics[width=0.4\textwidth]{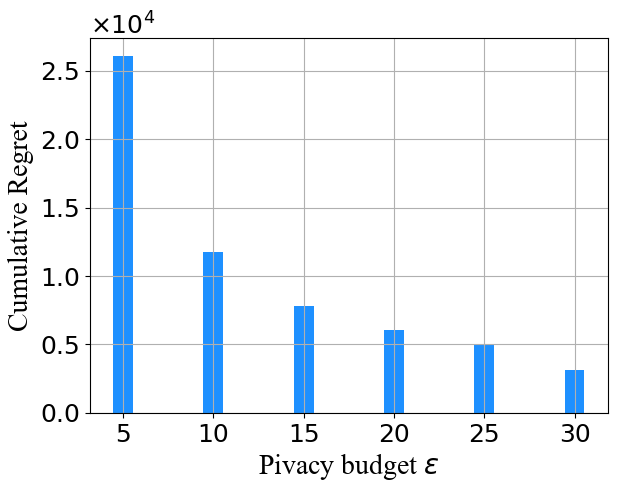}}
	\caption{Performance of \texttt{DP-DPBE}: Final cumulative regret vs. privacy budget $\epsilon$ with \Hi{$\delta=10^{-6}$}.} 
	\label{fig:dp_functions_epsilon}  
\end{figure}
\begin{figure*}[!t] 
\centering
	\subfigure[Sphere function]{
	\label{fig:sphere_dp_sigma}
	\includegraphics[width=0.4\textwidth]{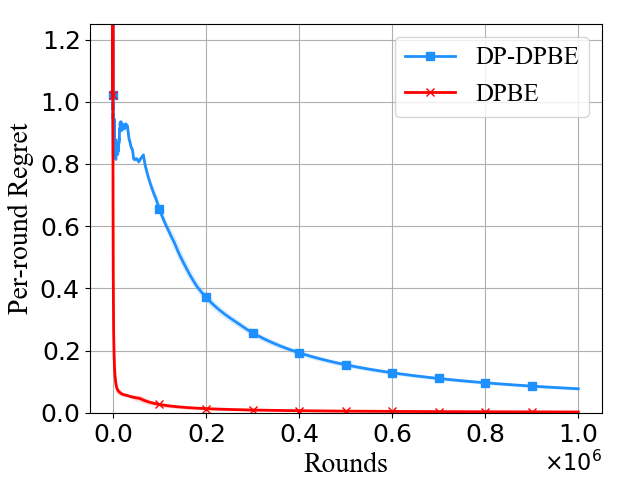}}
	\subfigure[Six-Hump Camel function]{
	\label{fig:six-hump_dp_sigma}
	\includegraphics[width=0.4\textwidth]{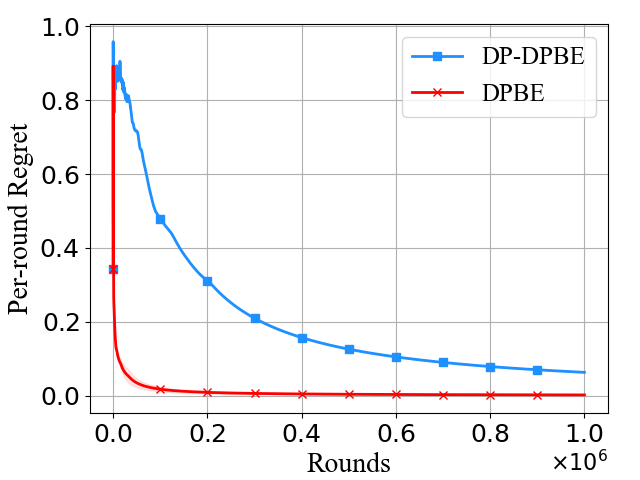}}
	\subfigure[Michalewicz function]{
	\label{fig:michalewicz_dp_sigma}
	\includegraphics[width=0.4\textwidth]{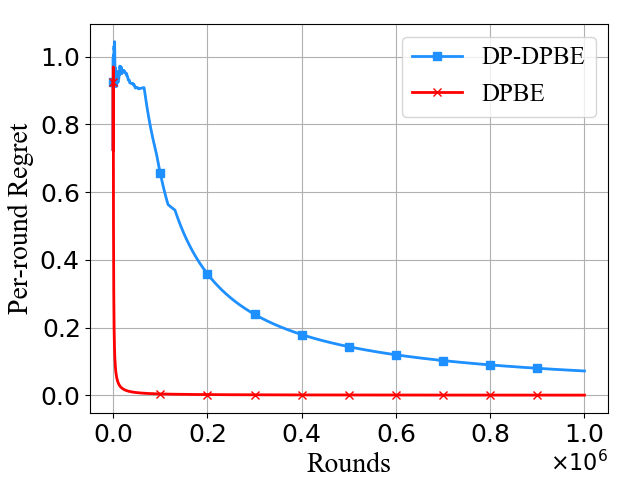}}
	\subfigure[Function from light sensor data]{
	\label{fig:light_dp_delta}
	\includegraphics[width=0.4\textwidth]{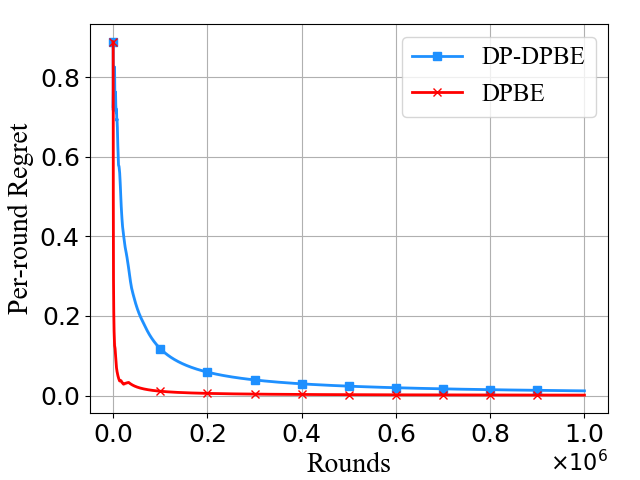}}
	\caption{Performance of \texttt{DP-DPBE}: \Hi{Per-round regret vs. time with parameters $\epsilon=15$ and $\delta=10^{-6}$.} 
}
\label{fig:dp_functions_delta_106} 
\end{figure*}
\subsection{Comparison with State-of-the-Art}
In Section~\ref{sec:compare_benchmarks}, we provide simulation results on the regret performance and running time of \texttt{GP-UCB}, \texttt{BPE}, and our algorithm \texttt{DPBE} with different values of $\alpha$ on the synthetic data generated in Section~\ref{sec:synthetic_func}
In this section, we add additional numerical results on three benchmark functions (Sphere, Six-hump Camel, Michalewicz) and one function from real-world data-- Light sensor data \cite{lightsensor}. The parameters of the problem setting and the algorithms are as follows: $T=4\times 10^4$, $|\cD|=100$, and $k=k_{SE}$ with $l_{SE}=0.2$; (a) Sphere function. Settings: $d=3,  C=1.5, \sigma=0.01, v^2=0.001, \lambda=\sigma^2/v^2$; (b) Six-Hump Camel function. Settings: $d=2, C=1.5, \sigma=0.01, v^2=0.01, \lambda=\sigma^2/v^2$; (c) Michalewicz function. Settings: $d=2, {C=1.5}, \sigma=0.01, v^2=0.01, \lambda=\sigma^2/v^2$; (d) Functions from real-world data. Settings: $d=2, C=1.42, \sigma=0.01, v^2=0.01, \lambda=\sigma^2/v^2$. We plot the cumulative regret for all the algorithms in Figure~\ref{fig:vs_gpucb_more} and present the running time in Table~\ref{tab:wallclock_functions}. 

\begin{figure*}[!t] 
\centering
	\subfigure[Sphere function]{
	\label{fig:sphere_sec8}
	\includegraphics[width=0.45\textwidth]{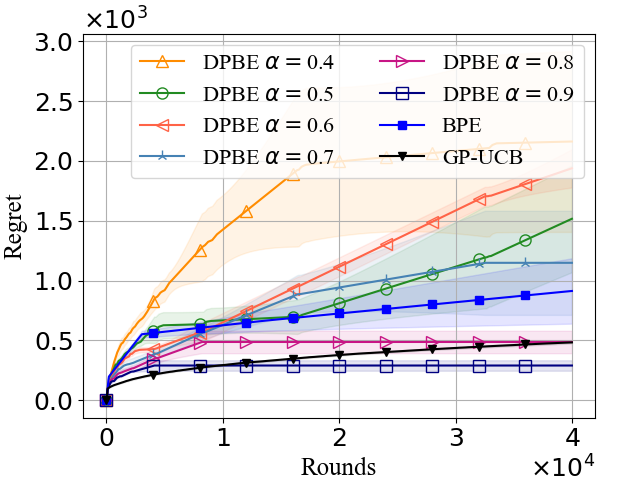}}
	\subfigure[Six-Hump Camel function]{
	\label{fig:six-hump_sec8}
	\includegraphics[width=0.45\textwidth]{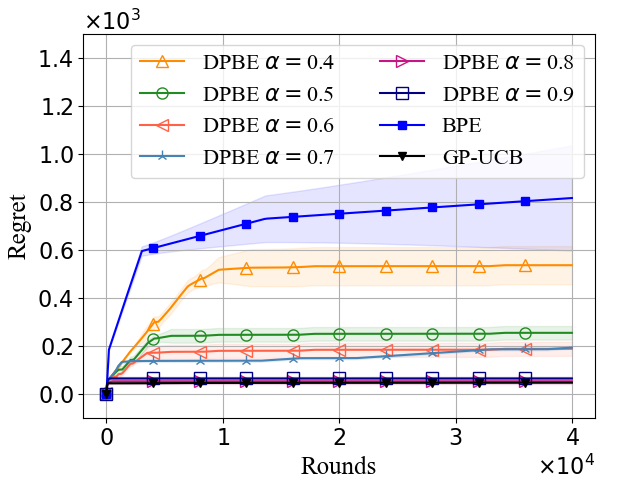}}
	\subfigure[Michalewicz function]{
	\label{fig:michalewicz_sec8}
	\includegraphics[width=0.45\textwidth]{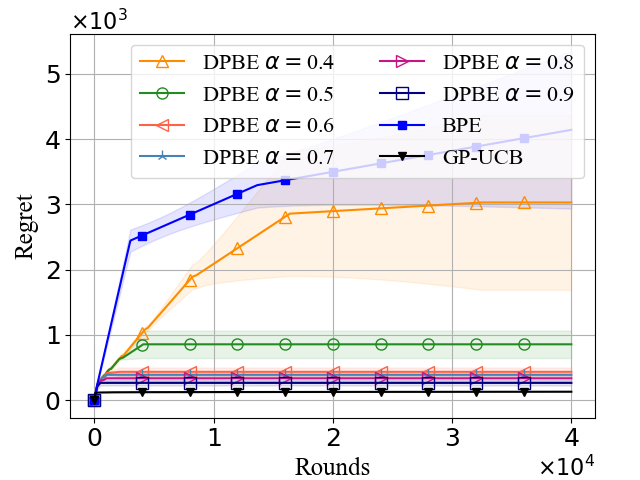}}
	\subfigure[Function from light sensor data]{
	\label{fig:light_sec8}
	\includegraphics[width=0.45\textwidth]{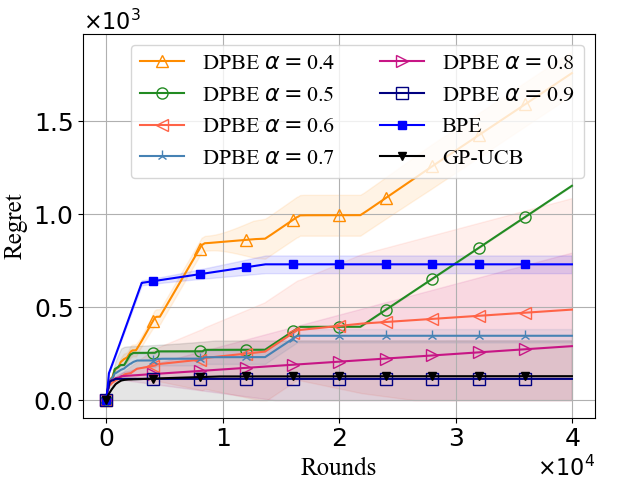}}
	\caption{Comparison of regret performance under \texttt{DPBE}, \texttt{GP-UCB}, and \texttt{BPE} on three benchmark functions and one function from real-world dataset. The shaded area represents the standard deviation. 
	}
	\label{fig:vs_gpucb_more}  
\end{figure*}

\begin{table}[!t]
\caption{Comparison of running time (seconds) under \texttt{GP-UCB}, \texttt{BPE}, and \texttt{DPBE} with different values of $\alpha$.}
\label{tab:wallclock_functions}
	\scalebox{0.9}{
\begin{tabular}{c|c|c|c|c|c|c|c|c} 
    \toprule
    \multirow{2}*{Algorithms}  & 
    \multicolumn{6}{c|}{\texttt{DPBE}}
    & \multirow{2}*{\texttt{GP-UCB}} & \multirow{2}*{\texttt{BPE}}\\
    \cline{2-7}
    ~ & $\alpha=0.4$ & $\alpha=0.5$& $\alpha=0.6$& $\alpha=0.7$& $\alpha=0.8$ & $\alpha=0.9$&~&~ \\  
    \hline
     Sphere   &$0.08$  & $0.07$  & $0.07$& $0.07$ & $0.09$ & $0.13$  & $4.68$ & $37.87$\\
     Six-Hump Camel  & $0.04$  & $0.03$  & $0.04$& $0.03$ & $0.04$ & $0.04$  &  $4.79$ & $10.43$\\
     Michalewicz   & $0.04$ & $0.04$  & $0.05$ & $0.06$ & $0.07$ & $0.11$  & $4.95$ & $4.48$\\
     Light Sensor Data   & $0.04$ & $0.06$  & $0.07$ & $0.03$ & $0.06$  & $0.05$  & $3.22$ & $82.08$\\
      \bottomrule
\end{tabular}
}
\end{table}

\end{appendix}
\end{document}